\pdfoutput=1

\documentclass{article}

\usepackage[accepted]{tmlr}

\usepackage{amsmath,amssymb,amsthm,mathtools,mathrsfs}
\usepackage{graphicx}
\usepackage{subcaption}
\usepackage{booktabs}
\usepackage{tabularx}
\usepackage{multirow}
\usepackage{makecell}
\usepackage{xspace}
\usepackage{xcolor}
\usepackage{algorithm}
\usepackage{algpseudocode}
\usepackage{acronym}
\usepackage[hidelinks]{hyperref}
\hypersetup{pdftitle={SafeExplorer: An Unbiased Mixed-Policy Gradient for Reinforcement Learning with Recovery Interventions}, pdfauthor={Elham Daneshmand, Majid Khadiv, Glen Berseth, Hsiu-Chin Lin}, pdfsubject={Transactions on Machine Learning Research}, pdfkeywords={safe reinforcement learning, recovery policy, policy gradient, PPO}}
\usepackage[capitalize,noabbrev]{cleveref}
\usepackage{enumitem}
\usepackage{wrapfig}

\AtBeginDocument{%
  {\fontencoding{T1}\fontfamily{lmr}\selectfont}%
  {\fontencoding{T1}\fontfamily{lmss}\selectfont}%
  \normalfont
  \DeclareFontShape{T1}{lmr}{bx}{sc}{<->ssub*lmr/bx/n}{}%
  \DeclareFontShape{T1}{lmss}{bx}{sc}{<->ssub*lmss/bx/n}{}%
}

\newcommand{\piMain}{\pi_\theta}
\newcommand{\piRec}{\pi^{\mathrm{rec}}}
\newcommand{\piMixed}{\pi^{\mathrm{mix}}_\theta}
\newcommand{\pTraj}{p_\theta(\tau)}
\newcommand{\pTrajMixed}{p^{\mathrm{mix}}_\theta(\tau)}
\newcommand{\safeR}{\mathcal{R}}
\newcommand{\Expect}{\mathbb{E}}
\newcommand{\indicator}{\mathbb{1}}
\newcommand{\alphaRate}{\alpha(\theta, d)}
\newcommand{\betaRate}{\beta(\theta, d)}

\newcommand{\methodName}{SafeExplorer\xspace}
\newcommand{\cmark}{\textcolor{green!55!black}{\checkmark}}
\newcommand{\xmark}{\textcolor{red!70!black}{\ensuremath{\times}}}

\newcommand{\Ctheta}{C_\theta(\tau)}
\newcommand{\Vmix}{V^{\piMixed}}
\newcommand{\Qmix}{Q^{\piMixed}}
\newcommand{\Amix}{A^{\piMixed}}
\newcommand{\dmax}{d_{\max}}
\newcommand{\diter}{d_{\mathrm{iter}}}
\newcommand{\rmax}{r_{\max}}
\newcommand{\Tsafe}{T_{\safeR}}
\newcommand{\dpiMain}{\nu^{\piMain}}
\newcommand{\dpiMixed}{\nu^{\piMixed}}
\newcommand{\Aest}{\hat{A}}
\newcommand{\seg}{\operatorname{seg}}

\newcommand{\envHC}{\textsc{HalfCheetah}\xspace}
\newcommand{\envAnt}{\textsc{Ant}\xspace}
\newcommand{\envGo}{\textsc{Go1}\xspace}

\newcommand{\valueLoss}{L^{V}}
\newcommand{\PPOloss}{L^{\mathrm{PPO}}}
\newcommand{\compatLoss}{L^{\mathrm{compat}}}
\newcommand{\lambdaCompat}{\lambda_{\mathrm{compat}}}

\theoremstyle{plain}
\newtheorem{theorem}{Theorem}
\newtheorem{proposition}[theorem]{Proposition}
\newtheorem{corollary}[theorem]{Corollary}
\newtheorem{lemma}[theorem]{Lemma}
\theoremstyle{definition}
\newtheorem{assumption}{Assumption}
\newtheorem{definition}[theorem]{Definition}
\theoremstyle{remark}
\newtheorem{remark}{Remark}

\acrodef{RL}{reinforcement learning}
\acrodef{PPO}{proximal policy optimization}
\acrodef{GAE}{generalized advantage estimation}
\acrodef{MDP}{Markov decision process}
\acrodef{CMDP}[CMDP]{constrained Markov decision process}
\acrodefplural{CMDP}[CMDPs]{constrained Markov decision processes}
\acrodef{SAC}{soft actor-critic}
\acrodef{MPC}{model predictive control}
\acrodef{IS}{importance sampling}
\acrodef{DoF}{degree-of-freedom}
\acrodef{CBF}{control barrier function}
\acrodef{CPO}{constrained policy optimization}
\acrodef{DAgger}{Dataset Aggregation}
\acrodef{JSRL}{Jump-Start RL}

\newcommand{\todo}[1]{}

\newif\ifrevisionmarks
\revisionmarksfalse
\newcommand{\revised}[1]{\ifrevisionmarks\begingroup\color{blue}#1\endgroup\else#1\fi}
\newcommand{\revisedfloat}{\ifrevisionmarks\color{blue}\fi}
\newcommand{\revisedgraphics}[2][]{%
  \ifrevisionmarks
    {\setlength{\fboxsep}{1.5pt}\setlength{\fboxrule}{1pt}%
     \textcolor{blue}{\fbox{\begin{minipage}{\dimexpr\linewidth-2\fboxsep-2\fboxrule\relax}%
       \normalcolor\centering
       \includegraphics[#1]{#2}%
     \end{minipage}}}}%
  \else
    \includegraphics[#1]{#2}%
  \fi}

\newcommand{\envWalker}{\textsc{Walker2d}\xspace}

\title{SafeExplorer: An Unbiased \revised{Mixed-Policy} Gradient for\\ Reinforcement Learning with Recovery Interventions}

\author{%
  \name Elham Daneshmand \email elham.daneshmand@mail.mcgill.ca \\
  \addr McGill University \& Mila
  \AND
  \name Majid Khadiv \email majid.khadiv@tum.de \\
  \addr Technical University of Munich
  \AND
  \name Glen Berseth \email glen.berseth@umontreal.ca \\
  \addr Universit\'e de Montr\'eal \& Mila
  \AND
  \name Hsiu-Chin Lin \email hsiu-chin.lin@mcgill.ca \\
  \addr McGill University \& Mila
}

\def\month{10}
\def\year{2026}
\def\openreview{\url{https://openreview.net/forum?id=ppOrbSsMM0}}

\begin{document}
\maketitle

{\let\thefootnote\relax\begin{NoHyper}\footnotetext{The authors used a large language model (Claude, Anthropic) to assist with drafting and editing the text and with writing analysis code. The authors checked all content and take full responsibility for it.}\end{NoHyper}}

\begin{abstract}
Training agents directly on the task they will be deployed on often leads to the best policy performance, but training on physical robots can cause damage through falls, so minimizing training falls is critical. A standard mitigation hands control to a separate \emph{recovery policy} whenever the agent leaves a designer-specified \emph{safe region}. The resulting mixed-policy rollouts bias the on-policy update, and importance sampling cannot remove this bias when the recovery policy is deterministic. We propose \methodName{}, a drop-in modification of \ac{PPO} built on an unbiased policy-gradient estimator of the mixed-policy return. The estimator uses the score function only inside the safe region and never evaluates the recovery policy's density, so deterministic recovery policies pose no problem. On final reward, the estimator matches or outperforms truncated importance sampling, which is itself biased for that same gradient. \methodName{} also speeds up learning near the safe-region boundary with a closed-form critic target at recovery-triggering states when dynamics and recovery are deterministic, and an imitation loss that copies recovery actions only from successful recoveries. On three simulated locomotion environments with ten seeds each, and given an already available recovery controller, \methodName{} incurs $255\times$, $43\times$, and $21\times$ fewer training falls than standard \acs{PPO} on HalfCheetah, Ant, and Unitree Go1. \methodName{} also matches or exceeds the final reward of \acs{PPO}. On Ant, where $8.7\%$ of recovery interventions end in a fall, over $40\times$ the HalfCheetah rate, \methodName{} reaches the success threshold on every seed and none of the five baselines reaches the threshold on any seed.
\end{abstract}

\section{Introduction}
\label{sec:intro}

\Ac{RL} policies often perform best when trained directly on the task they will be deployed on, but moving that training onto a physical robot is hard for a reason unrelated to asymptotic performance: learning requires failures, and on real hardware every failure has a price. A \emph{fall}, a loss of balance that ends the episode, can damage the platform and cannot be undone the way a simulator reset can. Recent methods learn capable controllers by collecting hundreds of thousands of episodes in massively parallel simulation \citep{Lee_2020,rudin2022learningwalkminutesusing,Agarwal2022LeggedLI}, where these failures are free, but on a real robot they are not \citep{Ibarz_2021,smith2022walkparklearningwalk}. What gates real-world training is therefore not the final reward but the number of falls spent reaching it, and the goal of \methodName{} is to keep that number small.
\revised{Falls are not the only obstacle. Sample complexity binds from-scratch real-world training, and our experiments are simulated at simulation-scale budgets. We therefore frame the contribution as an unbiased gradient estimator for \ac{PPO} pipelines, from scratch or fine-tuning a sim-pretrained policy with the same optimizer, not as a hardware-practicality claim.}

The common way to keep learning on real hardware is to pair the agent with a separate \emph{recovery policy}, a controller that takes over whenever the agent leaves a designer-specified \emph{safe region} and steers the system back toward a nominal state\revised{, upright standing in ours (\Cref{sec:background})}. Because the safe region is set conservatively, the agent leaves it routinely, so the recovery policy fires far more often than a fall would occur, converting most would-be falls into cheap, reversible interventions. The controller is easy to assemble from off-the-shelf parts, a \ac{MPC} solver \citep{pua2024safe}, a \ac{SAC} expert \citep{pmlr-v80-haarnoja18b}, or a learned recovery controller \citep{lee2019robustrecoverycontrollerquadrupedal}.
But it carries a hidden cost: the data the agent collects is no longer produced by the agent alone, but by a \emph{mixed policy}, the main policy inside the safe region and the recovery outside it, so each rollout is a blend of two controllers rather than a sample from the one being improved.

This blend breaks the on-policy assumption. \ac{PPO} \citep{schulman2017proximalpolicyoptimizationalgorithms} and related methods estimate the gradient of the policy that produced the rollout, so they are valid only when the rollout distribution matches the policy being updated. The moment the recovery policy intervenes, every update is biased toward the recovery policy's behavior. Most safe-\ac{RL} work leaves this mismatch unaddressed, updating as if all data came from the main policy \citep{srinivasan2020learningsafedeeprl,thananjeyan2021recoveryrlsafereinforcement,yang2022safereinforcementlearninglegged}, or sidesteps it by shaping the reward to discourage unsafe states \citep{tessler2018rewardconstrainedpolicyoptimization,stooke2020lagrangian}. The standard remedy, an importance-sampling correction \citep{degris2012off}, does not solve the problem either: a deterministic recovery (an \ac{MPC} solver or a greedy \ac{SAC} actor) produces a single action rather than a distribution, so the importance ratio has no density in its denominator and is undefined precisely where it is needed, the same obstruction that motivates deterministic policy gradients \citep{silver2014dpg}. The mechanism that makes real-world training feasible is thus the one that corrupts the learning signal, and a correct gradient must factor the recovery policy out rather than reweight through it.

\noindent \textbf{Contributions.} We close this gap with \methodName{}, a practical algorithm built on a theoretical foundation for safe \ac{RL} with a recovery policy. We make four contributions. \emph{First}, we prove a policy-gradient theorem that is unbiased for the gradient of the mixed-policy return for any external recovery policy that takes over on a fixed subset of states (\Cref{thm:gradient}). The gradient never evaluates the recovery policy's action density, so the gradient also covers deterministic recoveries. On final reward, the gradient estimator matches or outperforms truncated \ac{IS}, which is itself biased for this gradient (\Cref{rem:is-variance}). Jump-Start \ac{RL} \citep{uchendu2023jsrl} and state-triggered shielded \ac{RL} follow as corollaries (\Cref{cor:jsrl-gradient,cor:shielded-gradient}). \emph{Second}, we bound the gap between the mixed-policy return we optimize and the main-policy return, a gap that shrinks as the safe region grows (\Cref{thm:convergence}). \emph{Third}, when dynamics and recovery are deterministic, the value of a recovery-triggering state has a closed form that gives the critic accurate targets at the safe-region boundary (\Cref{prop:analytic}). \emph{Fourth}, an outcome-gated compatibility regularizer imitates recovery actions only from successful recovery segments, and without the gate the regularizer reduces to \ac{DAgger}-style behavioral cloning \citep{ross2011dagger} (\Cref{def:compat,prop:compat-special-cases}).

\revised{On three simulated locomotion environments, \methodName{} cuts training falls over standard \ac{PPO} by $21\times$ to roughly $255\times$ ($10$ seeds; \Cref{subsec:tradeoff}), at matching or better final reward; on \envAnt, where the recovery controller often fails, \methodName{} reaches the fixed success bar on every seed while no baseline of \Cref{tab:main-results} reaches the bar on any seed, the unmasked ablation variant (\Cref{sec:ablations}) on $2$ of $10$.}

\section{Related Work}
\label{sec:related}

We study an agent paired with an external recovery policy that takes over outside a safe region. The setup causes training rollouts to follow a mixed policy that biases every on-policy update.

\noindent \textbf{Off-policy RL} Correcting that bias looks like a standard off-policy problem, but off-policy policy-gradient and evaluation methods \citep{degris2012off,gu2017interpolated,jiang2016doubly} rely on the importance ratio that \Cref{sec:intro} showed a deterministic recovery leaves undefined. Truncated-\ac{IS} schemes such as V-trace \citep{espeholt2018impala} and Retrace \citep{munos2016retrace} cannot help, because the singularity sits upstream of the truncation. Our \Cref{thm:gradient} factors the recovery measure out instead of reweighting, so it applies uniformly to deterministic and stochastic recoveries, \ac{MPC} controllers \citep{Chiu2022ACM} included.

\revised{\noindent \textbf{Offline RL} Offline \ac{RL} \citep{jackson2025cleanslate}, including model-based MOReL \citep{kidambi2020morel} and MOPO \citep{yu2020mopo}, incurs zero additional environment falls by construction but presupposes a dataset whose collection incurred its own failures and whose coverage bounds the achievable policy; online training with a recovery policy is preferable when no adequate dataset exists or a recovery controller is cheaper than a covering dataset, offline in the complementary case. Our conditioning assumption, an already available recovery controller, is accounted for in \Cref{app:pretraining-accounting}.}

\noindent \textbf{Mixed-policy and intervention data} If reweighting is not the route, the alternative is to ask how prior work uses the same mixed-policy data, and rollouts in which a second policy intervenes are in fact common: they appear in \ac{DAgger} \citep{ross2011dagger,kelly2019hgdagger}, in human-in-the-loop \ac{RL} \citep{spencer2020learning,10.5555/3237383.3238074}, and in \ac{JSRL} \citep{uchendu2023jsrl}. These either imitate the intervening policy or discard its transitions, so none extract a bias-corrected on-policy gradient from that data. Our masked gradient (\Cref{thm:gradient}) recovers that gradient directly, and our compatibility regularizer (\Cref{sec:compat}) reduces in the \ac{JSRL} setting to a state-dependent imitation loss with an outcome-based gate (\Cref{prop:compat-special-cases}).

\noindent \textbf{Constraints and safety filters} Delegating safety to an external recovery policy is itself the minority choice: the dominant lines of safe \ac{RL} \citep{garcia2015comprehensive} build safety into the policy rather than handing it to a separate controller.
Constrained Markov decision processes (\acsp{CMDP})\acused{CMDP} encode safety as constraints on long-run cost \citep{altman1999constrained}, optimized via primal-dual schemes such as \ac{CPO} \citep{achiam2017constrained}, reward-constrained policy optimization \citep{tessler2018rewardconstrainedpolicyoptimization}, and PID-Lagrangian variants \citep{stooke2020lagrangian}. Because they trade cost against return, they suit soft-constraint settings where some failures are tolerable, whereas we treat falls as something to minimize rather than to budget against return. Closer to a hard guarantee, \ac{CBF} approaches \citep{ames2019cbf} and shields \citep{junges2015safetyconstrainedreinforcementlearningmdps,alshiekh2018safe,dalal2018safeexplorationcontinuousaction,srinivasan2020learningsafedeeprl,hasanbeig2020cautiousreinforcementlearninglogical,kang2022lyapunovdensitymodelsconstraining} project unsafe actions at execution time; this is exactly the case of our mechanism in which the intervention is a single-step projection, so our setup subsumes it as one instance. Closely related are the state-wise constrained Markov decision process family \citep{zhao2023statewise} and almost-surely safe \ac{RL} \citep{sootla2022saute}. \Cref{sec:background} places our setting in the constraint-formulation taxonomy of \citet{wachi2024constraint}.

\noindent \textbf{RL with a recovery policy} The work closest to ours commits fully to that delegation: it uses an explicit recovery policy as an alternative to \ac{CMDP} costs and pointwise filters, facing the same mixed-policy bias extended from the single-step projections of our shielded-\ac{RL} corollary to multi-step recovery segments. \revised{\citet{eysenbach2018leavenotrace} learn a reset policy jointly with the task policy and abort to it when the reset policy's value function predicts a state cannot be returned from; the recovery controller there is learned during training rather than fixed in advance, and no correction of the mixed-rollout bias is attempted.} \citet{thananjeyan2021recoveryrlsafereinforcement} train a safety critic and switch to the recovery policy when constraint risk crosses a threshold; they deliberately relabel recovery transitions with the task policy's proposed action, coherent for their off-policy Q-learning objective (\Cref{app:baseline-ports}) but the source of the bias we correct once carried into an on-policy update. \citet{yang2022safereinforcementlearninglegged} apply a similar template to legged locomotion with reward shaping, and \citet{lee2019robustrecoverycontrollerquadrupedal} train a model-free quadruped recovery controller, an example of the external recovery policy our masked gradient handles directly.
 Closest, \citet{wagener2021sailr} intervene on an advantage criterion and prove the intervened process safe, but still optimize the policy on the mixed rollouts without correcting the bias we address. Reverse-curriculum methods \citep{florensa2017reverse} grow the region of starting states, and teacher-scheduled interventions \citep{turchetta2020curriculum} induce a safety curriculum, whereas safe-region curricula like ours grow the region the agent is permitted to enter. None of these treat the resulting distributional bias on the on-policy gradient, which our masked-gradient theorem (\Cref{thm:gradient}) addresses without requiring the recovery policy to have a density.

\noindent \textbf{RL for legged locomotion} End-to-end \ac{RL} has advanced rapidly for legged locomotion~\citep{ha2025learning,Lee_2020,Peng_2018,tan2018simtoreallearningagilelocomotion,peng2020learningagileroboticlocomotion,rudin2022learningwalkminutesusing,Agarwal2022LeggedLI,kumar2021rmarapidmotoradaptation,haarnoja2019learningwalkdeepreinforcement,bogdanovic2022modelfreereinforcementlearningrobust}, and fine-tuning a policy in the real world has become common practice~\citep{smith2022walkparklearningwalk,smith2023growlimitscontinuousimprovement,smith2023learningadaptingagilelocomotion,Liu2024VisualWC}. These methods secure safety through reward design and sim-to-real robustness, but none confronts the bias a recovery policy injects into the on-policy gradient. %
That bias comes from the mixed rollout rather than from any property of legged robots, so \Cref{thm:gradient} applies to any $\theta$-independent, state-triggered intervention rule. The platform-independence claim attaches to the estimator, not to the system as instantiated, which needs a per-environment distance to the nominal state, safe-region size and pretrained recovery policy (\Cref{sec:exp-setup}).

\section{Problem Setting and Notation}
\label{sec:background}

We study \ac{RL} in the \emph{hard-safety regime} of the constraint-formulation taxonomy of \citet{wachi2024constraint}. On a real robot a fall can break hardware, damage property, or injure a bystander, so the goal is to keep falls as rare as possible rather than to budget them against return.

The strictest formalism related to this regime is the state-wise constrained Markov decision process family \citep{zhao2023statewise}, which demands constraint satisfaction at every step. We relax that requirement to \emph{violation-minimization}, the model-free goal the survey identifies: without prior knowledge of the dynamics, hard state-wise safety cannot be guaranteed during training, so the target is to incur as few constraint violations (falls) as possible.

Among the ways to enforce safety in this regime (\Cref{sec:related}), ours is an external \emph{multi-step} recovery policy that takes over whenever the state leaves a designer-specified safe region and %
returns the agent toward a nominal configuration\revised{, in ours the upright standing pose}, unlike single-step action filters such as control-barrier projections and shields. We call this the \emph{safe-region intervention} mechanism. \revised{A consolidated notation table is in \Cref{app:notation}.}

We instantiate the mechanism with a discounted \ac{MDP} $(\mathcal{S}, \mathcal{A}, P, r, \gamma, T)$ with continuous state space $\mathcal{S}$ and action space $\mathcal{A}$, transition kernel $P$, reward $r \in [-\rmax, \rmax]$, discount $\gamma \in (0,1)$, and episode horizon $T$. The analysis uses the infinite-horizon discounted convention (absorption at failure states); experiments truncate at $T = 1000$ steps.

\paragraph{\revised{Safe region and curriculum.}} The safe region is the sublevel set $\safeR(d) := \{s : \mathcal{D}(s) \le d\}$ of a task-space distance $\mathcal{D}$ to a nominal configuration. The radius $d \in [0, \dmax]$ grows over training by a curriculum schedule $\diter(\cdot)$, where $\dmax$ is a per-environment hyperparameter (\Cref{sec:algorithm}).

\paragraph{\revised{Policies.}} Let $\piMain(a\mid s)$ be the main policy, stochastic and $\theta$-parameterized. The recovery policy $\piRec(a\mid s)$ is a state-conditioned policy that may be deterministic (\ac{MPC}, greedy \ac{SAC}) or stochastic. The recovery policy does not depend on $\theta$, since the recovery policy's parameters (\ac{SAC} weights or an \ac{MPC} cost) stay fixed during training, and we omit these parameters from the notation.
The \emph{mixed policy} $\piMixed$ follows $\piMain$ inside $\safeR(d)$ and $\piRec$ outside it. We give its general form and an unbiased gradient in \Cref{sec:gradient}.

\paragraph{\revised{Objectives.}} The trajectory distributions under $\piMain$ alone and under $\piMixed$ are
$\pTraj = p(s_1) \prod_t \piMain(a_t\mid s_t) P(s_{t+1}\mid s_t, a_t)$ and
$\pTrajMixed = p(s_1) \prod_t \piMixed(a_t\mid s_t) P(s_{t+1}\mid s_t, a_t)$.
The \emph{main-policy return} is
\begin{equation}
    J(\theta) = \Expect_{\tau \sim \pTraj}[r(\tau)], \qquad r(\tau) := \sum_t \gamma^t r_t.
\end{equation}
The return $J$ is not directly observable during training, since rollouts are drawn from $\pTrajMixed$, not $\pTraj$.
The \emph{mixed-policy return}, the quantity training rollouts actually realize, is
\begin{equation}
    J^{\mathrm{mix}}(\theta) = \Expect_{\tau \sim \pTrajMixed}[r(\tau)].
    \label{equ:return-mix}
\end{equation}
Its state-conditional form, the \emph{mixed-policy value} $\Vmix(s) := \Expect_{\tau \sim \pTrajMixed}[r(\tau) \mid s_1 = s]$, is what \ac{PPO}'s learned critic $V_\theta$ estimates. Discounted state-visitations under $\piMain$ and $\piMixed$ are $\dpiMain$ and $\dpiMixed$, respectively.
\revised{We treat the recovery policy as remaining active at deployment (\Cref{sec:discussion}), so $J^{\mathrm{mix}}$ is itself deployment-relevant, not only a training surrogate for $J$.}

\paragraph{\revised{Two rates.}} We use two rates, distinguished by their sampling distribution. The \emph{recovery rate} $\alphaRate := \Pr_{s \sim \dpiMixed}[s \notin \safeR(d)]$ is sampled under the \emph{mixed} distribution that training rollouts produce, hence \emph{observable} as the fraction of rollout steps on which the recovery policy fired. The \emph{out-of-region rate} $\betaRate := \Pr_{s \sim \dpiMain}[s \notin \safeR(d)]$ is sampled under the \emph{main-only} distribution never deployed during training, hence \emph{not directly observable}, yet it is the quantity that controls the gap in \Cref{thm:convergence}.

Together, these objects give the setup four operational features: binary set-membership safety, a possibly-deterministic recovery policy, multi-step recovery (a genuine mixed-policy distribution), and the unmodified main-policy objective, each addressed by one method component in \Cref{sec:method} (\Cref{app:operational-features}).

\section{Method}
\label{sec:method}

\methodName{} builds on \ac{PPO} with four components, each removing a failure mode of recovery-using \ac{RL}: (i) a \emph{masked policy gradient} (\Cref{sec:gradient}) that removes the bias from treating recovery transitions as main-policy transitions; (ii) an \emph{objective-gap bound} (\Cref{sec:convergence}) relating the training (mixed-policy) return to the deployment (main-policy) return and showing a safe-region curriculum closes the gap; (iii) an \emph{analytic recovery value} (\Cref{sec:value}) that replaces the learned critic at recovery-triggering states with a closed form under deterministic dynamics; and (iv) an \emph{outcome-gated compatibility regularizer} (\Cref{sec:compat}) that imitates recovery only after successful segments. We develop each in turn.

\subsection{Unbiased policy gradient via partition masking}
\label{sec:gradient}

To make the most of training data, we want every rollout to contribute to the task gradient, even those where the recovery policy took over for long stretches.
The obstacle is that treating recovery transitions as if the main policy generated them biases the on-policy update, and the \ac{IS} fix breaks down whenever the recovery policy is deterministic.
$\piMain$, a Gaussian in our \ac{PPO} implementation, has a differentiable log-density, whereas a deterministic $\piRec$ is a Dirac measure with no density. Because this obstacle arises whenever an external, $\theta$-independent policy takes over on a subset of states, not only in safety, we solve it in full generality first and then specialize to safe-region intervention.

\noindent \textbf{Setup.} We generalize the mixed policy of \Cref{sec:background}. Let $\mathcal{M} \subseteq \mathcal{S}$ be a designer-specified \emph{main-policy set} (where $\piMain$ acts), and let $\mu(\cdot\mid s)$ be any state-conditioned action measure that is \emph{independent of} $\theta$ (it may be deterministic or stochastic). The resulting \emph{partition policy} is
\begin{equation}\label{eq:partition-policy}
\pi_\theta^{\mathrm{mix}}(a\mid s) = \begin{cases} \piMain(a\mid s) & s \in \mathcal{M}, \\ \mu(a\mid s) & s \notin \mathcal{M}, \end{cases}
\end{equation}
and the trajectory return $J^{\mathrm{mix}}(\theta) = \Expect_{\tau \sim p_\theta^{\mathrm{mix}}}[r(\tau)]$.

\begin{assumption}[Regularity]\label{asm:regularity}
\textbf{(R1)} For every $(s,a)$ on the support of $p_\theta^{\mathrm{mix}}$ with $s \in \mathcal{M}$, $\piMain(a\mid s) > 0$ and $\theta \mapsto \log \piMain(a\mid s)$ is differentiable. \textbf{(R2)} $T < \infty$, or $|r_t| \le \rmax$ uniformly. \textbf{(R3)} Differentiation and integration commute in \cref{eq:correct-gradient}.
\end{assumption}
For diagonal-Gaussian $\piMain$ with bounded mean and log-std, and bounded reward, (R1) to (R3) all hold.

\begin{theorem}[Unbiased partition-policy gradient]\label{thm:gradient}
Under \Cref{asm:regularity}, the gradient of $J^{\mathrm{mix}}$ is
\begin{equation}\label{eq:correct-gradient}
\nabla_\theta J^{\mathrm{mix}}(\theta) = \Expect_{\tau \sim p_\theta^{\mathrm{mix}}}\!\left[ \left(\sum_{t=1}^T \indicator[s_t \in \mathcal{M}]\,\nabla_\theta \log \piMain(a_t \mid s_t)\right)\, r(\tau) \right].
\end{equation}
Here $\indicator[\cdot]$ is the indicator function, so the sum runs only over timesteps with $s_t \in \mathcal{M}$.
\end{theorem}

We prove this in \Cref{app:gradient-proof}: because $\mu$ is $\theta$-independent, its factor in the trajectory density vanishes under differentiation, so the masked gradient never evaluates $\mu$'s density and holds for any $\mu$. This is exactly the case the intra-option policy-gradient results \citep{sutton1999between,bacon2017option} assume away; \Cref{thm:gradient} extends the factorization to the density-free external controllers those results exclude, which is why we never need the \ac{IS} correction.

\begin{remark}[Masking \revised{versus \acs{IS}}]\label{rem:is-variance}
For stochastic $\mu$ the per-step \ac{IS} estimator decomposes into the masked estimator plus a term carried by the weights $\piMain/\mu$ at recovery steps; that term has a generically nonzero mean under a genuine advantage, so per-step \ac{IS} is in general biased for $\nabla_\theta J^{\mathrm{mix}}$, whereas the masked estimator is exactly unbiased (\Cref{thm:gradient}; decomposition in \cref{eq:is-decomp}, \Cref{app:is-variance}). \revised{The comparison below is thus between estimators that do not share a mean, with no dominance claimed among unbiased estimators of a common objective; the trained comparison uses truncated per-step \ac{IS} (clipped at $c=10$), the practical alternative, itself biased for $\nabla_\theta J^{\mathrm{mix}}$ and undefined for a deterministic recovery (\Cref{app:is-variance}).} The weights also degenerate, since recovery acts unlike the main policy by design: the measured \ac{IS} gradient variance exceeds the masked variance by up to 13 orders of magnitude (an empirical measurement, not a theorem), and truncated \ac{IS} loses $40\%$ reward on \envAnt, while masking \revised{matches or outperforms it} on final reward in every regime we test. Masking, not \ac{IS}, is the estimator \methodName{} optimizes (\Cref{cor:ppo}).
\end{remark}

\noindent \textbf{Specializations.} Different choices of $(\mathcal{M}, \mu)$ specialize \Cref{thm:gradient} to concrete settings.

\begin{corollary}[Safe-region intervention]\label{cor:safe-region-gradient}
When the main-policy set is the safe region, $\mathcal{M} = \safeR$, and the action measure is the recovery policy, $\mu = \piRec$, the partition policy is the mixed policy $\piMixed$ of \Cref{sec:background}. \Cref{thm:gradient} then gives
\(
\nabla_\theta J^{\mathrm{mix}}(\theta) = \Expect_{\tau \sim \pTrajMixed}\!\left[ \left(\sum_{t} \indicator[s_t \in \safeR]\,\nabla_\theta \log \piMain(a_t \mid s_t)\right) r(\tau) \right].
\)
\end{corollary}

The same factorization yields Jump-Start \ac{RL} \citep{uchendu2023jsrl} and state-triggered shielded \ac{RL} as corollaries
(\Cref{cor:jsrl-gradient,cor:shielded-gradient}, \Cref{app:gradient-proof}). The rest of the paper develops safe-region intervention (\Cref{cor:safe-region-gradient}).

\begin{corollary}[\ac{PPO} surrogate with safe-step masking]\label{cor:ppo}
Specializing \Cref{cor:safe-region-gradient} to the clipped \ac{PPO} surrogate, the empirical masked-gradient estimator is
\begin{equation}\label{eq:ppo-safe}
\widehat{\nabla J^{\mathrm{mix}}}(\theta) = \frac{1}{|\Tsafe|} \sum_{t : s_t \in \safeR} \nabla_\theta \log \piMain(a_t\mid s_t)\, \Aest_t^{\mathrm{full}},
\end{equation}
where $\Aest_t^{\mathrm{full}}$ is the \ac{GAE} \citep{Schulman2015HighDimensionalCC} advantage over the full mixed-policy trajectory (recovery segments enter through the return exactly as in \Cref{thm:gradient}, while the mask restricts only the policy gradient term) and $|\Tsafe|=|\{t:s_t \in\safeR\}|$ is the number of timesteps inside the safe region. The clipped \ac{PPO} surrogate is
\begin{equation}
\PPOloss_{\mathrm{safe}}(\theta) = -\frac{1}{|\Tsafe|} \sum_{t : s_t \in \safeR} \min\!\left( \rho_t(\theta)\, \Aest_t^{\mathrm{full}},\; \mathrm{clip}(\rho_t(\theta), 1\!-\!\epsilon, 1\!+\!\epsilon)\, \Aest_t^{\mathrm{full}} \right),
\end{equation}
where $\rho_t(\theta) = \piMain(a_t\mid s_t) / \pi_{\theta_{\mathrm{old}}}(a_t\mid s_t)$ is the probability ratio of the main policy, well defined because both terms are the main policy.
\end{corollary}

Relative to the standard clipped-\ac{PPO} surrogate over the same rollouts, \cref{eq:ppo-safe} changes only two things: the sums run over safe steps, and the normalizer is $|\Tsafe|$, the number of such steps in a rollout and hence random, which the curriculum on $d$ keeps large in practice (\Cref{sec:algorithm}, \Cref{app:value-loss-target}). The masked score-function term is exactly unbiased for $\nabla_\theta J^{\mathrm{mix}}$ (\Cref{thm:gradient}). %
\ac{GAE}, the clipping, \revised{the multi-epoch reuse of each rollout,} and thus \ac{PPO}'s stability heuristics are otherwise unchanged\revised{: the mask and this normalizer are the only estimator-level departures from standard \ac{PPO}, and the approximations that remain are inherited by every \ac{PPO}-based method we compare against}.

\revised{Because the normalizer $|\Tsafe|$ is random, \cref{eq:ppo-safe} is a self-normalized ratio estimator, consistent for a positive multiple of $\nabla_\theta J^{\mathrm{mix}}$, with $O(1/B)$ bias in the batch size $B$, zero once recovery stops firing (\Cref{prop:ratio-estimator}, \Cref{app:value-loss-target}).}

\subsection{Bounding the gap between mixed-policy and main-policy returns}
\label{sec:convergence}

\Cref{thm:gradient}'s gradient targets the \emph{mixed-policy return} $J^{\mathrm{mix}}(\theta)$, the quantity training rollouts realize, but the deployment-relevant quantity is the \emph{main-policy return} $J(\theta)$, what $\piMain$ earns alone, which training never samples directly. Optimizing one while caring about the other is only safe if the two cannot drift far apart, so this section bounds $|J^{\mathrm{mix}} - J|$ and shows it is exactly zero under reading (R-A) below.

The two returns differ only on the steps where $\piMain$ acts but $\piMixed$ would have called the recovery policy, that is, the steps where $\piMain$ alone leaves the safe region. The gap is therefore controlled by how reliably $\piMain$ keeps itself inside $\safeR$, which we make precise with a one-step invariance property.

\begin{assumption}[Approximate $\piMain$-invariance of $\safeR$]\label{asm:invariance}
There exists $\eta \ge 0$ such that, starting from any reachable $s \in \safeR$, $\piMain$ stays in $\safeR$ at the next step with probability at least $1-\eta$ uniformly over time.
\end{assumption}
This one-step set-invariance property holds with $\eta = 0$ exactly in the idealized limit (R-A) below and approximately in our setting (R-B).

Summed over a discounted trajectory, this one-step slack $\eta$ controls the out-of-region rate $\betaRate$ (\Cref{sec:background}), which inherits a bound $\betaRate \le \eta\gamma/(1-\gamma) \le \eta/(1-\gamma)$ whenever episodes start inside the region, as ours do (\Cref{app:convergence-proof}).

\begin{theorem}[Objective-gap bound]\label{thm:convergence}
Under \Cref{asm:regularity},
\begin{equation}\label{eq:convergence}
|J(\theta) - J^{\mathrm{mix}}(\theta)| \le \frac{2\,\rmax}{(1-\gamma)^2}\, \betaRate.
\end{equation}
\end{theorem}

The proof (\Cref{app:convergence-proof}) applies the Performance Difference Lemma \citep{kakade2002approximately}. The advantage of $\piMain$ against $\piMixed$ vanishes inside $\safeR$, where the two policies agree, and is bounded by $2\rmax/(1-\gamma)$ outside it, so the gap scales with the main-policy out-of-region rate $\betaRate$. The bound is non-vacuous only for $\betaRate < 1-\gamma$ ($\eta_\star < (1-\gamma)^2$ in \Cref{cor:fixed-point}); we read it structurally (\Cref{app:thm2-check}).

How far \cref{eq:convergence} can be pushed as the curriculum radius $\diter$ grows to $\dmax$ depends on the largest safe region it reaches, $\safeR(\dmax)$, which admits two readings.

\begin{corollary}[Conditional fixed-point]\label{cor:fixed-point}
\textbf{(Exact, R-A.)} If $\safeR(\dmax)$ covers the reachable state space, then at $\diter = \dmax$, $\piMain \equiv \piMixed$ on the reachable support of $\pTraj$, $\beta(\theta,\dmax)=0$, and $J^{\mathrm{mix}}(\theta) = J(\theta)$ exactly. \textbf{(Approximate, R-B.)} Otherwise $|J(\theta)-J^{\mathrm{mix}}(\theta)| \le 2\rmax \eta_\star / (1-\gamma)^3$, where $\eta_\star = \eta(\theta,\dmax) \ge 0$ is the env-dependent invariance slack that (R-B) does not force to zero.
\end{corollary}

The proof specializes \Cref{thm:convergence} to the two readings of $\safeR(\dmax)$ (\Cref{app:fixedpoint-proof}).
Our setting is (R-B): $\safeR(\dmax)$ is a hand-set hyperparameter strictly inside the reachable set, so the bound is approximate, the theory no longer forcing $\eta = 0$, though the curriculum can still drive it low in practice. \Cref{thm:convergence} then says any reduction in $\betaRate$ tightens the gap proportionally, which the linear schedule $\diter\!:\!d_0 \to \dmax$ achieves, with the observable $\alphaRate$ as its diagnostic (\Cref{app:thm2-check}). \Cref{thm:convergence} does not bound falls, and the falls still possible under (R-B) are what the compatibility regularizer (\Cref{sec:compat}) targets. \Cref{app:obs-bound} sketches why $\betaRate$ is hard to bound from \emph{observable} quantities in continuous-action settings.

\subsection{Analytic recovery value}
\label{sec:value}

The masked gradient fixes the policy gradient, but it leaves a second error untouched: the critic.
A learned critic trained over the recovery segments bootstraps its value targets \emph{through} them, so even with a corrected gradient the value-target error persists; empirically, the learned-$V$ variant under-performs every other variant on \envHC. %
Replacing the learned critic \revised{over recovery segments} with the analytic value derived below removes this error\revised{; \Cref{sec:ablations} and \Cref{app:crossed-cells} give the per-component attribution, train-side and recovery-disabled}.

When dynamics and recovery are both deterministic, the practically common case (\ac{MPC}, greedy \ac{SAC}), the value at a recovery-triggering state has a closed form: the recovery segment unrolls along a single fixed path, so its contribution is the return along that path rather than something the critic must learn.

\begin{proposition}[Analytic recovery value]\label{prop:analytic}
Suppose dynamics $P$ and recovery $\piRec$ are both deterministic. Let $s_t$ be a recovery-triggering state, $k \ge 1$ the number of recovery steps until either re-entry into $\safeR$ at $t+k$ (\emph{success}) or \revised{episode termination at $t+k-1$, the segment's last step} (\emph{failure}), and $G_{t,k} = \sum_{j=0}^{k-1} \gamma^j r_{t+j}$ the realized segment return. Then
\begin{equation}\label{eq:analytic-V}
\Vmix(s_t) = \begin{cases} G_{t,k} + \gamma^k\, \Vmix(s_{t+k}) & \text{success}, \\ G_{t,k} & \text{failure}. \end{cases}
\end{equation}
Under our convention that recovery accrues zero per-step task reward and failure incurs a one-time terminal failure reward $r_{\mathrm{term}}$, the segment return $G_{t,k}$ is $0$ on success and $\gamma^{k-1} r_{\mathrm{term}}$ on failure. In practice we bootstrap the re-entry value with the learned critic $V_\theta$, writing $\gamma^k V_\theta(s_{t+k})$ into the buffer on success (\Cref{alg:safeexplorer}). This carries no Monte-Carlo variance under deterministic $P$ and $\piRec$.
\end{proposition}

The proof is direct (\Cref{app:analytic-proof}):
under deterministic dynamics and recovery, the post-trigger trajectory is a single deterministic path of length $k$, so its return is the realized return.%
\footnote{\Cref{thm:gradient,thm:convergence} hold for any bounded reward, covering this training-signal convention; the induced train-to-deploy return gap is controlled by the same $\betaRate$ (\Cref{app:proofs}). Reported returns always use the unmodified task reward (\Cref{sec:exp-setup}).}

\revised{We write \cref{eq:analytic-V} into the rollout buffer at \emph{every} step of each recovery segment (intermediate recovery rewards stored as zero), so intra-segment TD residuals vanish and $\lambda$-\ac{GAE} telescopes through the segment. \Cref{lem:induced-target} (\Cref{app:induced-target}) derives the target this induces: under deterministic dynamics it is exact up to the critic residual at re-entry, and under stochastic dynamics it becomes a one-sample plug-in (\Cref{app:analytic-stochastic}).}

\subsection{The compatibility regularizer}
\label{sec:compat}

The masked gradient \revised{removes the specific bias that recovery attribution introduces into the on-policy update (\Cref{thm:gradient})} and the analytic value \revised{supplies correct targets over recovery segments,} but \revised{removing that bias} alone does not make the agent self-sufficient: neither component pulls $\piMain$ toward behaving like recovery at unsafe states. This matters precisely when recovery is unreliable (\Cref{subsec:q2}): \revised{an update that is unbiased for $J^{\mathrm{mix}}$} over repeated failed recoveries teaches the policy nothing about avoiding those states unaided. We close this gap with a lightweight imitation-from-recovery loss with an outcome-conditioned weight. As an imitation signal rather than an importance correction, it needs no recovery density and covers deterministic controllers.
The closest methods, \ac{DAgger} \citep{ross2011dagger} and \ac{JSRL} \citep{uchendu2023jsrl}, do not gate the imitation signal by segment outcome, and self-imitation learning \citep{oh2018sil}, which does, imitates the agent's own past actions rather than an external recovery policy, the distinction we now make precise.

\begin{definition}[Compatibility regularizer]\label{def:compat}
Let $a_t^{\mathrm{rec}} = \piRec(s_t)$ be the recovery action executed at unsafe step $t$. The trajectory-level compatibility score is
\begin{equation}\label{eq:Ctheta}
\Ctheta = \prod_{t : s_t \notin \safeR} \piMain(a^{\mathrm{rec}}_t \mid s_t),
\quad
\log \Ctheta = \sum_{t : s_t \notin \safeR} \log \piMain(a^{\mathrm{rec}}_t \mid s_t).
\end{equation}
\end{definition}

A high $\Ctheta$ measures exactly the agreement we want, $\piMain$ assigning high likelihood to the recovery actions at the states where it intervened. Raising it indiscriminately is harmful: a failed recovery (segment ends in termination) is exactly what $\piMain$ should \emph{not} imitate. We therefore gate the imitation signal by the realized outcome of each segment. Let $\sigma_k \in \{0,1\}$ indicate success of recovery segment $k$ (re-entry into $\safeR$), and $\seg(t)$ the segment containing step $t$. Reading success off the realized rollout keeps the gate applicable whether recovery is deterministic or stochastic. The \emph{hard outcome gate} is
\begin{equation}\label{eq:hard-gate}
w_t^{\mathrm{hard}} = \sigma_{\seg(t)}, \quad t : s_t \notin \safeR.
\end{equation}
Two boundary cases follow the same reading: a segment cut off by rollout truncation counts as a tentative success with a bootstrapped value, and an immediate re-exit opens a new segment. A signed soft variant based on per-step value change is ablated in \Cref{app:soft-gate}; the hard gate dominates empirically, so $w_t = w_t^{\mathrm{hard}}$ throughout unless noted.

Putting the score and its gate together, the regularizer is the outcome-gated form of $-\log \Ctheta$ from \cref{eq:Ctheta}, added to the \ac{PPO} loss with coefficient $\lambdaCompat$:
\begin{equation}\label{eq:compat-loss}
\compatLoss(\theta) = -\, \frac{\lambdaCompat}{N_{\mathrm{rec}}} \sum_{t : s_t \notin \safeR} w_t \, \log \piMain(a^{\mathrm{rec}}_t \mid s_t),
\quad
L(\theta) = \PPOloss_{\mathrm{safe}} + c_v\, \valueLoss + \compatLoss.
\end{equation}
Here $N_{\mathrm{rec}}$ counts the recovery-controlled minibatch steps (\Cref{app:notation}; per-segment normalization ablated in \Cref{app:compat-norm}), and $\valueLoss$ is the standard \ac{PPO} critic loss with coefficient $c_v$ (entropy coefficient $c_e = 0$ by default; \Cref{app:algorithm}).

$\compatLoss$ is a supervised auxiliary term, deliberately outside the unbiased $J^{\mathrm{mix}}$ gradient: it trades a tunable amount of bias, scaled by $\lambdaCompat$, for a signal the unbiased components cannot supply, pulling $\piMain$ toward the actions of \emph{successful} recovery segments.
That trade pays off exactly where it should: the term costs a little reward on \envHC but adds substantially on \envGo and \envAnt (\revised{$+1036$ and $+2524$} reward over the variant without the regularizer, \Cref{sec:ablations}).

The gate is all that separates this term from standard imitation: always on ($\sigma_k \equiv 1$), $\compatLoss$ reduces to \ac{DAgger}-style \citep{ross2011dagger} cloning on the recovery-controlled states, $\piRec$ as teacher (\Cref{prop:compat-special-cases}, \Cref{app:compat-proof}); the gate echoes the Q-filter of \citet{nair2018demonstrations}, reading a realized segment outcome instead of a learned value estimate.

\section{Algorithm}
\label{sec:algorithm}

\methodName{} assembles the three algorithmic components developed in the previous sections, the masked policy gradient (\Cref{thm:gradient}, \Cref{cor:ppo}), the analytic recovery value (\Cref{prop:analytic}), and the hard outcome-gated compatibility regularizer (\Cref{sec:compat}), into a single \ac{PPO} update (\Cref{alg:safeexplorer}); the fourth component, the objective-gap bound (\Cref{thm:convergence}), needs no implementation beyond the curriculum below.
Because each component intervenes only where the recovery policy acts, the resulting algorithm is a small edit to a standard \ac{PPO} implementation: the changes touch only the rollout's recovery branch and the loss, everything else is unchanged. That locality keeps the method cheap: the added cost is one recovery-policy forward pass per triggered step plus segment bookkeeping (runtimes in \Cref{app:hyperparams}).

This leaves one moving part, the safe-region radius that decides when recovery is triggered, which we anneal over training with a linear curriculum \citep{10.1145/1553374.1553380,JMLR:v21:20-212}: $\diter(u) = d_0 + \frac{u-1}{N}\,\dmax$, growing the region from a tight $d_0$ to the per-environment $\dmax$ over the $N$ updates ($d_0$, $\dmax$, $N$ in \Cref{app:hyperparams}).
\Cref{app:curr-shape} ablates logarithmic, step, and constant schedules on \envAnt, our least reliable recovery policy, where the linear schedule attains the highest reward. This annealing is the practical mechanism by which the recovery rate $\alphaRate$ falls over training (\Cref{app:thm2-check}), an empirical diagnostic for the gap-tightening predicted by \Cref{thm:convergence}.

\section{Experimental Setup}
\label{sec:exp-setup}

Our experiments evaluate three claims. The first is the \emph{safety-reward trade-off} of \methodName{} against standard \ac{PPO}, recovery-using baselines, and \ac{CMDP} baselines, judged on both falls and reward. The second is the \emph{unreliable-recovery regime} that most motivates the method, where the recovery policy itself often fails and the trade-off is hardest to win. The third asks \emph{which component} drives the gains: the unbiased masked gradient (\Cref{thm:gradient}), the analytic recovery value (\Cref{prop:analytic}), or the outcome-gated regularizer (\Cref{def:compat}). %
A single metric adjudicates all three: the number of training falls a method incurs before first reaching task success, where success is \revised{reaching a reward threshold, reported across a grid of thresholds rather than at a single one (defined under \emph{Evaluation})}. \Cref{tab:main-results} therefore reports the underlying rewards and fall counts separately. \revised{All safety results are simulated and conditional on an already available recovery controller, with the end-to-end accounting in \Cref{app:pretraining-accounting}; the instantiation additionally requires a per-environment distance $\mathcal{D}$, a per-environment $\dmax$, and a separately pretrained recovery policy (\Cref{app:hyperparams,app:env}).}

\noindent \textbf{Environments.} We evaluate on three continuous-control locomotion environments: \envHC and \envAnt (custom MuJoCo \citep{todorov2012mujoco}/Gymnasium \citep{towers_gymnasium_2023} variants with a configurable safe-region indicator; full specification in \Cref{app:env}), and \envGo, a Unitree Go1 quadruped built on the \texttt{mujoco\_menagerie} model \citep{menagerie2022mujoco} with a velocity-tracking reward adapted from \texttt{legged\_gym} \citep{rudin2022learningwalkminutesusing}.
The mechanism is identical across all three: the distance $\mathcal{D}$ defining $\safeR(d)$ is a $(z, \mathrm{tilt})$ task-space distance to the nominal upright pose (exact formulas in \Cref{app:safe-region-distance}), and each environment keeps its standard locomotion reward, zeroed only during recovery (\Cref{sec:value}). The environments differ in how hard they make safety: \envHC (planar, strong recovery) sets the baseline, \envAnt (3D, weaker recovery) tests whether the compatibility regularizer compensates for an unreliable fall-back, and \envGo (12-\ac{DoF} quadruped) scales the claim to a higher-dimensional, multi-contact platform.

\noindent \textbf{Recovery policies.} Each environment uses a separately pre-trained \ac{SAC} \citep{pmlr-v80-haarnoja18b} actor as $\piRec$: a stand-and-stabilize objective on its own environment, task terms disabled, no demonstrations, aggressively randomized resets (recipes in \Cref{app:recovery-training}). At rollout time it takes the \ac{SAC} mean action, deterministic, so it exercises the no-density case the masked gradient is built for (\Cref{thm:gradient}).
The \envAnt recovery policy is meaningfully weaker than the others, the most informative axis of variation: its strong control-magnitude penalty caps how aggressively it can correct, and the harder 3D stabilization needs those corrections (\Cref{app:env}).

The remaining choice is the radius $\dmax$, the one quantity a skeptic might suspect was tuned. Unbiasedness holds for any $\dmax$ (\Cref{thm:gradient}), so no radius can manufacture statistical validity, though it still shapes results through the gap bound (\Cref{thm:convergence}) and the curriculum. We set each $\dmax$ (\Cref{app:hyperparams}) by a qualitative criterion, that the curriculum lowers the observable intervention rate (\Cref{app:thm2-check}), without a quantitative search. The fixed-$d$ sweep (\Cref{tab:fixed-d}, \Cref{app:fixed-d}) shows sensitivity to a \emph{constant} radius that the curriculum mostly recovers (\Cref{app:curr-shape}); no systematic sweep varies $\dmax$ under the curriculum, a limitation (\Cref{app:extra-results}).
\revised{Every reported reduction is therefore tied to the trigger radius it was run at, and every table whose runs use a radius other than the defaults of \Cref{app:hyperparams} prints that radius; the radius-dependence case study is in \Cref{app:walker2d}.}

\noindent \textbf{Variants compared.}
\methodName{} (\Cref{alg:safeexplorer}) is compared against three external baselines spanning the established alternatives identified in \Cref{sec:related}. Standard \ac{PPO} is the no-recovery extreme.
Recovery RL \revised{(action-relabeling port)} \citep{thananjeyan2021recoveryrlsafereinforcement} and Safe Legged \revised{(reward-shaping port)} \citep{yang2022safereinforcementlearninglegged} are \emph{on-policy ports} of the two published recovery-using alternatives: both run inside the same \ac{PPO} loop, sharing \methodName{}'s safe-region trigger and frozen \ac{SAC} recovery policy, and each isolates its published method's data-handling rule at recovery steps, action relabeling for Recovery RL versus relabeling plus a fixed reward penalty of $1$ for Safe Legged (exact rules and deviations in \Cref{app:baseline-ports}). \revised{The comparison set is deliberately on-policy under this shared protocol, so the data-handling rule at recovery steps is the only varying factor; a native off-policy entrant would vary optimizer, replay, and network capacity at once, which is why Recovery RL enters as a port rather than in its off-policy form.}
A controlled ablation isolating each component of \methodName{} is reported in \Cref{sec:ablations}; additional soft-gate variants are in \Cref{app:soft-variants}.

\noindent \textbf{\ac{CMDP} baselines.} To position \methodName{} against the standard no-recovery alternative, we compare to two \ac{CMDP} solvers, \ac{CPO} \citep{achiam2017constrained} and \ac{PPO}-Lagrangian \citep{ray2019benchmarking,stooke2020lagrangian}, which act with a single policy and encode safety as a constraint on an indicator cost (1 per fall).
We run both via OmniSafe \citep{ji2023omnisafe} at settings matched to \methodName{}, except the solvers' default observation and reward normalization, which \methodName{} does not use, favoring the baselines (\Cref{app:cmdp-feasibility}). \Cref{sec:exp-cmdp} analyzes their constraint feasibility.

\noindent \textbf{Training.}
We train each variant for $8$M environment steps on \envHC and \envAnt and $16$M on \envGo (its higher-dimensional action space needs a longer horizon)\revised{; all main results use ten seeds, and each appendix cell prints its own count.} Final reward is the per-episode mean of the last $10\%$ of training, computed from the unmodified task reward, not the recovery-zeroed learning signal; falls are the cumulative count over the full training run, a fall being a per-environment unhealthy termination (\Cref{tab:env-summary}), never a time-limit truncation. \revised{We call this pair the \emph{train-side} numbers, to distinguish them from the recovery-disabled evaluation protocol, which freezes the final policy and rolls it out with the recovery branch switched off (\Cref{subsec:q2}, \Cref{app:recovery-disabled}).} Shared \ac{PPO} hyperparameters follow \revised{the CleanRL reference values} \citep{JMLR:v23:21-1342} and are listed in \Cref{tab:hyperparams}.

\noindent \textbf{Evaluation.}
Tables report mean $\pm$ \revised{sample} standard deviation over seeds; plots aggregate \revised{by interquartile mean (IQM) with} $95\%$ stratified-bootstrap confidence intervals \citep{agarwal2021deep}\revised{; all cells use every seed in their batch, no post-hoc exclusions. We define \emph{falls-to-success} as the number of training falls before the agent first reaches a reward threshold. Specifically, falls-to-success is the cumulative fall count up to the first point where a trailing mean of the reward curve crosses the threshold and stays at or above it for the full window. Seeds that never reach the threshold within the training budget are not excluded but right-censored; each such seed counts as unsolved in the solve fraction $k/n$, the $k$ of $n$ seeds that reach the threshold, and enters a censored all-seeds lower bound. Over the seeds that succeed, falls-to-success is aggregated by IQM with a percentile bootstrap over seeds. We report falls-to-success at seven thresholds per environment, spread around a fixed success bar of $4800$ reward on \envHC, $2250$ on \envAnt, and $4500$ on \envGo{}, at which the tabulated numbers are read (\Cref{fig:profile}). The window size, the interpolation rule, and the remaining fine print are in \Cref{app:fixed-bar-detail}.}

\section{Main Results}
\label{sec:exp-main}

\subsection{\texorpdfstring{\revised{The safety-reward trade-off}}{The safety-reward trade-off}}
\label{subsec:tradeoff}

\begin{figure}[t]\revisedfloat
\centering
\revisedgraphics[width=\linewidth]{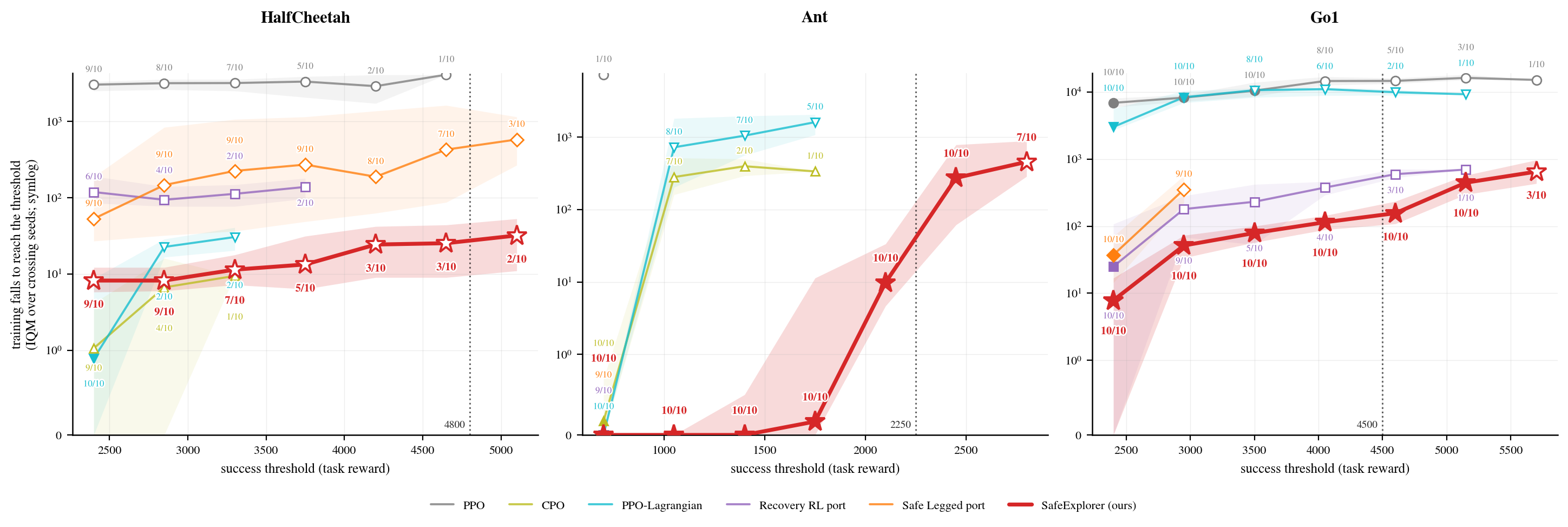}
\caption{\revised{\textbf{Falls-to-success across the threshold range.} Cumulative training falls at the first persistent crossing of each success threshold, aggregated as an interquartile mean over the $k$ of $n$ seeds that reach it (the count is printed at each marker, which is hollow when $k < n$; a missing marker means no seed reaches it), with $95\%$ percentile-bootstrap bands over seeds; symlog vertical axis. The dotted line is the fixed threshold of \Cref{sec:exp-setup}. Lower is cheaper, further right is more capable.}}
\label{fig:profile}
\end{figure}

Addressing our first claim, we compare all methods on the three \revised{wide-band} environments: \Cref{fig:profile} reports falls-to-success \revised{across a grid of success thresholds,} with the per-method solve counts $k/n$, \Cref{tab:main-results} the separate final-reward and total-training-fall figures (different quantities, since total falls span the whole run, not just the climb to success), and \Cref{fig:headline-iqm} (\Cref{app:extra-results}) the per-metric rliable interquartile-mean view. %
\revised{Under the fixed-bar metric of \Cref{sec:exp-setup}, \methodName{} solves $10/10$ on \envGo{} and reaches the \envAnt{} bar on all ten seeds, where no baseline of \Cref{tab:main-results} reaches it on any seed and the unmasked ablation variant on $2$ of $10$; on \envHC{} the bar sits above the mean reward of every method in \Cref{tab:main-results} and solve rates are variance-driven, so \Cref{tab:main-results} instead shows \methodName{} at $3855\pm958$ reward against the constrained solvers' $2768$ and $2743$, at fall counts in the same range ($15.3$ against $17$ and $9$). In the lifetime view (\Cref{tab:main-results}) it cuts total training falls over \ac{PPO} by $21\times$ to roughly $255\times$ at matching or better reward; \Cref{app:fixed-bar-detail} carries the ratios with intervals, the solve counts, and the every-seed comparisons.}

\revised{Across the profile, no method reaches a higher solve rate at a lower cost than \methodName{} at any plotted threshold on \envAnt{} or \envGo, nor on \envHC{} from $2850$ upward. At the lowest plotted \envHC{} threshold ($2400$, under $0.6\times$ the best mean reward in \Cref{tab:main-results}), \ac{PPO}-Lagrangian reaches $10/10$ at one fall and \ac{CPO} $9/10$ at about one, against \methodName{}'s $9/10$ at eight.}

\begin{table}[ht]
\centering
\caption{\revised{Final reward (same unmodified task reward for all methods) and total training falls, mean $\pm$ sample standard deviation over $10$ seeds at full budget: the body batch for recovery-using rows, the matched-budget OmniSafe batch (cost $1$ per fall, limit $0.05$) for \ac{CMDP} rows; recovery-using rewards are training-time mixed-policy returns (recovery-disabled evaluation in \Cref{tab:recovery-disabled-all}). Batch provenance and caveats in \Cref{app:fixed-bar-detail}.}}
\label{tab:main-results}
\setlength{\abovecaptionskip}{4pt}
\setlength{\belowcaptionskip}{3pt}
\scriptsize
\setlength{\tabcolsep}{3pt}
\renewcommand{\arraystretch}{0.92}
\begin{tabular}{l rr rr rr}
\toprule
& \multicolumn{2}{c}{\envHC} & \multicolumn{2}{c}{\envAnt} & \multicolumn{2}{c}{\envGo} \\
\cmidrule(lr){2-3}\cmidrule(lr){4-5}\cmidrule(lr){6-7}
Method & Reward~$\uparrow$ & Falls~$\downarrow$ & Reward~$\uparrow$ & Falls~$\downarrow$ & Reward~$\uparrow$ & Falls~$\downarrow$ \\
\midrule
\ac{PPO} (no recovery) & \revised{$3462\pm924$} & \revised{$3899\pm1573$} & \revised{$259\pm100$} & \revised{$37600\pm11202$} & \revised{$3915\pm603$} & \revised{$17458\pm1950$} \\
Recovery RL \revised{(action-relabeling port)} & \revised{$2209\pm1404$} & \revised{$1298\pm2492$} & \revised{$76\pm14$} & \revised{$51987\pm7168$} & \revised{$3332\pm949$} & \revised{$1753\pm807$} \\
Safe Legged \revised{(reward-shaping port)} & \revised{$\mathbf{4117\pm964}$} & \revised{$1135\pm987$} & \revised{$70\pm14$} & \revised{$55202\pm9929$} & \revised{$1838\pm155$} & \revised{$5410\pm1228$} \\
\ac{CPO} & \revised{$2768\pm343$} & \revised{$17\pm24$} & \revised{$1153\pm256$} & \revised{$\mathbf{838\pm152}$} & \revised{$364\pm51$} & \revised{$37364\pm4787$} \\
\ac{PPO}-Lagrangian & \revised{$2743\pm355$} & \revised{$9\pm16$} & \revised{$1530\pm368$} & \revised{$3586\pm2515$} & \revised{$4018\pm737$} & \revised{$12859\pm1649$} \\
\midrule
\textbf{\methodName{} (ours)} & \revised{$3855\pm958$} & \revised{$15.3\pm14.7$} & \revised{$\mathbf{2698\pm237}$} & \revised{$866\pm718$} & \revised{$\mathbf{5509\pm155}$} & \revised{$\mathbf{830\pm216}$} \\
\bottomrule
\end{tabular}
\vspace{-4pt}
\end{table}

Reading the environments one at a time, the same ranking holds but the reason shifts with how hard safety is. On \envHC, where the recovery policy is reliable, \ac{PPO} is competitive on reward but pays for it in falls, while \methodName{} matches that reward at a fraction of the falls. On \envGo, the larger action space and longer horizon make \ac{PPO} highly unsafe, while the Recovery RL and Safe Legged \revised{ports} control falls only at substantial reward cost and \methodName{} attains both the highest reward and the lowest fall count. \envAnt is the informative exception: the ranking still favors \methodName{}, but it behaves differently because the recovery policy itself is unreliable, the case we isolate next in \Cref{subsec:q2}.

This dependence on recovery is exactly what \Cref{thm:convergence} predicts: \Cref{app:thm2-check} confirms on body-matched \revised{10-seed} runs that the gap $J^{\mathrm{mix}}(\theta_k) - J(\theta_k)$ tracks the recovery rate $\alphaRate$, closing to seed noise on \envHC and \envAnt, where the policy becomes self-stable, and settling at a proportional residual on \envGo, where it does not. Per-environment learning curves are deferred to \Cref{fig:curves} in \Cref{app:extra-results}.

\subsection{The unreliable-recovery regime: \envAnt}
\label{subsec:q2}

This subsection addresses our second claim, the unreliable-recovery regime: \Cref{subsec:tradeoff} cannot show \emph{why} \methodName{} remains effective where the recovery policy itself fails, so we isolate the one environment where that happens, \envAnt, whose \ac{SAC} recovery is markedly less reliable. Replaying the recovery from the states a trained policy actually drives it into, it completes the segment only $70$ to $91\%$ of the time on \envAnt, against $97$ to $100\%$ on \envHC and \envGo (one representative run per environment; \Cref{tab:operational-recovery}).
On \envAnt, \ac{PPO} and the Recovery RL and Safe Legged \revised{ports} all fail ($\revised{38}$k to $\revised{55}$k falls, reward $\le \revised{259}$\revised{, $10$ seeds}). The mechanism is a feedback cycle, because when the recovery policy is unreliable, any method that relies on it \emph{reactively} keeps re-entering unsafe states, recovery keeps failing, and the agent never learns to avoid them unaided.
\methodName{} breaks that cycle and reaches the \revised{fixed $2250$ bar on all ten evaluated seeds,} at an order of magnitude fewer falls than any recovery-using baseline; \revised{no baseline of \Cref{tab:main-results} reaches that bar on any seed}. The ablations pinpoint why: \revised{every variant that drops the hard outcome gate reaches the bar on at most $2$ of $10$ seeds} (\Cref{sec:ablations}); the gate, not the masked gradient or the analytic value, supplies the fallback skill when recovery is unreliable.

\revised{\noindent \textbf{Does the policy stand on its own?} Evaluated with the recovery branch disabled, \methodName{} on \envAnt{} retains $3124 \pm 358$ reward at $90 \pm 72$ evaluation falls, against the next-best variant's $1215 \pm 406$ at $460 \pm 189$, \ac{PPO}'s $299 \pm 95$ at $939 \pm 23$, and the ports' $83 \pm 69$ and $70 \pm 36$. On \envHC{} the leading variants' recovery-disabled rewards are within seed noise of one another; on \envGo{} \methodName{} retains the highest recovery-disabled reward in both batches while remaining below its mixed-policy reward, so it is not fully self-sufficient there. Recovery-disabled reward and falls for the remaining methods and environments are in \Cref{tab:recovery-disabled-all}.}

\subsection{Comparison to \texorpdfstring{\ac{CMDP}}{CMDP} baselines}
\label{sec:exp-cmdp}

This subsection completes our first claim by comparing \methodName{} against the \ac{CMDP} alternative to a recovery policy, which casts safety as a constraint and solves the resulting \ac{CMDP} rather than intervening on the trajectory. \ac{CPO} and \ac{PPO}-Lagrangian, the trust-region and Lagrangian solvers, appear in \Cref{tab:main-results} and \Cref{fig:headline} (\Cref{app:extra-results}) with a cost of $1$ per fall, at the matched settings of \Cref{sec:exp-setup}. %
\methodName{} attains the highest reward of the three in every environment\revised{. On falls, \methodName{} has the fewest on \envGo; on \envAnt{} its count is within seed noise of \ac{CPO}'s, so the claim there is the $2.3\times$ reward lead rather than a falls lead; on \envHC{} all three are low, within seed noise (\Cref{tab:main-results}, \Cref{app:cmdp-feasibility}).} The more revealing comparison is constraint feasibility: whether each solver keeps the falls it promised to.

The constrained solvers' poor safety on the hard environments is not under-tuning. Each \ac{CMDP} solver is configured with a cost limit on falls, as its formulation requires, and meets it on \envHC where staying safe is easy, but \emph{violates its own limit} on \envAnt and \envGo by up to $19\times$: \ac{CPO} converges to an episodic cost of $0.96$ on \envGo against the $0.05$ budget it was set, and even incurs more \envGo falls than unconstrained \ac{PPO} at a tenth of the reward (\Cref{app:cmdp-feasibility}). Constraint satisfaction thus holds where safety is easy and fails exactly where the safe-region intervention is needed.
The difference is one of mechanism: the \ac{CMDP} baselines act alone and only penalize the violation after the fact, whereas \methodName{} acts on the safe region directly, which is why it keeps falls low where the constrained solvers cannot.

\section{Ablation Analysis}
\label{sec:ablations}

This section addresses our third claim, which component drives \methodName{}'s gains over prior recovery-using methods. We isolate the three components that change the training update with a ladder of controlled variants. All variants share the same recovery policy and curriculum and differ only in how recovery transitions are handled. The ladder starts from the biased update those prior methods use, which keeps recovery transitions as if the main policy had generated them (a learned-critic on-policy gradient), and turns on one component at a time. \emph{Unmasked PG} is that biased baseline. The \emph{masked, learned $V$} variant switches to the unbiased masked policy gradient (\Cref{thm:gradient}), the \emph{masked, analytic $V$} variant adds the analytic recovery value (\Cref{prop:analytic}), and \methodName{} adds the hard outcome gate (\cref{eq:hard-gate}). \ac{PPO} (no recovery) anchors the no-intervention end. \Cref{fig:headline} (\Cref{app:extra-results}) ranks these variants by falls-to-success and \Cref{fig:ablation-pareto} (\Cref{app:ablation-ladder}) gives the complementary reward-vs-falls view. %
\revised{Masking removes this bias at the cost of dropping every recovery step from the gradient; the analytic value and the outcome gate are what pay that cost back, so its payoff appears in combination with them rather than in isolation (\Cref{app:crossed-cells}).}

The first component, the masked policy gradient (\Cref{thm:gradient}), is the foundation. Its biased alternative, Unmasked PG,
reaches competitive or higher reward on \envHC and \envGo, but pays \revised{$11\times$, $5.2\times$, and $3.4\times$} more falls than \methodName{} on \envHC, \envAnt, and \envGo (\Cref{tab:soft-variants}), and is ill-posed for deterministic recovery. Masking removes that bias, the prerequisite for the next two components.

Removing the bias exposes a second problem: the unbiased gradient is sparse at the safe-region boundary, so the masked, learned-$V$ variant alone regresses to $1982$ reward on \envHC, the lowest of any variant (\Cref{tab:soft-variants}). The analytic recovery value (\Cref{prop:analytic}) supplies dense, correct targets exactly there, lifting \envHC reward to $3870$, the largest single-component gain on that environment ($+1888$) \revised{on the train-side 5-seed ladder; the same configuration's $10$-seed cell reads $2608\pm1090$ (\Cref{tab:crossed-cells}), so the size of the gain depends on the seed batch. Under recovery-disabled evaluation that gain does not transfer. On \envHC{} and \envAnt{} the analytic value is within seed noise both with and without the gate, and the outcome-gated imitation does most of the work there; on \envGo{} the attribution runs the other way, with the analytic value the largest single-component contribution at the body operating point. The components interact rather than add, and which one carries an environment changes with the environment (\Cref{app:crossed-cells}, \Cref{tab:recovery-disabled-all})}.

The top variant adds the hard outcome gate (\methodName{}), which \revised{leads the ladder on safety in all three environments (on \envGo{} nominally, inside seed noise), and} supplies the largest single-component reward gains on \envGo and \envAnt (\revised{$+1036$ and $+2524$}; \Cref{tab:soft-variants}), where the analytic value alone leaves reward flat; it is the only ladder variant to reach the \revised{fixed \envAnt{}} success bar \revised{on every seed}, the strongest evidence for the imitation reading of $\Ctheta$ (\Cref{sec:compat}).

In the $\lambdaCompat$ sweep (\Cref{tab:lambda-sweep}; \envHC and \envAnt only, the default transferred to \envGo untested), \envHC reward stays at \revised{$4770$ to $5886$} across $[10^{-3}, 1]$, peaking at $10^{-2}$, but falls are minimized at the $10^{-3}$ default (\revised{$33$}) and rise \revised{$4.3\times$} within a decade (\revised{$142$} at $10^{-2}$); we choose the default on falls, not reward. \Cref{app:extra-results} reports the rest:
robustness to hyperparameters and noise (\Cref{app:lambda-sensitivity,app:fixed-d,app:curr-shape,app:compat-norm,app:action-noise}), the soft-gate variants (\Cref{app:soft-variants}), compute-scaling (\Cref{app:16m-scaling}), the full ladder (\Cref{app:ablation-ladder}), and a failure-mode catalog (\Cref{app:failure-modes}).

\section{Discussion and Limitations}
\label{sec:discussion}

\noindent \textbf{When the method helps, and when it does not.} \methodName{}'s reward advantage is smallest when \ac{PPO} is already safe and recovery is reliable, and largest in the opposite regimes: weak recovery (\envAnt), dangerous unaided \ac{PPO} (\envGo). The hard outcome gate makes this work: a soft gate imitates failed recovery actions and degrades both metrics (\Cref{app:soft-variants}). We treat the recovery policy as a permanent fallback, not temporary scaffolding: on \envGo it stays mildly engaged ($\alphaRate \approx 0.05$) even at $16$M steps, and keeping it available is what bounds falls there (\Cref{app:thm2-check}); our claim is reduced training falls, not their elimination.
\revised{\Cref{app:pretraining-accounting} gives the end-to-end accounting, with recovery pretraining counted in full; charged to a single training run the reductions read $21.3\times$ (\envHC), $2.0\times$ (\envAnt), and $2.8\times$ (\envGo), rising with checkpoint reuse. Claims are scoped to the simulated environments we run, and reducing falls in simulation is not a real-world safety guarantee (\Cref{app:broader-impact}). Where training falls cannot be tolerated at all, offline \ac{RL} is the competing paradigm (\Cref{sec:related}). We keep the imperfect \envAnt{} recovery deliberately; an imperfect recovery is where the outcome gate earns its keep, but a bad enough one is worse than none, sub-threshold \envWalker{} controllers producing roughly $3\times$ \ac{PPO}'s falls while destroying task learning. Rather than retraining a weak recovery controller, we recommend screening it by its offline catch rate before any task training is spent (\Cref{app:recovery-training}, \Cref{app:walker2d}).}

\noindent \textbf{Assumptions behind the guarantees.} The analytic value (\Cref{prop:analytic}) is exact only under deterministic dynamics and recovery. Under noise it becomes a single-sample estimate, unbiased up to critic error at re-entry, favorable in practice since $\piRec$ is near-deterministic; a multi-sample extension is in \Cref{app:analytic-stochastic}. The objective-gap bound (\Cref{thm:convergence}) is controlled by the unobservable out-of-region rate $\betaRate$, tracked via the observable $\alphaRate$ as a diagnostic (\Cref{app:thm2-check}). A certified observable bound remains open (a deterministic recovery makes the natural simulation-lemma bound vacuous; routes in \Cref{app:obs-bound}).

\noindent \textbf{Generality of the compatibility regularizer.} The gate is also what makes the regularizer general:
without it $\compatLoss$ is \ac{DAgger}-style imitation (\Cref{prop:compat-special-cases}), so outcome-gated imitation should transfer to other teacher-intervention settings, human-in-the-loop \ac{RL} \citep{spencer2020learning} or \ac{JSRL} \citep{uchendu2023jsrl}.

\noindent \textbf{Beyond locomotion.} That transfer argument is theoretical; our evaluation \revised{covers simulated locomotion only}, though the correction applies wherever a fixed recovery policy takes over at unsafe states (a scripted reset in manipulation, an emergency-stop in driving); validating these is future work.
The method also presupposes a pre-trained recovery policy and a hand-designed distance $\mathcal{D}$ (\Cref{app:hyperparams,app:env}).

\section{Conclusion}
\label{sec:conclusion}

We built \methodName{}, an on-policy \ac{RL} algorithm for safety-critical control that, as a single \ac{PPO} update, \revised{cuts training falls on three simulated locomotion environments by $21\times$ to roughly $255\times$ ($10$ seeds)} while matching or beating standard \ac{PPO} on reward.

These gains rest on a theoretical foundation for \ac{RL} with an external recovery policy: %
a masked policy-gradient theorem unbiased for the \revised{gradient of the mixed-policy return $J^{\mathrm{mix}}$} under any recovery, an objective-gap bound with an exact fixed point, an analytic recovery value, and an outcome-gated compatibility regularizer, all instantiated for \emph{safe-region intervention} to minimize training falls. \revised{The unbiasedness claim attaches to the Theorem 1 estimator of $\nabla_\theta J^{\mathrm{mix}}$, not to the full practical update; \Cref{thm:convergence} bridges to $J$.}

\subsubsection*{Acknowledgements}
This work benefited from a grant from the Fonds de recherche du Qu\'ebec (\url{https://doi.org/10.69777/332223}) and from a Borealis AI Fellowship. We also acknowledge funding support from the Natural Sciences and Engineering Research Council of Canada (NSERC) and the Canadian Institute for Advanced Research (CIFAR), and compute support from the Digital Research Alliance of Canada (through Calcul Qu\'ebec and SHARCNET), Mila IDT, and NVIDIA.

\bibliography{references}

\begin{thebibliography}{73}
\providecommand{\natexlab}[1]{#1}
\providecommand{\url}[1]{\texttt{#1}}
\expandafter\ifx\csname urlstyle\endcsname\relax
  \providecommand{\doi}[1]{doi: #1}\else
  \providecommand{\doi}{doi: \begingroup \urlstyle{rm}\Url}\fi

\bibitem[Achiam et~al.(2017)Achiam, Held, Tamar, and
  Abbeel]{achiam2017constrained}
Joshua Achiam, David Held, Aviv Tamar, and Pieter Abbeel.
\newblock Constrained policy optimization.
\newblock In \emph{Proceedings of the 34th International Conference on Machine
  Learning (ICML)}, pp.\  22--31, 2017.

\bibitem[Agarwal et~al.(2022)Agarwal, Kumar, Malik, and
  Pathak]{Agarwal2022LeggedLI}
Ananye Agarwal, Ashish Kumar, Jitendra Malik, and Deepak Pathak.
\newblock Legged locomotion in challenging terrains using egocentric vision.
\newblock In \emph{Conference on Robot Learning}, 2022.

\bibitem[Agarwal et~al.(2021)Agarwal, Schwarzer, Castro, Courville, and
  Bellemare]{agarwal2021deep}
Rishabh Agarwal, Max Schwarzer, Pablo~Samuel Castro, Aaron~C. Courville, and
  Marc~G. Bellemare.
\newblock Deep reinforcement learning at the edge of the statistical precipice.
\newblock \emph{Advances in Neural Information Processing Systems (NeurIPS)},
  34:\penalty0 29304--29320, 2021.

\bibitem[Alshiekh et~al.(2018)Alshiekh, Bloem, Ehlers, K{\"o}nighofer, Niekum,
  and Topcu]{alshiekh2018safe}
Mohammed Alshiekh, Roderick Bloem, R{\"u}diger Ehlers, Bettina K{\"o}nighofer,
  Scott Niekum, and Ufuk Topcu.
\newblock Safe reinforcement learning via shielding.
\newblock In \emph{Proceedings of the 32nd AAAI Conference on Artificial
  Intelligence}, 2018.

\bibitem[Altman(1999)]{altman1999constrained}
Eitan Altman.
\newblock \emph{Constrained {M}arkov Decision Processes}.
\newblock Chapman \& Hall/CRC, 1999.

\bibitem[Ames et~al.(2019)Ames, Coogan, Egerstedt, Notomista, Sreenath, and
  Tabuada]{ames2019cbf}
Aaron~D. Ames, Samuel Coogan, Magnus Egerstedt, Gennaro Notomista, Koushil
  Sreenath, and Paulo Tabuada.
\newblock Control barrier functions: Theory and applications.
\newblock In \emph{Proceedings of the 2019 European Control Conference (ECC)},
  pp.\  3420--3431, 2019.

\bibitem[Bacon et~al.(2017)Bacon, Harb, and Precup]{bacon2017option}
Pierre-Luc Bacon, Jean Harb, and Doina Precup.
\newblock The option-critic architecture.
\newblock In \emph{Proceedings of the 31st AAAI Conference on Artificial
  Intelligence (AAAI)}, pp.\  1726--1734, 2017.

\bibitem[Bengio et~al.(2009)Bengio, Louradour, Collobert, and
  Weston]{10.1145/1553374.1553380}
Yoshua Bengio, J\'{e}r\^{o}me Louradour, Ronan Collobert, and Jason Weston.
\newblock Curriculum learning.
\newblock In \emph{Proceedings of the 26th Annual International Conference on
  Machine Learning}, ICML '09, pp.\  41–48, New York, NY, USA, 2009.
  Association for Computing Machinery.
\newblock ISBN 9781605585161.
\newblock \doi{10.1145/1553374.1553380}.
\newblock URL \url{https://doi.org/10.1145/1553374.1553380}.

\bibitem[Bogdanovic et~al.(2022)Bogdanovic, Khadiv, and
  Righetti]{bogdanovic2022modelfreereinforcementlearningrobust}
Miroslav Bogdanovic, Majid Khadiv, and Ludovic Righetti.
\newblock Model-free reinforcement learning for robust locomotion using
  demonstrations from trajectory optimization.
\newblock \emph{Frontiers in Robotics and AI}, 9, 2022.

\bibitem[Chiu et~al.(2022)Chiu, Sleiman, Mittal, Farshidian, and
  Hutter]{Chiu2022ACM}
Jiawei Chiu, Jean-Pierre Sleiman, Mayank Mittal, Farbod Farshidian, and Marco
  Hutter.
\newblock A collision-free {MPC} for whole-body dynamic locomotion and
  manipulation.
\newblock In \emph{2022 International Conference on Robotics and Automation
  (ICRA)}, pp.\  4686--4693, 2022.

\bibitem[Cochran(1977)]{cochran1977sampling}
William~G. Cochran.
\newblock \emph{Sampling Techniques}.
\newblock John Wiley \& Sons, 3rd edition, 1977.

\bibitem[Dalal et~al.(2018)Dalal, Dvijotham, Vecer{\'{\i}}k, Hester, Paduraru,
  and Tassa]{dalal2018safeexplorationcontinuousaction}
Gal Dalal, Krishnamurthy Dvijotham, Matej Vecer{\'{\i}}k, Todd Hester, Cosmin
  Paduraru, and Yuval Tassa.
\newblock Safe exploration in continuous action spaces.
\newblock \emph{CoRR}, abs/1801.08757, 2018.

\bibitem[Degris et~al.(2012)Degris, White, and Sutton]{degris2012off}
Thomas Degris, Martha White, and Richard~S. Sutton.
\newblock Off-policy actor-critic.
\newblock In \emph{Proceedings of the 29th International Conference on Machine
  Learning (ICML)}, 2012.

\bibitem[Espeholt et~al.(2018)Espeholt, Soyer, Munos, Simonyan, Mnih, Ward,
  Doron, Firoiu, Harley, Dunning, Legg, and Kavukcuoglu]{espeholt2018impala}
Lasse Espeholt, Hubert Soyer, Remi Munos, Karen Simonyan, Volodymyr Mnih, Tom
  Ward, Yotam Doron, Vlad Firoiu, Tim Harley, Iain Dunning, Shane Legg, and
  Koray Kavukcuoglu.
\newblock {IMPALA}: Scalable distributed deep-{RL} with importance weighted
  actor-learner architectures.
\newblock In \emph{Proceedings of the 35th International Conference on Machine
  Learning (ICML)}, pp.\  1407--1416, 2018.

\bibitem[Eysenbach et~al.(2018)Eysenbach, Gu, Ibarz, and
  Levine]{eysenbach2018leavenotrace}
Benjamin Eysenbach, Shixiang Gu, Julian Ibarz, and Sergey Levine.
\newblock Leave no trace: Learning to reset for safe and autonomous
  reinforcement learning.
\newblock In \emph{International Conference on Learning Representations
  (ICLR)}, 2018.

\bibitem[Florensa et~al.(2017)Florensa, Held, Wulfmeier, Zhang, and
  Abbeel]{florensa2017reverse}
Carlos Florensa, David Held, Markus Wulfmeier, Michael Zhang, and Pieter
  Abbeel.
\newblock Reverse curriculum generation for reinforcement learning.
\newblock In \emph{Proceedings of the 1st Annual Conference on Robot Learning
  (CoRL)}, pp.\  482--495, 2017.

\bibitem[Garc{\'i}a \& Fern{\'a}ndez(2015)Garc{\'i}a and
  Fern{\'a}ndez]{garcia2015comprehensive}
Javier Garc{\'i}a and Fernando Fern{\'a}ndez.
\newblock A comprehensive survey on safe reinforcement learning.
\newblock \emph{Journal of Machine Learning Research}, 16\penalty0
  (42):\penalty0 1437--1480, 2015.

\bibitem[Gu et~al.(2017)Gu, Lillicrap, Ghahramani, Turner, Sch{\"o}lkopf, and
  Levine]{gu2017interpolated}
Shixiang Gu, Tim Lillicrap, Zoubin Ghahramani, Richard~E. Turner, Bernhard
  Sch{\"o}lkopf, and Sergey Levine.
\newblock Interpolated policy gradient: Merging on-policy and off-policy
  gradient estimation for deep reinforcement learning.
\newblock In \emph{Advances in Neural Information Processing Systems
  (NeurIPS)}, 2017.

\bibitem[Ha et~al.(2025)Ha, Lee, van~de Panne, Xie, Yu, and
  Khadiv]{ha2025learning}
Sehoon Ha, Joonho Lee, Michiel van~de Panne, Zhaoming Xie, Wenhao Yu, and Majid
  Khadiv.
\newblock Learning-based legged locomotion: State of the art and future
  perspectives.
\newblock \emph{The International Journal of Robotics Research}, 44\penalty0
  (8):\penalty0 1396--1427, 2025.

\bibitem[Haarnoja et~al.(2018)Haarnoja, Zhou, Abbeel, and
  Levine]{pmlr-v80-haarnoja18b}
Tuomas Haarnoja, Aurick Zhou, Pieter Abbeel, and Sergey Levine.
\newblock Soft actor-critic: Off-policy maximum entropy deep reinforcement
  learning with a stochastic actor.
\newblock In Jennifer Dy and Andreas Krause (eds.), \emph{Proceedings of the
  35th International Conference on Machine Learning}, volume~80 of
  \emph{Proceedings of Machine Learning Research}, pp.\  1861--1870. PMLR,
  10--15 Jul 2018.
\newblock URL \url{https://proceedings.mlr.press/v80/haarnoja18b.html}.

\bibitem[Haarnoja et~al.(2019)Haarnoja, Ha, Zhou, Tan, Tucker, and
  Levine]{haarnoja2019learningwalkdeepreinforcement}
Tuomas Haarnoja, Sehoon Ha, Aurick Zhou, Jie Tan, George Tucker, and Sergey
  Levine.
\newblock Learning to walk via deep reinforcement learning.
\newblock In \emph{Robotics: Science and Systems XV}, 2019.
\newblock \doi{10.15607/RSS.2019.XV.011}.

\bibitem[Hasanbeig et~al.(2020)Hasanbeig, Abate, and
  Kroening]{hasanbeig2020cautiousreinforcementlearninglogical}
Mohammadhosein Hasanbeig, Alessandro Abate, and Daniel Kroening.
\newblock Cautious reinforcement learning with logical constraints.
\newblock In Amal El~Fallah Seghrouchni, Gita Sukthankar, Bo~An, and Neil
  Yorke{-}Smith (eds.), \emph{Proceedings of the 19th International Conference
  on Autonomous Agents and Multiagent Systems, {AAMAS} '20, Auckland, New
  Zealand, May 9-13, 2020}, pp.\  483--491. International Foundation for
  Autonomous Agents and Multiagent Systems, 2020.
\newblock \doi{10.5555/3398761.3398821}.

\bibitem[Huang et~al.(2022)Huang, Dossa, Ye, Braga, Chakraborty, Mehta, and
  Ara{\'u}jo]{JMLR:v23:21-1342}
Shengyi Huang, Rousslan Fernand~Julien Dossa, Chang Ye, Jeff Braga, Dipam
  Chakraborty, Kinal Mehta, and Jo{\~a}o G.~M. Ara{\'u}jo.
\newblock {CleanRL}: High-quality single-file implementations of deep
  reinforcement learning algorithms.
\newblock \emph{Journal of Machine Learning Research}, 23\penalty0
  (274):\penalty0 1--18, 2022.
\newblock URL \url{http://jmlr.org/papers/v23/21-1342.html}.

\bibitem[Ibarz et~al.(2021)Ibarz, Tan, Finn, Kalakrishnan, Pastor, and
  Levine]{Ibarz_2021}
Julian Ibarz, Jie Tan, Chelsea Finn, Mrinal Kalakrishnan, Peter Pastor, and
  Sergey Levine.
\newblock How to train your robot with deep reinforcement learning: lessons we
  have learned.
\newblock \emph{The International Journal of Robotics Research}, 40:\penalty0
  698 -- 721, 2021.

\bibitem[Jackson et~al.(2025)Jackson, Berdica, Liesen, Whiteson, and
  Foerster]{jackson2025cleanslate}
Matthew~Thomas Jackson, Uljad Berdica, Jarek Liesen, Shimon Whiteson, and
  Jakob~Nicolaus Foerster.
\newblock A clean slate for offline reinforcement learning.
\newblock In \emph{Advances in Neural Information Processing Systems}, 2025.

\bibitem[Ji et~al.(2024)Ji, Zhou, Zhang, Dai, Pan, Sun, Huang, Geng, Liu, and
  Yang]{ji2023omnisafe}
Jiaming Ji, Jiayi Zhou, Borong Zhang, Juntao Dai, Xuehai Pan, Ruiyang Sun,
  Weidong Huang, Yiran Geng, Mickel Liu, and Yaodong Yang.
\newblock {OmniSafe}: An infrastructure for accelerating safe reinforcement
  learning research.
\newblock \emph{Journal of Machine Learning Research}, 25\penalty0
  (285):\penalty0 1--6, 2024.

\bibitem[Jiang \& Li(2016)Jiang and Li]{jiang2016doubly}
Nan Jiang and Lihong Li.
\newblock Doubly robust off-policy value evaluation for reinforcement learning.
\newblock In \emph{Proceedings of the 33rd International Conference on Machine
  Learning (ICML)}, pp.\  652--661, 2016.

\bibitem[Junges et~al.(2016)Junges, Jansen, Dehnert, Topcu, and
  Katoen]{junges2015safetyconstrainedreinforcementlearningmdps}
Sebastian Junges, Nils Jansen, Christian Dehnert, Ufuk Topcu, and Joost-Pieter
  Katoen.
\newblock Safety-constrained reinforcement learning for {MDPs}.
\newblock In \emph{Tools and Algorithms for the Construction and Analysis of
  Systems (TACAS)}, Lecture Notes in Computer Science, pp.\  130--146, 2016.
\newblock \doi{10.1007/978-3-662-49674-9_8}.

\bibitem[Kakade \& Langford(2002)Kakade and Langford]{kakade2002approximately}
Sham Kakade and John Langford.
\newblock Approximately optimal approximate reinforcement learning.
\newblock In \emph{Proceedings of the 19th International Conference on Machine
  Learning (ICML)}, pp.\  267--274, 2002.

\bibitem[Kang et~al.(2022)Kang, Gradu, Choi, Janner, Tomlin, and
  Levine]{kang2022lyapunovdensitymodelsconstraining}
Katie Kang, Paula Gradu, Jason~J. Choi, Michael Janner, Claire~J. Tomlin, and
  Sergey Levine.
\newblock Lyapunov density models: Constraining distribution shift in
  learning-based control.
\newblock In Kamalika Chaudhuri, Stefanie Jegelka, Le~Song, Csaba
  Szepesv{\'{a}}ri, Gang Niu, and Sivan Sabato (eds.), \emph{International
  Conference on Machine Learning, {ICML} 2022, 17-23 July 2022, Baltimore,
  Maryland, {USA}}, volume 162 of \emph{Proceedings of Machine Learning
  Research}, pp.\  10708--10733. {PMLR}, 2022.

\bibitem[Kearns et~al.(2002)Kearns, Mansour, and Ng]{kearns2002sparse}
Michael Kearns, Yishay Mansour, and Andrew~Y. Ng.
\newblock A sparse sampling algorithm for near-optimal planning in large
  {M}arkov decision processes.
\newblock \emph{Machine Learning}, 49\penalty0 (2--3):\penalty0 193--208, 2002.

\bibitem[Kelly et~al.(2019)Kelly, Sidrane, Driggs-Campbell, and
  Kochenderfer]{kelly2019hgdagger}
Michael Kelly, Chelsea Sidrane, Katherine Driggs-Campbell, and Mykel~J.
  Kochenderfer.
\newblock {HG-DAgger}: Interactive imitation learning with human experts.
\newblock In \emph{2019 International Conference on Robotics and Automation
  (ICRA)}, pp.\  8077--8083, 2019.
\newblock \doi{10.1109/ICRA.2019.8793698}.

\bibitem[Kidambi et~al.(2020)Kidambi, Rajeswaran, Netrapalli, and
  Joachims]{kidambi2020morel}
Rahul Kidambi, Aravind Rajeswaran, Praneeth Netrapalli, and Thorsten Joachims.
\newblock {MOReL}: Model-based offline reinforcement learning.
\newblock In \emph{Advances in Neural Information Processing Systems},
  volume~33, 2020.

\bibitem[Kumar et~al.(2021)Kumar, Fu, Pathak, and
  Malik]{kumar2021rmarapidmotoradaptation}
Ashish Kumar, Zipeng Fu, Deepak Pathak, and Jitendra Malik.
\newblock {RMA:} rapid motor adaptation for legged robots.
\newblock In Dylan~A. Shell, Marc Toussaint, and M.~Ani Hsieh (eds.),
  \emph{Robotics: Science and Systems XVII, Virtual Event, July 12-16, 2021},
  2021.
\newblock \doi{10.15607/RSS.2021.XVII.011}.

\bibitem[Lee et~al.(2019)Lee, Hwangbo, and
  Hutter]{lee2019robustrecoverycontrollerquadrupedal}
Joonho Lee, Jemin Hwangbo, and Marco Hutter.
\newblock Robust recovery controller for a quadrupedal robot using deep
  reinforcement learning.
\newblock \emph{CoRR}, abs/1901.07517, 2019.

\bibitem[Lee et~al.(2020)Lee, Hwangbo, Wellhausen, Koltun, and
  Hutter]{Lee_2020}
Joonho Lee, Jemin Hwangbo, Lorenz Wellhausen, Vladlen Koltun, and Marco Hutter.
\newblock Learning quadrupedal locomotion over challenging terrain.
\newblock \emph{Science Robotics}, 5\penalty0 (47):\penalty0 eabc5986, 2020.
\newblock \doi{10.1126/scirobotics.abc5986}.

\bibitem[Liu et~al.(2024)Liu, Chen, Cheng, Ji, Yang, and Wang]{Liu2024VisualWC}
Minghuan Liu, Zixuan Chen, Xuxin Cheng, Yandong Ji, Ruihan Yang, and Xiaolong
  Wang.
\newblock Visual whole-body control for legged loco-manipulation.
\newblock In \emph{Conference on Robot Learning (CoRL)}, 2024.

\bibitem[Munos et~al.(2016)Munos, Stepleton, Harutyunyan, and
  Bellemare]{munos2016retrace}
R{\'e}mi Munos, Tom Stepleton, Anna Harutyunyan, and Marc~G. Bellemare.
\newblock Safe and efficient off-policy reinforcement learning.
\newblock In \emph{Advances in Neural Information Processing Systems
  (NeurIPS)}, 2016.

\bibitem[Nair et~al.(2018)Nair, McGrew, Andrychowicz, Zaremba, and
  Abbeel]{nair2018demonstrations}
Ashvin Nair, Bob McGrew, Marcin Andrychowicz, Wojciech Zaremba, and Pieter
  Abbeel.
\newblock Overcoming exploration in reinforcement learning with demonstrations.
\newblock In \emph{2018 IEEE International Conference on Robotics and
  Automation (ICRA)}, pp.\  6292--6299, 2018.
\newblock \doi{10.1109/ICRA.2018.8463162}.

\bibitem[Narvekar et~al.(2020)Narvekar, Peng, Leonetti, Sinapov, Taylor, and
  Stone]{JMLR:v21:20-212}
Sanmit Narvekar, Bei Peng, Matteo Leonetti, Jivko Sinapov, Matthew~E. Taylor,
  and Peter Stone.
\newblock Curriculum learning for reinforcement learning domains: A framework
  and survey.
\newblock \emph{Journal of Machine Learning Research}, 21\penalty0
  (181):\penalty0 1--50, 2020.
\newblock URL \url{http://jmlr.org/papers/v21/20-212.html}.

\bibitem[Oh et~al.(2018)Oh, Guo, Singh, and Lee]{oh2018sil}
Junhyuk Oh, Yijie Guo, Satinder Singh, and Honglak Lee.
\newblock Self-imitation learning.
\newblock In \emph{Proceedings of the 35th International Conference on Machine
  Learning (ICML)}, volume~80 of \emph{Proceedings of Machine Learning
  Research}, pp.\  3878--3887. PMLR, 2018.

\bibitem[Peng et~al.(2018)Peng, Andrychowicz, Zaremba, and Abbeel]{Peng_2018}
Xue~Bin Peng, Marcin Andrychowicz, Wojciech Zaremba, and Pieter Abbeel.
\newblock Sim-to-real transfer of robotic control with dynamics randomization.
\newblock In \emph{2018 IEEE International Conference on Robotics and
  Automation (ICRA)}, pp.\  3803--3810, 2018.
\newblock \doi{10.1109/ICRA.2018.8460528}.

\bibitem[Peng et~al.(2020)Peng, Coumans, Zhang, Lee, Tan, and
  Levine]{peng2020learningagileroboticlocomotion}
Xue~Bin Peng, Erwin Coumans, Tingnan Zhang, Tsang{-}Wei~Edward Lee, Jie Tan,
  and Sergey Levine.
\newblock Learning agile robotic locomotion skills by imitating animals.
\newblock In Marc Toussaint, Antonio Bicchi, and Tucker Hermans (eds.),
  \emph{Robotics: Science and Systems XVI, Virtual Event / Corvalis, Oregon,
  USA, July 12-16, 2020}, 2020.
\newblock \doi{10.15607/RSS.2020.XVI.064}.

\bibitem[Precup et~al.(2000)Precup, Sutton, and Singh]{precup2000eligibility}
Doina Precup, Richard~S. Sutton, and Satinder~P. Singh.
\newblock Eligibility traces for off-policy policy evaluation.
\newblock In \emph{Proceedings of the Seventeenth International Conference on
  Machine Learning (ICML)}, pp.\  759--766, 2000.

\bibitem[Pua \& Khadiv(2024)Pua and Khadiv]{pua2024safe}
Xun Pua and Majid Khadiv.
\newblock Safe learning of locomotion skills from {MPC}.
\newblock In \emph{2024 IEEE-RAS 23rd International Conference on Humanoid
  Robots (Humanoids)}, pp.\  459--466, 2024.
\newblock \doi{10.1109/Humanoids58906.2024.10769799}.

\bibitem[Ray et~al.(2019)Ray, Achiam, and Amodei]{ray2019benchmarking}
Alex Ray, Joshua Achiam, and Dario Amodei.
\newblock Benchmarking safe exploration in deep reinforcement learning.
\newblock Technical report, OpenAI, 2019.
\newblock URL \url{https://cdn.openai.com/safexp-short.pdf}.

\bibitem[Ross et~al.(2011)Ross, Gordon, and Bagnell]{ross2011dagger}
St{\'e}phane Ross, Geoffrey~J. Gordon, and J.~Andrew Bagnell.
\newblock A reduction of imitation learning and structured prediction to
  no-regret online learning.
\newblock In \emph{Proceedings of the 14th International Conference on
  Artificial Intelligence and Statistics (AISTATS)}, pp.\  627--635, 2011.

\bibitem[Rudin et~al.(2021)Rudin, Hoeller, Reist, and
  Hutter]{rudin2022learningwalkminutesusing}
Nikita Rudin, David Hoeller, Philipp Reist, and Marco Hutter.
\newblock Learning to walk in minutes using massively parallel deep
  reinforcement learning.
\newblock In Aleksandra Faust, David Hsu, and Gerhard Neumann (eds.),
  \emph{Conference on Robot Learning, 8-11 November 2021, London, {UK}}, volume
  164 of \emph{Proceedings of Machine Learning Research}, pp.\  91--100.
  {PMLR}, 2021.

\bibitem[Saunders et~al.(2018)Saunders, Sastry, Stuhlm\"{u}ller, and
  Evans]{10.5555/3237383.3238074}
William Saunders, Girish Sastry, Andreas Stuhlm\"{u}ller, and Owain Evans.
\newblock Trial without error: Towards safe reinforcement learning via human
  intervention.
\newblock In \emph{Proceedings of the 17th International Conference on
  Autonomous Agents and MultiAgent Systems}, AAMAS '18, pp.\  2067–2069,
  Richland, SC, 2018. International Foundation for Autonomous Agents and
  Multiagent Systems.

\bibitem[Schulman et~al.(2016)Schulman, Moritz, Levine, Jordan, and
  Abbeel]{Schulman2015HighDimensionalCC}
John Schulman, Philipp Moritz, Sergey Levine, Michael~I. Jordan, and Pieter
  Abbeel.
\newblock High-dimensional continuous control using generalized advantage
  estimation.
\newblock In \emph{International Conference on Learning Representations
  (ICLR)}, 2016.

\bibitem[Schulman et~al.(2017)Schulman, Wolski, Dhariwal, Radford, and
  Klimov]{schulman2017proximalpolicyoptimizationalgorithms}
John Schulman, Filip Wolski, Prafulla Dhariwal, Alec Radford, and Oleg Klimov.
\newblock Proximal policy optimization algorithms.
\newblock \emph{CoRR}, abs/1707.06347, 2017.

\bibitem[Silver et~al.(2014)Silver, Lever, Heess, Degris, Wierstra, and
  Riedmiller]{silver2014dpg}
David Silver, Guy Lever, Nicolas Heess, Thomas Degris, Daan Wierstra, and
  Martin Riedmiller.
\newblock Deterministic policy gradient algorithms.
\newblock In \emph{Proceedings of the 31st International Conference on Machine
  Learning (ICML)}, 2014.

\bibitem[Smith et~al.(2023{\natexlab{a}})Smith, Kew, Li, Luu, Peng, Ha, Tan,
  and Levine]{smith2023learningadaptingagilelocomotion}
Laura~M. Smith, J.~Chase Kew, Tianyu Li, Linda Luu, Xue~Bin Peng, Sehoon Ha,
  Jie Tan, and Sergey Levine.
\newblock Learning and adapting agile locomotion skills by transferring
  experience.
\newblock In Kostas~E. Bekris, Kris Hauser, Sylvia~L. Herbert, and Jingjin Yu
  (eds.), \emph{Robotics: Science and Systems XIX, Daegu, Republic of Korea,
  July 10-14, 2023}, 2023{\natexlab{a}}.
\newblock \doi{10.15607/RSS.2023.XIX.051}.

\bibitem[Smith et~al.(2023{\natexlab{b}})Smith, Kostrikov, and
  Levine]{smith2022walkparklearningwalk}
Laura~M. Smith, Ilya Kostrikov, and Sergey Levine.
\newblock Demonstrating {A Walk in the Park}: Learning to walk in 20 minutes
  with model-free reinforcement learning.
\newblock In \emph{Robotics: Science and Systems XIX}, 2023{\natexlab{b}}.
\newblock \doi{10.15607/RSS.2023.XIX.056}.

\bibitem[Smith et~al.(2024)Smith, Cao, and
  Levine]{smith2023growlimitscontinuousimprovement}
Laura~M. Smith, Yunhao Cao, and Sergey Levine.
\newblock Grow your limits: Continuous improvement with {Real-World} {RL} for
  robotic locomotion.
\newblock In \emph{2024 IEEE International Conference on Robotics and
  Automation (ICRA)}, pp.\  10829--10836, 2024.

\bibitem[Sootla et~al.(2022)Sootla, Cowen-Rivers, Jafferjee, Wang, Mguni, Wang,
  and Ammar]{sootla2022saute}
Aivar Sootla, Alexander~I Cowen-Rivers, Taher Jafferjee, Ziyan Wang, David~H
  Mguni, Jun Wang, and Haitham Ammar.
\newblock Saute {RL}: Almost surely safe reinforcement learning using state
  augmentation.
\newblock In Kamalika Chaudhuri, Stefanie Jegelka, Le~Song, Csaba Szepesvari,
  Gang Niu, and Sivan Sabato (eds.), \emph{Proceedings of the 39th
  International Conference on Machine Learning}, volume 162 of
  \emph{Proceedings of Machine Learning Research}, pp.\  20423--20443. PMLR,
  17--23 Jul 2022.
\newblock URL \url{https://proceedings.mlr.press/v162/sootla22a.html}.

\bibitem[Spencer et~al.(2020)Spencer, Choudhury, Barnes, Schmittle, Chiang,
  Ramadge, and Srinivasa]{spencer2020learning}
Jonathan Spencer, Sanjiban Choudhury, Matthew Barnes, Matthew Schmittle, Mung
  Chiang, Peter Ramadge, and Sidd Srinivasa.
\newblock Learning from interventions: Human-robot interaction as both explicit
  and implicit feedback.
\newblock In \emph{Proceedings of Robotics: Science and Systems (RSS)}, 2020.

\bibitem[Srinivasan et~al.(2020)Srinivasan, Eysenbach, Ha, Tan, and
  Finn]{srinivasan2020learningsafedeeprl}
Krishnan Srinivasan, Benjamin Eysenbach, Sehoon Ha, Jie Tan, and Chelsea Finn.
\newblock Learning to be safe: Deep {RL} with a safety critic.
\newblock \emph{CoRR}, abs/2010.14603, 2020.

\bibitem[Stooke et~al.(2020)Stooke, Achiam, and Abbeel]{stooke2020lagrangian}
Adam Stooke, Joshua Achiam, and Pieter Abbeel.
\newblock Responsive safety in reinforcement learning by {PID} {L}agrangian
  methods.
\newblock In \emph{Proceedings of the 37th International Conference on Machine
  Learning (ICML)}, 2020.

\bibitem[Sutton et~al.(1999)Sutton, Precup, and Singh]{sutton1999between}
Richard~S. Sutton, Doina Precup, and Satinder Singh.
\newblock Between {MDPs} and semi-{MDPs}: A framework for temporal abstraction
  in reinforcement learning.
\newblock \emph{Artificial Intelligence}, 112\penalty0 (1-2):\penalty0
  181--211, 1999.

\bibitem[Tan et~al.(2018)Tan, Zhang, Coumans, Iscen, Bai, Hafner, Bohez, and
  Vanhoucke]{tan2018simtoreallearningagilelocomotion}
Jie Tan, Tingnan Zhang, Erwin Coumans, Atil Iscen, Yunfei Bai, Danijar Hafner,
  Steven Bohez, and Vincent Vanhoucke.
\newblock Sim-to-real: Learning agile locomotion for quadruped robots.
\newblock In Hadas Kress{-}Gazit, Siddhartha~S. Srinivasa, Tom Howard, and
  Nikolay Atanasov (eds.), \emph{Robotics: Science and Systems XIV, Carnegie
  Mellon University, Pittsburgh, Pennsylvania, USA, June 26-30, 2018}, 2018.
\newblock \doi{10.15607/RSS.2018.XIV.010}.

\bibitem[Tessler et~al.(2019)Tessler, Mankowitz, and
  Mannor]{tessler2018rewardconstrainedpolicyoptimization}
Chen Tessler, Daniel~J. Mankowitz, and Shie Mannor.
\newblock Reward constrained policy optimization.
\newblock In \emph{International Conference on Learning Representations}, 2019.

\bibitem[Thananjeyan et~al.(2021)Thananjeyan, Balakrishna, Nair, Luo,
  Srinivasan, Hwang, Gonzalez, Ibarz, Finn, and
  Goldberg]{thananjeyan2021recoveryrlsafereinforcement}
Brijen Thananjeyan, Ashwin Balakrishna, Suraj Nair, Michael Luo, Krishnan
  Srinivasan, Minho Hwang, Joseph~E. Gonzalez, Julian Ibarz, Chelsea Finn, and
  Ken Goldberg.
\newblock Recovery {RL}: Safe reinforcement learning with learned recovery
  zones.
\newblock \emph{IEEE Robotics and Automation Letters}, 6\penalty0 (3):\penalty0
  4915--4922, 2021.
\newblock \doi{10.1109/LRA.2021.3070252}.

\bibitem[Todorov et~al.(2012)Todorov, Erez, and Tassa]{todorov2012mujoco}
Emanuel Todorov, Tom Erez, and Yuval Tassa.
\newblock {MuJoCo}: A physics engine for model-based control.
\newblock In \emph{2012 IEEE/RSJ International Conference on Intelligent Robots
  and Systems}, pp.\  5026--5033. IEEE, 2012.
\newblock \doi{10.1109/IROS.2012.6386109}.

\bibitem[Towers et~al.(2023)Towers, Terry, Kwiatkowski, Balis, Cola, Deleu,
  Goulão, Kallinteris, KG, Krimmel, Perez-Vicente, Pierré, Schulhoff, Tai,
  Shen, and Younis]{towers_gymnasium_2023}
Mark Towers, Jordan~K. Terry, Ariel Kwiatkowski, John~U. Balis, Gianluca~de
  Cola, Tristan Deleu, Manuel Goulão, Andreas Kallinteris, Arjun KG, Markus
  Krimmel, Rodrigo Perez-Vicente, Andrea Pierré, Sander Schulhoff, Jun~Jet
  Tai, Andrew Tan~Jin Shen, and Omar~G. Younis.
\newblock Gymnasium, March 2023.

\bibitem[Turchetta et~al.(2020)Turchetta, Kolobov, Shah, Krause, and
  Agarwal]{turchetta2020curriculum}
Matteo Turchetta, Andrey Kolobov, Shital Shah, Andreas Krause, and Alekh
  Agarwal.
\newblock Safe reinforcement learning via curriculum induction.
\newblock In \emph{Advances in Neural Information Processing Systems 33
  (NeurIPS)}, 2020.

\bibitem[Uchendu et~al.(2023)Uchendu, Xiao, Lu, Zhu, Yan, Simon, Bennice, Fu,
  Ma, Jiao, Levine, and Hausman]{uchendu2023jsrl}
Ikechukwu Uchendu, Ted Xiao, Yao Lu, Banghua Zhu, Mengyuan Yan, Joséphine
  Simon, Matthew Bennice, Chuyuan Fu, Cong Ma, Jiantao Jiao, Sergey Levine, and
  Karol Hausman.
\newblock Jump-start reinforcement learning.
\newblock In \emph{Proceedings of the 40th International Conference on Machine
  Learning (ICML)}, 2023.

\bibitem[Wachi et~al.(2024)Wachi, Shen, and Sui]{wachi2024constraint}
Akifumi Wachi, Xun Shen, and Yanan Sui.
\newblock A survey of constraint formulations in safe reinforcement learning.
\newblock In Kate Larson (ed.), \emph{Proceedings of the Thirty-Third
  International Joint Conference on Artificial Intelligence, {IJCAI-24}}, pp.\
  8262--8271. International Joint Conferences on Artificial Intelligence
  Organization, August 2024.
\newblock \doi{10.24963/ijcai.2024/913}.
\newblock URL \url{https://doi.org/10.24963/ijcai.2024/913}.
\newblock Survey Track.

\bibitem[Wagener et~al.(2021)Wagener, Boots, and Cheng]{wagener2021sailr}
Nolan~C. Wagener, Byron Boots, and Ching-An Cheng.
\newblock Safe reinforcement learning using advantage-based intervention.
\newblock In \emph{Proceedings of the 38th International Conference on Machine
  Learning (ICML)}, volume 139 of \emph{Proceedings of Machine Learning
  Research}, pp.\  10630--10640. PMLR, 2021.

\bibitem[Yang et~al.(2022)Yang, Zhang, Luu, Ha, Tan, and
  Yu]{yang2022safereinforcementlearninglegged}
Tsung-Yen Yang, Tingnan Zhang, Linda Luu, Sehoon Ha, Jie Tan, and Wenhao Yu.
\newblock Safe reinforcement learning for legged locomotion.
\newblock In \emph{2022 IEEE/RSJ International Conference on Intelligent Robots
  and Systems (IROS)}, pp.\  2454--2461, 2022.
\newblock \doi{10.1109/IROS47612.2022.9982038}.

\bibitem[Yu et~al.(2020)Yu, Thomas, Yu, Ermon, Zou, Levine, Finn, and
  Ma]{yu2020mopo}
Tianhe Yu, Garrett Thomas, Lantao Yu, Stefano Ermon, James~Y. Zou, Sergey
  Levine, Chelsea Finn, and Tengyu Ma.
\newblock {MOPO}: Model-based offline policy optimization.
\newblock In \emph{Advances in Neural Information Processing Systems},
  volume~33, 2020.

\bibitem[Zakka et~al.(2022)Zakka, Tassa, and {MuJoCo Menagerie
  Contributors}]{menagerie2022mujoco}
Kevin Zakka, Yuval Tassa, and {MuJoCo Menagerie Contributors}.
\newblock {MuJoCo} {Menagerie}: A collection of high-quality models for the
  {MuJoCo} physics engine, 2022.
\newblock URL \url{http://github.com/google-deepmind/mujoco_menagerie}.

\bibitem[Zhao et~al.(2023)Zhao, He, Chen, Wei, and Liu]{zhao2023statewise}
Weiye Zhao, Tairan He, Rui Chen, Tianhao Wei, and Changliu Liu.
\newblock State-wise safe reinforcement learning: A survey.
\newblock In Edith Elkind (ed.), \emph{Proceedings of the Thirty-Second
  International Joint Conference on Artificial Intelligence, {IJCAI-23}}, pp.\
  6814--6822. International Joint Conferences on Artificial Intelligence
  Organization, August 2023.
\newblock \doi{10.24963/ijcai.2023/763}.
\newblock URL \url{https://doi.org/10.24963/ijcai.2023/763}.
\newblock Survey Track.

\end{thebibliography}
\bibliographystyle{tmlr}

\appendix
\section{Full Proofs}
\label{app:proofs}

\subsection{Proof of \texorpdfstring{\Cref{thm:gradient}}{Theorem 1} (unbiased gradient of \texorpdfstring{$J^{\mathrm{mix}}(\theta)$}{J})}
\label{app:gradient-proof}

\begin{proof}
Starting from the definition,
\(
J^{\mathrm{mix}}(\theta) = \int p_\theta^{\mathrm{mix}}\, r(\tau)\, d\tau.
\)
Under \Cref{asm:regularity}, differentiation and integration commute:
\(
\nabla_\theta J^{\mathrm{mix}}(\theta) = \int \nabla_\theta p_\theta^{\mathrm{mix}}\, r(\tau)\, d\tau.
\)
Factor $p_\theta^{\mathrm{mix}}$ explicitly by partition membership of the visited states:
\begin{equation}\label{eq:thm1-factored}
p_\theta^{\mathrm{mix}}
= \underbrace{p(s_1)}_{\text{no }\theta}
\cdot \underbrace{\prod_{t : s_t \in \mathcal{M}} \piMain(a_t \mid s_t)}_{\text{only $\theta$-dependent factor}}
\cdot \underbrace{\prod_{t : s_t \notin \mathcal{M}} \mu(a_t \mid s_t)}_{\text{no }\theta}
\cdot \underbrace{\prod_{t=1}^T P(s_{t+1} \mid s_t, a_t)}_{\text{no }\theta},
\end{equation}
where $\mu(\cdot|s)$ denotes the (possibly Dirac) $\theta$-independent action measure at $s \notin \mathcal{M}$. Equation \eqref{eq:thm1-factored} is an equality of measures on trajectory space and does \emph{not} require $\mu$ to have a Lebesgue density. Deterministic, stochastic, point-mass at a hand-coded action, and pre-trained-policy measures are all covered. Only the $\mathcal{M}$-step factor depends on $\theta$. By the product rule and the score-function identity applied factor-wise,
\begin{equation}
\nabla_\theta p_\theta^{\mathrm{mix}} = p_\theta^{\mathrm{mix}}\, \sum_{t : s_t \in \mathcal{M}} \nabla_\theta \log \piMain(a_t \mid s_t).
\end{equation}
No gradient of $\mu$ appears, because $\mu$ carries no $\theta$-dependence; at no point do we evaluate $\log$ of a Dirac measure. Substituting back,
\[
\nabla_\theta J^{\mathrm{mix}}(\theta)
= \Expect_{\tau \sim p_\theta^{\mathrm{mix}}}\!\left[ \Big(\sum_{t : s_t \in \mathcal{M}} \nabla_\theta \log \piMain(a_t\mid s_t)\Big)\, r(\tau) \right].
\qedhere
\]
\end{proof}

\noindent \textbf{Corollaries (\Cref{cor:safe-region-gradient,cor:jsrl-gradient,cor:shielded-gradient}).}
The proof never used the specific form of $\mu$, only that $\mu$ carries no $\theta$-dependence, so each corollary follows by simply instantiating the pair $(\mathcal{M}, \mu)$. The safe-region intervention case ($\mathcal{M} = \safeR$, $\mu = \piRec$) is the body of this paper. The Jump-Start \ac{RL} case ($\mathcal{M} = \{(s,t) : t \ge h\}$ for a per-trajectory handoff step $h$, $\mu = \pi^{\mathrm{teach}}$ for a fixed teacher) inherits the same factorization: the teacher is $\theta$-independent, the handoff step is data-dependent but not $\theta$-dependent, and the score function evaluates only on the student-controlled tail of each trajectory. The state-triggered shielded \ac{RL} case ($\mathcal{M} = \{s : \Sigma \text{ does not engage at } s\}$, $\mu(a\mid s) = \delta(a - \sigma(s))$ for the shield's corrective action $\sigma$) inherits the same factorization with $\mu$ a Dirac measure that the proof's factor-level differentiation handles directly. The common payoff is the same in all three: the \ac{IS} fix is ill-defined when $\mu$ has no density at the relevant states, whereas the factor-level proof never forms the ratio and so bypasses the need for a density on $\mu$. We state the two non-safe-region specializations formally for reference.

\begin{corollary}[Jump-Start \ac{RL} \citep{uchendu2023jsrl}]\label{cor:jsrl-gradient}
Let $\pi^{\mathrm{teach}}$ be a fixed teacher policy and $h$ a per-trajectory handoff step. Take $\mathcal{M} = \{(s, t) : t \ge h\}$ (states reached after handoff; the timestep $t$ is absorbed into the state, standard for finite-horizon MDPs, so $\mathcal{M}$ is a subset of the augmented state space) and $\mu = \pi^{\mathrm{teach}}$. Then \Cref{thm:gradient} gives an unbiased gradient of the student's contribution to the mixed return, where the score function is evaluated only on the student-controlled tail of each trajectory.
\end{corollary}

\begin{corollary}[State-triggered shielded \ac{RL}]\label{cor:shielded-gradient}
Let $\Sigma$ be a \emph{state-triggered} shield, whose engagement is decided by the state alone and which, at unsafe states, replaces the proposed action with a $\theta$-independent corrective action $\sigma(s)$. Take $\mathcal{M} = \{s : \Sigma\text{ does not engage at }s\}$ and $\mu(a\mid s) = \delta(a - \sigma(s))$. Then \Cref{thm:gradient} gives an unbiased gradient evaluated only at unshielded states, with no density required on the shield's correction.
\end{corollary}

Action-triggered shields, as in \citet{alshiekh2018safe}, decide the override from the proposed action $a \sim \piMain(\cdot \mid s)$: at states where only some actions are unsafe, the executed-action distribution mixes $\piMain$ restricted to safe actions with the correction weighted by the $\theta$-dependent probability of proposing an unsafe action, so it depends on $\theta$ and falls outside \Cref{thm:gradient}'s hypothesis of a $\theta$-independent $\mu$.

\noindent \textbf{Truncated importance sampling.} One might hope that the singularity is only a problem for the naive \ac{IS} estimator and that a more careful clipped estimator escapes it. It does not. Truncated-\ac{IS} estimators such as V-trace \citep{espeholt2018impala} and Retrace \citep{munos2016retrace} reach the mixed-policy gradient by clipping the importance ratio $p_\theta/p_\theta^{\mathrm{mix}}$, but that ratio is still undefined at the deterministic-recovery steps where $\mu$ is a point mass. Clipping bounds variance downstream of the singularity rather than removing it. The factor-level proof above avoids the ratio entirely, which is exactly why it extends to the deterministic recovery policy our method uses.

\subsection{Proof of \texorpdfstring{\Cref{thm:convergence}}{the objective-gap bound} (objective-gap bound)}
\label{app:convergence-proof}

This subsection and the remainder of the appendix work in the infinite-horizon discounted convention fixed in \Cref{sec:background}, under which $\Vmix$, $\Qmix$, $\Amix$, and the normalized discounted visitations $\dpiMain$, $\dpiMixed$ are stationary.

\begin{proof}
\Cref{thm:gradient} shows we optimize $J^{\mathrm{mix}}$ rather than the deployment objective $J$, so the question is how far apart the two can be. We bound the gap by routing it through the states where the two policies actually differ. By the Performance Difference Lemma \citep{kakade2002approximately},
\begin{equation}\label{eq:pdl}
J(\theta) - J^{\mathrm{mix}}(\theta) = \frac{1}{1-\gamma}\, \Expect_{s \sim \dpiMain,\, a \sim \piMain(\cdot|s)}\!\left[ \Amix(s, a) \right],
\end{equation}
where $\Amix(s,a) = \Qmix(s,a) - \Vmix(s)$. Split the expectation by the safe region:
\[
J(\theta) - J^{\mathrm{mix}}(\theta) = \frac{1}{1-\gamma}\!\left( \underbrace{\Expect_{s \in \safeR}[\cdot]}_{=\,0} + \Expect_{s \notin \safeR}[\cdot] \right).
\]
On $s \in \safeR$, both $\piMain$ and $\piMixed$ sample from $\piMain$, so the inner expectation $\Expect_{a \sim \piMain(\cdot|s)}[\Amix(s,a)] = 0$ by definition of $\Vmix$. On $s \notin \safeR$, $|\Amix(s,a)| \le |\Qmix(s,a)| + |\Vmix(s)| \le 2\rmax/(1-\gamma)$ since both $|Q|$ and $|V|$ are bounded by $\rmax/(1-\gamma)$. Therefore
\[
|J(\theta) - J^{\mathrm{mix}}(\theta)| \le \frac{1}{1-\gamma} \cdot \frac{2\rmax}{1-\gamma} \cdot \Pr_{s \sim \dpiMain}[s \notin \safeR] = \frac{2\rmax}{(1-\gamma)^2}\, \betaRate. \qedhere
\]
\end{proof}

\noindent \textbf{Bound on the out-of-region rate $\betaRate$ (\Cref{sec:convergence}).} The gap above scales with $\betaRate$, so the bound is only useful once $\betaRate$ itself is controlled. \Cref{sec:convergence} uses $\betaRate \le \eta\gamma/(1-\gamma) \le \eta/(1-\gamma)$, which we now derive from the per-step invariance slack, under the additional hypothesis that episodes start inside the safe region, $\operatorname{supp}\, p(s_1) \subseteq \safeR(d)$ (otherwise the initial step alone contributes $(1-\gamma)\Pr[s_1 \notin \safeR]$ to $\betaRate$ even with $\eta = 0$). All three of our environments satisfy the hypothesis: episodes reset to an upright pose inside the region. Starting from a reachable initial state $s_0 \in \safeR$ (this derivation counts steps from zero, so $s_0$ here is the $s_1$ of the trajectory notation), \Cref{asm:invariance} gives $\Pr[s_0, \dots, s_t \in \safeR] \ge (1-\eta)^t$. Since $\{s_t \in \safeR\} \supseteq \{s_0, \dots, s_t \in \safeR\}$, the marginal in-region probability dominates the all-stay probability, $\Pr[s_t \in \safeR] \ge (1-\eta)^t$, hence $\Pr[s_t \notin \safeR] \le 1 - (1-\eta)^t$. With the discounted state-visitation $\betaRate = (1-\gamma)\sum_{t\ge 0} \gamma^t \Pr[s_t \notin \safeR]$,
\[
\betaRate \le (1-\gamma)\sum_{t\ge 0} \gamma^t \big(1 - (1-\eta)^t\big) = \frac{\eta\gamma}{1 - \gamma(1-\eta)} \le \frac{\eta\gamma}{1-\gamma} \le \frac{\eta}{1-\gamma},
\]
the last step using $1-\gamma(1-\eta) \ge 1-\gamma$ and $\gamma \le 1$.

\noindent \textbf{Train-to-deploy reward gap (footnote in \Cref{sec:value}).} The training-signal convention of \Cref{prop:analytic} (zero per-step reward during recovery, one-time terminal reward $r_{\mathrm{term}}$) and the unmodified task reward differ only at states outside $\safeR$ (and through the one-time $r_{\mathrm{term}}$). Bounding the per-step discrepancy by $\rmax + |r_{\mathrm{term}}|$ on those steps and summing the discounted visitation under $\piMain$, the returns the two conventions induce under $\piMain$ differ by at most $(\rmax + |r_{\mathrm{term}}|)\,\betaRate/(1-\gamma)$, so the train-to-deploy chain is controlled by the same $\betaRate$ as \Cref{thm:convergence}.

\subsection{Proof of \texorpdfstring{\Cref{cor:fixed-point}}{Corollary} (conditional fixed point)}
\label{app:fixedpoint-proof}

\begin{proof}
The gap bound of \Cref{thm:convergence} leaves $\betaRate$ free, so the natural endpoint of the analysis is the case that drives $\betaRate$ to zero. The corollary splits into two readings of when that happens.

\textbf{Reading (R-A):} If $\safeR(\dmax)$ covers the reachable state space, then $\Pr[s \notin \safeR(\dmax)] = 0$ under \emph{any} action distribution, including $\piMain$ alone. Hence $\beta(\theta, \dmax) = 0$, and \cref{eq:convergence} gives $J^{\mathrm{mix}}(\theta) = J(\theta)$ exactly. Furthermore, $\piMixed$ is identical to $\piMain$ on the reachable support of $\pTraj$ (recovery is never engaged).

\textbf{Reading (R-B):} If $\safeR(\dmax)$ is strictly contained in the reachable state space, the theory no longer forces $\piMain$ to remain in $\safeR$, so its invariance slack $\eta_\star := \eta(\theta,\dmax) \ge 0$ (\Cref{asm:invariance}) need not be zero. Substituting $\beta \le \eta_\star/(1-\gamma)$ (derived above under the initial-state hypothesis $\operatorname{supp}\, p(s_1) \subseteq \safeR$) into \cref{eq:convergence} gives the upper bound $|J(\theta)-J^{\mathrm{mix}}(\theta)| \le 2\rmax \eta_\star / (1-\gamma)^3$, which, unlike (R-A), is not pinned to zero. Whether $\eta_\star$ (and hence the gap) is small is environment-dependent: \Cref{app:thm2-check} finds $\eta_\star \to 0$ on \envHC and \envAnt, where the gap closes to seed noise, and $\eta_\star > 0$ on \envGo, where a proportional residual remains.
\end{proof}

\subsection{Proof of \texorpdfstring{\Cref{prop:analytic}}{the analytic recovery value} (analytic recovery value)}
\label{app:analytic-proof}

\begin{proof}
The gradient and gap results above treat $\Vmix$ as given; what remains is to compute it cheaply on the recovery segments themselves, where the trajectory is no longer under the agent's control. By the definition of $\Vmix$,
\(
\Vmix(s_t) = \Expect_{\tau \sim \pTrajMixed | s_t}\!\left[ \sum_{j=0}^{\infty} \gamma^j r_{t+j} \right].
\)
Under deterministic dynamics and deterministic recovery, the trajectory starting from a recovery-triggering state $s_t$ is fully determined for the duration of the recovery segment $[t, t+k]$. The expectation collapses. Two cases:
\begin{description}
\item[Recovery success] ($s_{t+k} \in \safeR$): the realized return over the segment is $G_{t,k} = \sum_{j=0}^{k-1}\gamma^j r_{t+j}$. From $s_{t+k}$ the agent resumes under $\piMixed$ with value $\Vmix(s_{t+k})$. Hence $\Vmix(s_t) = G_{t,k} + \gamma^k \Vmix(s_{t+k})$.
\item[Recovery failure] (terminal): the realized return is $G_{t,k}$ over the $k$ deterministic steps, the episode ends, no further reward accrues. Hence $\Vmix(s_t) = G_{t,k}$.
\end{description}
Combining gives \cref{eq:analytic-V}. Our implementation instantiates this general statement under a chosen reward: zero per-step reward during recovery and a one-time terminal reward $r_{\mathrm{term}}$ accrued at the final segment step ($t{+}k{-}1$) on failure. The terminal penalty is thus the only nonzero contribution to the segment return $G_{t,k} = \sum_{j=0}^{k-1}\gamma^j r_{t+j}$, which is therefore $0$ on success and $\gamma^{k-1} r_{\mathrm{term}}$ on failure, matching \cref{eq:analytic-V}.
\end{proof}

\subsection{Analytic value under stochastic dynamics}
\label{app:analytic-stochastic}

The closed form just derived rests on the determinism assumption, so it is worth asking what survives when that assumption is relaxed. \Cref{prop:analytic} assumes deterministic $P$ and $\piRec$. Under process noise or stochastic recovery, the realized $\gamma^k V_\theta(s_{t+k})$ is no longer equal to $\Vmix(s_t)$. It becomes a single Monte-Carlo realization drawn from the distribution over post-trigger trajectories, whose variance grows with the noise of $P$ and $\piRec$ accumulated over the segment. %
Two natural extensions preserve the spirit of bypassing critic bootstrapping across recovery. (i) A multi-sample MC target: re-roll the recovery $M$ times from $s_t$ (where the simulator supports state-resets) and average $G_{t,k_m} + \gamma^{k_m} V_\theta(s_{t+k_m})$ across the $M$ realizations. \revised{This averages over the segment stochasticity, with variance falling as $1/M$.} (ii) The single-sample plug-in we already use \revised{is a one-sample MC estimate of $\Vmix(s_t)$, and} its single-realization noise reaches downstream \ac{GAE}. \revised{Both targets bootstrap the endpoint critic, so both are unbiased for $\Vmix(s_t)$ only to the extent that $V_\theta$ is unbiased at the re-entry state; any bias there carries through, and the stochastic-case guarantee is conditional on it.} The bias-variance trade-off is favorable in our setting because $\piRec$ is approximately deterministic in practice (\ac{MPC}, greedy \ac{SAC}), so the per-segment variance is small. A formal analysis under controlled stochasticity, and the regime where (i) is worth the extra simulator calls, is left to follow-up.

\subsection{Proof of \texorpdfstring{\Cref{prop:compat-special-cases}}{Proposition (compatibility regularizer specializations)}}
\label{app:compat-proof}

The recovery segments that supplied the analytic value also supply state-action pairs for the compatibility regularizer, and identifying its limiting cases shows that this loss is not an ad-hoc addition but a familiar imitation objective in disguise.

\begin{proposition}[Special cases of the compatibility regularizer]\label{prop:compat-special-cases}
Let $\piRec$ be deterministic, so $a_t^{\mathrm{rec}} = \piRec(s_t)$ at recovery-controlled states, and consider the gate-set choice $\sigma_k \equiv 1$ for all segments.
\begin{enumerate}[label=(\alph*),leftmargin=2em,nosep]
\item \textbf{Behavioral cloning on the recovery dataset.} Under $\sigma_k \equiv 1$,
\begin{equation}\label{eq:compat-bc-limit}
\compatLoss_{\sigma \equiv 1}(\theta) = -\frac{\lambdaCompat}{N_{\mathrm{rec}}} \sum_{t : s_t \notin \safeR} \log \piMain\!\left(\piRec(s_t) \mid s_t\right),
\end{equation}
the (scaled) negative log-likelihood of $\piMain$ on the dataset $\mathcal{B} = \{(s_t, \piRec(s_t)) : s_t \notin \safeR\}$ of recovery-controlled state-action pairs collected from rollouts of $\piMixed$. Equivalently, $\compatLoss_{\sigma \equiv 1}$ is the per-step behavioral-cloning objective applied to $\piRec$'s actions at the states $\piMixed$ visits during recovery.
\item \textbf{A region-restricted \ac{DAgger} aggregation step.} The dataset $\mathcal{B}$ is collected as in one iteration of \ac{DAgger} \citep{ross2011dagger} with the teacher queried only on the intervention set: states are visited under the mixed roll-out policy $\piMixed$ (predominantly student-induced up to the safe-region boundary), and labels are the controller $\piRec$'s actions at the recovery-controlled states only, whereas \ac{DAgger} labels every visited state. Minimizing $\compatLoss_{\sigma \equiv 1}$ as the per-iteration training step then reproduces a region-restricted analogue of \ac{DAgger}'s update with $\piRec$ in the teacher role and $\piMain$ as the student.
\item \textbf{Outcome gate as a designer filter.} The general gate $\sigma_k \in \{0, 1\}$ restricts $\mathcal{B}$ to recovery segments that re-entered $\safeR$, discarding the rest. \ac{JSRL}-style mixed-policy data collection \citep{uchendu2023jsrl} produces the same mixed-policy roll-out distribution but does not include an imitation loss. $\compatLoss$ with $\sigma_k \equiv 1$ supplies the analogous student-imitation step that the \ac{JSRL} algorithm itself omits.
What is new in the case $\sigma_k \not\equiv 1$ is gating imitation of an \emph{external recovery controller} by the \emph{realized success of its multi-step segment}, as opposed to gating by value estimates of one's own past actions \citep{oh2018sil} or of demonstrations \citep{nair2018demonstrations}.
\end{enumerate}
\end{proposition}

\begin{proof}
\textbf{(a)} From \cref{eq:compat-loss} with $w_t = \sigma_{\seg(t)}$, where $\seg(t)$ is the index of the recovery segment containing $t$ (\Cref{sec:compat}), and the assumption $\sigma_k \equiv 1$,
\begin{equation*}
\compatLoss_{\sigma \equiv 1}(\theta)
= -\frac{\lambdaCompat}{N_{\mathrm{rec}}} \sum_{t : s_t \notin \safeR} 1 \cdot \log \piMain(a^{\mathrm{rec}}_t \mid s_t)
= -\frac{\lambdaCompat}{N_{\mathrm{rec}}} \sum_{t : s_t \notin \safeR} \log \piMain\!\left(\piRec(s_t) \mid s_t\right),
\end{equation*}
using $a_t^{\mathrm{rec}} = \piRec(s_t)$ for deterministic $\piRec$. This is, up to the constant scale $\lambdaCompat / N_{\mathrm{rec}}$, the empirical negative log-likelihood of $\piMain$ on the dataset $\mathcal{B} = \{(s_t, \piRec(s_t)) : s_t \notin \safeR\}$, the (population-level) maximum-likelihood / behavioral-cloning objective for $\piMain$ on $\mathcal{B}$.

\textbf{(b)} The dataset $\mathcal{B}$ is collected by rolling out $\piMixed$, retaining the state-action pairs at the unsafe-region timesteps. Under $\piMixed$, $\piMain$ controls the rollout up to the first exit from $\safeR$ (so the visited boundary state $s_t$ has $\piMain$-induced distribution), and $\piRec$ provides actions thereafter. This matches the data-collection rule of a single \ac{DAgger} iteration with $\piRec$ as the teacher and $\piMain$ as the student, states visited under the mixed roll-out policy and labels from the teacher, except that the teacher is queried only at the recovery-controlled states rather than at every visited state as in \citet{ross2011dagger}. Minimizing $\compatLoss_{\sigma \equiv 1}$ as the per-iteration update therefore reproduces a region-restricted analogue of the \ac{DAgger} aggregation step.

\textbf{(c)} For general $\sigma_k \in \{0, 1\}$, the regularizer is
\begin{equation*}
\compatLoss(\theta)
= -\frac{\lambdaCompat}{N_{\mathrm{rec}}} \sum_{k} \sigma_k \sum_{t : s_t \notin \safeR,\, \seg(t) = k} \log \piMain\!\left(\piRec(s_t) \mid s_t\right),
\end{equation*}
i.e.\ the behavioral-cloning objective restricted to the sub-dataset $\mathcal{B}_{\mathrm{succ}} = \{(s_t, \piRec(s_t)) : s_t \notin \safeR,\, \sigma_{\seg(t)} = 1\}$ collected from segments that re-entered $\safeR$. \ac{JSRL} \citep{uchendu2023jsrl} likewise rolls out a mixed teacher-then-student policy but trains the student only by reinforcement on the student-controlled tail. It does not include a teacher-imitation step. The $\sigma_k \equiv 1$ specialization of $\compatLoss$ is the analogous imitation loss \ac{JSRL} omits.
In the $\sigma_k \not\equiv 1$ case, the new element is the gating criterion and target: imitation of an external recovery controller gated by the realized success of its multi-step segment, rather than by value estimates of the agent's own past actions \citep{oh2018sil} or of demonstrations \citep{nair2018demonstrations}.
\end{proof}

\subsection{Sketch: why an observable upper bound on \texorpdfstring{$\beta$}{beta} is hard in continuous actions}
\label{app:obs-bound}

The bound on $\betaRate$ above is stated in terms of the per-step invariance slack $\eta$, which is not directly measurable; a bound in terms of an observable quantity would be far more useful, and we close by explaining why that is hard. A natural attempt is to bound $\beta$ by the observable recovery rate $\alpha$ via a simulation-lemma argument:
\(
\beta = \Pr_{\dpiMain}[s\notin\safeR] \le \alpha + \mathrm{TV}(\dpiMain, \dpiMixed).
\)
Here $\mathrm{TV}(\cdot,\cdot)$ denotes total variation (TV) distance. The TV term can in turn be bounded by the expected per-step TV between $\piMain$ and $\piMixed$ via the standard simulation lemma \citep{kearns2002sparse}. But on unsafe states $\piMixed$ collapses to a Dirac (deterministic recovery) while $\piMain$ remains absolutely continuous. The per-step TV at unsafe states is $1$ pointwise, so the simulation-lemma term contributes a full $1/(1-\gamma) > 1$, and the bound $\beta \le \alpha + 1/(1-\gamma)$ exceeds $1$ and is therefore vacuous (since $\beta \le 1$ trivially). Routes that may produce a non-vacuous bound include (i) replacing the deterministic recovery with a smoothed (LSE / Gaussian-mixture) variant whose TV with $\piMain$ is finite, (ii) Wasserstein-action-gap bounds under known dynamics Lipschitz constants, and (iii) periodic offline evaluation of $\piMain$ alone. We sketch (i) below and leave a complete treatment to follow-up work.

\noindent \textbf{Sketch via smoothed recovery.} Replace $\piRec(\cdot|s) = \delta(a - \bar a(s))$ with $\tilde \pi_\mathrm{rec}(\cdot|s) = \mathcal{N}(\bar a(s), \sigma_\mathrm{rec}^2 I)$ for small $\sigma_\mathrm{rec}$.
Then both $\piMain(\cdot|s)$ and $\tilde\pi_\mathrm{rec}(\cdot|s)$ are absolutely continuous and, when the scales are matched ($\sigma_\mathrm{rec} \approx \sigma_\theta(s)$, with $\sigma_\theta$ the main policy's standard deviation), the Gaussian TV bound gives $\mathrm{TV}(\piMain, \tilde\pi_\mathrm{rec}) \le \mathrm{const}\cdot \|\mu_\theta(s) - \bar a(s)\| / \sigma_\mathrm{rec}$ at unsafe states up to a scale-mismatch term (with mismatched covariances the TV is bounded away from zero even at equal means), so the per-step TV is integrable under these stated conditions. The constant, however, grows with the action dimension, so even this smoothed bound loosens in high dimensions, the same mechanism by which the importance weights collapse on the 12-dimensional \envGo action space (\Cref{app:is-variance}). Moreover, the TV to a fixed-scale $\piMain$ approaches one as $\sigma_\mathrm{rec} \to 0$, so the smoothing scale trades fidelity to the deployed deterministic recovery against tightness of the bound; the sketch therefore yields an observable upper-bound proxy for $\beta$ only under the matched-scale conditions above, not a general bound.

\section{Algorithm Details}
\label{app:algorithm}

\subsection{Full pseudocode}

\Cref{alg:safeexplorer} gives the complete \methodName{} update. The per-environment curriculum grows the safe-region radius linearly, $\diter(u) = d_0 + \frac{u-1}{N}\,\dmax$ at update $u$ of $N$, with $d_0 = 0.01$ on \envHC and \envAnt and $d_0 = 0.05$ on \envGo, up to the per-environment $\dmax$ values in \Cref{tab:hyperparams} (\Cref{app:hyperparams}).

\begin{algorithm}[t]
\caption{\methodName{}: \ac{PPO} with safe-step masking, analytic recovery V, and hard $\Ctheta$.}
\label{alg:safeexplorer}
\begin{algorithmic}[1]
\Require Main policy $\piMain$, recovery policy $\piRec$, value function $V_\theta$, safe-region radius schedule $\diter$, total updates $N$, rollout length $T_{\mathrm{ro}}$, num envs $E$, compatibility coef $\lambdaCompat$, learning rate $\eta_{\mathrm{lr}}$, terminal reward convention $r_{\mathrm{term}}$ (\Cref{app:rterm}).
\For{$\mathrm{update} = 1, \ldots, N$}
  \State $d \gets \diter(\mathrm{update})$
  \State \textbf{Rollout phase:} for $T_{\mathrm{ro}}$ steps in each of $E$ parallel envs,
  \State \quad if $s_t \in \safeR(d)$: $a_t \sim \piMain(\cdot | s_t)$, store $(s_t, a_t, r_t, \log \piMain(a_t|s_t), V_\theta(s_t))$
  \State \quad if $s_t \notin \safeR(d)$: $a_t = \piRec(s_t)$, $\,\mathtt{is\_rec}_t \gets 1$
  \State \quad\quad store $(s_t, a_t, r_t, \log \piMain(a_t|s_t), V_\theta(s_t))$ \Comment{$\log\piMain(a^{\mathrm{rec}}_t|s_t)$ feeds $\compatLoss$ only; never enters the \ac{PPO} ratio (safe-step mask).}
  \State \revised{\textbf{Analytic-V overwrite, whole segment (\Cref{prop:analytic}, \Cref{lem:induced-target}):} for each recovery segment starting at $t_k^\mathrm{start}$ of length $\ell_k$ (intermediate recovery rewards stored as zero), and \emph{every} offset $j = 0, \ldots, \ell_k - 1$,}
  \State \quad \revised{if segment ended in re-entry at $t_k^\mathrm{start} + \ell_k$: $V(s_{t_k^\mathrm{start}+j}) \gets \gamma^{\ell_k - j} V_\theta(s_{t_k^\mathrm{start}+\ell_k})$ \Comment{the segment return is $0$ under the zeroed-recovery-reward convention, \Cref{prop:analytic}}}
  \State \quad \revised{if segment ended in termination: $V(s_{t_k^\mathrm{start}+j}) \gets \gamma^{\ell_k-1-j} r_{\mathrm{term}}$ \Comment{the terminal reward; recovery steps are zeroed, so this is the remaining failure-segment return}}
  \State \quad if segment cut off by rollout truncation: treat as re-entry at the truncation boundary \Comment{tentative success with a bootstrapped value, \Cref{sec:compat}}
  \State \textbf{Advantage estimation:} compute $\Aest^\mathrm{full}$ via \ac{GAE} on the overwritten value sequence; the value target is $\hat R_t = \Aest_t^\mathrm{full} + V(s_t)$\revised{, with $V(s_t)$ the stored, possibly overwritten, value}.
  \State \textbf{Outcome gates:} for each recovery segment $k$, set $\sigma_k \gets \indicator[\text{segment ended in re-entry}]$.
  \For{epoch $= 1, \ldots, K$}
    \For{minibatch $\subset$ rollout}
      \State $\Tsafe \gets \{t : \mathtt{is\_rec}_t = 0\}$ (safe-step mask)
      \State $\PPOloss_\mathrm{safe} \gets -\frac{1}{|\Tsafe|}\sum_{t\in\Tsafe}\min(\rho_t \Aest_t^\mathrm{full},\,\mathrm{clip}(\rho_t)\Aest_t^\mathrm{full})$
      \State $\valueLoss \gets \frac{1}{T_{\mathrm{ro}}}\sum_t (V_\theta(s_t) - \hat R_t)^2$ \Comment{unmasked: recovery-trigger targets are variance-free, exact up to critic error at re-entry}
      \State $\compatLoss \gets -\frac{\lambdaCompat}{N_\mathrm{rec}} \sum_{t\notin\safeR} \sigma_{\seg(t)} \log\piMain(a_t^\mathrm{rec} | s_t)$
      \State $L \gets \PPOloss_\mathrm{safe} + c_v \valueLoss + \compatLoss - c_e H[\piMain]$
      \State $\theta \gets \theta - \eta_{\mathrm{lr}}\, \nabla_\theta L$
    \EndFor
  \EndFor
\EndFor
\end{algorithmic}
\end{algorithm}

\revised{\subsection{The exact target induced by the whole-segment overwrite}
\label{app:induced-target}

Algorithm 1 overwrites the value at \emph{every} step of a recovery segment, not only at the triggering state. The following lemma derives the exact $\lambda$-\ac{GAE} target this implementation induces and relates it to the analytic target of \Cref{prop:analytic}; a start-only overwrite would not have this property for $\lambda < 1$, since intermediate learned-critic values would enter the TD residuals.

\begin{lemma}[Induced \ac{GAE} target of the whole-segment overwrite]\label{lem:induced-target}
Consider a recovery segment of length $k$ starting at $t$ that is resolved within the rollout, by re-entry at $t+k$ (success) or \revised{episode} termination at $t+k-1$ (failure), with intermediate recovery rewards stored as zero, the failure-terminal step storing $r_{\mathrm{term}}$, and values overwritten as in Algorithm 1: on success $V(s_{t+j}) \gets \gamma^{k-j} V_\theta(s_{t+k})$ and on failure $V(s_{t+j}) \gets \gamma^{k-1-j} r_{\mathrm{term}}$, for $j = 0, \ldots, k-1$. Then, for every $\lambda \in [0,1]$:
\begin{enumerate}[label=(\roman*),leftmargin=2em,nosep]
\item every intra-segment TD residual vanishes identically, so ordinary $\lambda$-\ac{GAE} telescopes through the segment;
\item on failure segments the induced target at the triggering state equals the \Cref{prop:analytic} analytic target exactly;
\item on success segments the induced target equals the analytic target with the endpoint critic bootstrap $V_\theta(s_{t+k})$ partially replaced by the observed post-re-entry $\lambda$-return (fully replaced at $\lambda = 1$), which reduces reliance on the endpoint critic rather than corrupting the target; at no $\lambda$ does any intermediate critic estimate enter;
\item the last safe step before the segment receives the residual $\delta = r + \gamma \cdot \gamma^{k} V_\theta(s_{t+k}) - V(s)$, that is, the entire recovery segment collapses into a single discounted semi-MDP transition. The one difference from a literal one-transition collapse is credit decay: the segment still counts as $k+1$ steps in the $\lambda$-return, contributing a factor $(\gamma\lambda)^{k+1}$ rather than the single $\lambda$ factor of a true one-step collapse.
\end{enumerate}
\end{lemma}

\begin{proof}
The overwritten values satisfy the zero-reward Bellman recursion exactly. On a success segment, interior steps have reward $0$ and consecutive overwritten values $\gamma^{k-j} V_\theta(s_{t+k})$ and $\gamma^{k-j-1} V_\theta(s_{t+k})$, so
$\delta_{t+j} = 0 + \gamma \cdot \gamma^{k-j-1} V_\theta(s_{t+k}) - \gamma^{k-j} V_\theta(s_{t+k}) = 0$;
on a failure segment the same computation with $\gamma^{k-1-j} r_{\mathrm{term}}$ in place of the bootstrap gives $\delta_{t+j} = 0$, and the failure-terminal step, because it stores $r_{\mathrm{term}}$ and has overwritten value $r_{\mathrm{term}}$, gives $\delta = r_{\mathrm{term}} - r_{\mathrm{term}} = 0$. This proves (i). The $\lambda$-\ac{GAE} advantage at any step is the $(\gamma\lambda)$-discounted sum of downstream residuals; with all intra-segment residuals zero, the value target at the triggering state is its overwritten value plus $(\gamma\lambda)^{k}$-discounted post-segment residuals. On failure there are no post-segment residuals, giving (ii). On success the surviving terms are exactly the ones that mix the endpoint bootstrap $\gamma^{k} V_\theta(s_{t+k})$ with the observed post-re-entry $\lambda$-return, in the standard $\lambda$-return proportions, giving (iii); at $\lambda = 1$ the bootstrap is fully replaced by the Monte-Carlo return. For (iv), substituting the overwritten trigger-state value into the residual of the preceding safe step yields $\delta = r + \gamma^{k+1} V_\theta(s_{t+k}) - V(s)$, the single discounted semi-MDP transition, and the $(\gamma\lambda)^{k+1}$ credit-decay factor follows from the segment occupying $k+1$ positions in the $\lambda$-return.
\end{proof}

\textbf{Credit decay is a trace convention.} The $(\gamma\lambda)^{k+1}$ factor in (iv) is not an error in the target. At $\lambda = 1$ the induced advantage is the Monte-Carlo advantage $G_t - V(s_t)$ under either accounting, so the two coincide exactly. For $\lambda < 1$ the segment keeps \ac{GAE}'s per-timestep eligibility trace over an interval that lasts $k$ real steps, whereas a literal collapse would apply a single $\lambda$ factor to the same interval. We keep the per-timestep trace: it propagates less post-segment credit back across a long intervention, which matches the fact that the post-segment state is set by where $\piRec$ lands rather than by the task action preceding the segment. The gap between the two conventions therefore grows with segment length, and is largest on \envAnt, whose mean segment is $159$ steps (\Cref{app:pretraining-accounting}).

\textbf{Exclusions.} The lemma is scoped to segments resolved within a rollout by re-entry or termination. Two code edge cases are excluded and stated as such: time-limit-truncated segments, which keep the environment reward at the last step and bootstrap the post-reset value, and segments spanning an environment auto-reset.}

\subsection{Notation}
\label{app:notation}

\Cref{tab:notation} collects the symbols used in \Cref{alg:safeexplorer} and the design rationale below.

\begin{table}[ht]
\centering\small
\setlength{\tabcolsep}{4pt}
\caption{Notation used in \Cref{alg:safeexplorer} and throughout the paper.}
\label{tab:notation}
\begin{tabular}{l p{0.78\linewidth}}
\toprule
$\piMain(a\mid s)$ & main policy, stochastic, $\theta$-parameterized \\
$\piRec(a\mid s)$ & recovery policy; no $\theta$-dependence; arbitrary form \\
$\piMixed$ & mixed policy (\Cref{sec:background}) \\
$\safeR$, $d$, $\dmax$ & safe region, current boundary \revised{(referred to as the radius)}, per-environment safe-region maximum (tuned hyperparameter) \\
\revised{$\mathcal{D}(s)$} & \revised{task-space distance from $s$ to the nominal configuration; $\safeR(d) = \{s : \mathcal{D}(s) \le d\}$ (\Cref{app:safe-region-distance})} \\
\revised{$d_0$} & \revised{initial safe-region radius of the curriculum ($0.01$ on \envHC/\envAnt, $0.05$ on \envGo)} \\
\revised{$\diter(u)$, $N$} & \revised{curriculum schedule $d_0 + \frac{u-1}{N}\dmax$ at update $u$, and the total number of training updates} \\
\revised{$\gamma$, $\lambda$} & \revised{discount factor and \ac{GAE} coefficient ($0.99$ and $0.95$, \Cref{tab:pseudocode-params})} \\
\revised{$\Tsafe$, $\Aest^{\mathrm{full}}$} & \revised{set of safe (main-policy) steps in a rollout, and the \ac{GAE} advantage over the full mixed-policy trajectory} \\
$\pTraj, \pTrajMixed$ & trajectory distributions under $\piMain$ alone and the mixed policy \\
$J(\theta), J^{\mathrm{mix}}(\theta)$ & main-policy return (unobservable during training) and mixed-policy return \\
$\alphaRate$ & recovery rate (observable) \\
$\betaRate$ & main-policy out-of-region rate (unobservable during training) \\
$\Ctheta$ & compatibility score $\prod_{t \notin \safeR} \piMain(a^\mathrm{rec}_t \mid s_t)$ \\
$\Vmix, \Qmix, \Amix$ & value, $Q$, advantage under the mixed policy \\
$\ell_k$ & length of recovery segment $k$ \\
$\seg(t)$ & index of the recovery segment containing step $t$ \\
$\sigma_k$ & outcome gate of segment $k$ ($1$ iff it ends with re-entry into $\safeR$) \\
$N_\mathrm{rec}$ & number of recovery-controlled steps in the current minibatch (normalizer of $\compatLoss$) \\
$r_{\mathrm{term}}$ & environment reward on the terminal step of a failed recovery segment; per-environment values in \Cref{app:rterm} \\
\bottomrule
\end{tabular}
\end{table}

\subsection{Operational features of the safe-region intervention setup}
\label{app:operational-features}

The safe-region intervention mechanism produces a training-time setup with four operational features that simultaneously place it outside \ac{CMDP} and pointwise-filter approaches: (i) the safety predicate is binary set membership $s \in \safeR$ rather than a scalar cost, so there is no Lagrangian dual to optimize; (ii) $\piRec$ may be deterministic (\ac{MPC}, greedy \ac{SAC}), in which case it has no Lebesgue density and the \ac{IS} correction is ill-defined; (iii) recovery takes control for multi-step segments rather than at a single step, so training rollouts follow a genuine mixed-policy distribution rather than a one-step perturbation; (iv) the objective is the original main-policy return $J(\theta)$, not a constrained surrogate, so safety is enforced operationally rather than penalized in the reward. Each of the four method components in \Cref{sec:method} responds to a specific feature in this list.

\subsection{Soft outcome gate}
\label{app:soft-gate}

In addition to the hard outcome gate of \Cref{sec:compat}, we considered a signed, per-step preference based on whether the post-segment state was higher- or lower-value than the pre-segment state, with $k(t)$ the length of the segment containing $t$:
\begin{equation}\label{eq:soft-gate}
w_t^{\mathrm{soft}} = \gamma^{k(t)} V_\theta(s_{t+k(t)}) - V_\theta(s_t), \quad \text{normalized to zero mean per minibatch}.
\end{equation}
Empirically, the soft variant is unstable on \envHC and \envGo (the soft-gate variants in \Cref{app:soft-variants}) because the regularizer fits failed-recovery actions, so we drop it from the headline comparison.

\subsection{Random-denominator and value-loss target choices}
\label{app:value-loss-target}

Two further choices in \Cref{alg:safeexplorer} warrant explanation: the denominator of the masked policy gradient and the set of timesteps the critic regresses on. The masked PG estimator divides by $|\Tsafe|$, a random variable that depends on the rollout.

\revised{\noindent\textbf{From \Cref{thm:gradient} to \cref{eq:ppo-safe}.} \Cref{thm:gradient} weights the masked score sum by the whole-trajectory return $r(\tau)$, while \cref{eq:ppo-safe} pairs each safe step's score with the per-step target $\Aest_t^{\mathrm{full}}$. The two standard policy-gradient reductions justify this at safe steps exactly as in \ac{PPO}. At a safe step the action is drawn from $\piMain(\cdot \mid s_t)$, so the score $\nabla_\theta \log \piMain(a_t \mid s_t)$ has conditional mean zero given the history; multiplying it by any history-measurable quantity, the rewards collected before $t$ or a state baseline $b(s_t)$, contributes nothing in expectation, which turns $r(\tau)$ into the reward-to-go minus a baseline. Recovery steps contribute no score term, so neither reduction touches them. Substituting \ac{GAE} for the reward-to-go and clipping then modify the estimator exactly as in standard \ac{PPO}, the inherited approximations named in \Cref{sec:gradient}.}

\revised{%
The immediate question about \cref{eq:ppo-safe} is what a random
denominator does to the estimator, since in general
$\Expect[\sum_{t} g_t / |\Tsafe|] \neq \Expect[\sum_t g_t] / \Expect[|\Tsafe|]$.
We answer it exactly. \Cref{eq:ppo-safe} is a self-normalized ratio estimator whose weights are the
safe-step indicators, the construction weighted importance sampling uses in off-policy evaluation
\citep{precup2000eligibility}, where an $O(1/n)$ bias is knowingly accepted in exchange for variance
reduction \citep{cochran1977sampling}. Parts (i) and (ii) below are that standard behaviour specialized to
indicator weights; part (iii) is specific to the safe-region setting.
\cref{eq:ppo-safe} is written over a whole rollout for readability; the
implementation applies the same normalization within each minibatch, so the
denominator is the safe count of the minibatch rather than of the rollout.
Both are the same estimator with a different sampling unit, and the
proposition is stated for either.

\begin{proposition}[What \cref{eq:ppo-safe} computes]\label{prop:ratio-estimator}
Fix $\theta$ and consider a batch of $B$ timesteps drawn from the mixed-policy
occupancy, treated as independent draws (the standard idealization in
on-policy analyses; $B$ is the rollout length or the minibatch size according
to the sampling unit of \cref{eq:ppo-safe}). For timestep $i$ put
\[
D_i = \indicator[s_i \in \safeR],
\qquad
g_i = D_i\, \nabla_\theta \log \piMain(a_i \mid s_i)\, \Aest_i^{\mathrm{full}},
\]
so that $\sum_i D_i = |\Tsafe|$ and \cref{eq:ppo-safe} is the ratio
$\hat g_B = \bar g / \bar D$ of the two sample means. Let
$\alphaRate = \Pr[s_i \notin \safeR]$, assume $\alphaRate < 1$ and
$\Expect\|g_1\| < \infty$, and write $g^\star = \Expect[g_1]$, which
\Cref{thm:gradient} identifies with a positive multiple of
$\nabla_\theta J^{\mathrm{mix}}$ when $\Aest^{\mathrm{full}}$ is replaced by
$r(\tau)$. Then
\begin{enumerate}[nosep,leftmargin=1.5em]
  \item \textbf{(Consistency)} $\hat g_B \to g^\star / (1 - \alphaRate)$
        almost surely as $B \to \infty$. The limit is a positive scalar
        multiple of the target, so the ascent direction is unchanged and
        the factor only rescales the step.
  \item \textbf{(Finite-batch bias)} Componentwise, for any $j$ with
        $g^\star_j \neq 0$,
        \[
        \frac{\Expect[\hat g_{B,j}] - g^\star_j/(1-\alphaRate)}
             {g^\star_j/(1-\alphaRate)}
        = \frac{1}{B}\left(
            \frac{\alphaRate}{1-\alphaRate}
            - \rho_j\, \mathrm{CV}(g_{1,j})\, \mathrm{CV}(D_1)
          \right) + O(B^{-2}),
        \]
        where $\rho_j = \mathrm{Corr}(g_{1,j}, D_1)$. The bias is $O(1/B)$, and its
        second term is nonzero precisely when the safe count covaries with
        the masked gradient, which is what happens when the number of safe
        samples tracks trajectory return, recovery duration or gradient
        magnitude.
  \item \textbf{(Exactness without intervention)} If $\alphaRate = 0$ then
        $D_i \equiv 1$, the denominator equals $B$ deterministically, both
        terms above vanish identically, and \cref{eq:ppo-safe} is the standard
        \ac{PPO} batch mean.
\end{enumerate}
\end{proposition}

\begin{proof}
(i) $D_1$ is Bernoulli with mean $1 - \alphaRate > 0$ and $\Expect\|g_1\| <
\infty$, so the strong law gives $\bar g \to g^\star$ and $\bar D \to 1 -
\alphaRate$ almost surely; the continuous mapping theorem applied to
$(x,y) \mapsto x/y$, continuous at $(g^\star, 1-\alphaRate)$ because the
second coordinate is nonzero, gives the claim.

(ii) Write $u = \bar g - g^\star$ and $v = \bar D - (1-\alphaRate)$, both mean
zero with $\Expect[v^2] = \mathrm{Var}(D_1)/B$ and $\Expect[u_j v] =
\mathrm{Cov}(g_{1,j}, D_1)/B$. Expanding the ratio in $v$,
\[
\begin{aligned}
\hat g_{B,j}
&= \frac{g^\star_j + u_j}{(1-\alphaRate)\left(1 + \tfrac{v}{1-\alphaRate}\right)} \\
&= \frac{g^\star_j}{1-\alphaRate}
  + \frac{u_j}{1-\alphaRate}
  - \frac{g^\star_j v}{(1-\alphaRate)^2}
  - \frac{u_j v}{(1-\alphaRate)^2}
  + \frac{g^\star_j v^2}{(1-\alphaRate)^3}
  + \cdots
\end{aligned}
\]
Taking expectations kills the two linear terms and leaves
\[
\Expect[\hat g_{B,j}] - \frac{g^\star_j}{1-\alphaRate}
= \frac{1}{B}\left(
    \frac{g^\star_j\,\mathrm{Var}(D_1)}{(1-\alphaRate)^3}
    - \frac{\mathrm{Cov}(g_{1,j}, D_1)}{(1-\alphaRate)^2}
  \right) + O(B^{-2}).
\]
Dividing by $g^\star_j/(1-\alphaRate)$ and using
$\mathrm{Var}(D_1) = \alphaRate(1-\alphaRate)$, so that
$\mathrm{CV}^2(D_1) = \alphaRate/(1-\alphaRate)$, and
$\mathrm{Cov}(g_{1,j},D_1)/\big((1-\alphaRate) g^\star_j\big)
= \rho_j \mathrm{CV}(g_{1,j}) \mathrm{CV}(D_1)$ gives the stated form.

(iii) If $\alphaRate = 0$ then $D_1 = 1$ almost surely, so
$\mathrm{Var}(D_1) = 0$ and $\mathrm{Cov}(g_{1,j}, D_1) = 0$; the expansion in
(ii) terminates and every correction term is identically zero, not merely
$O(1/B)$.
\end{proof}

Three consequences follow. First, the leading term
$\alphaRate/\big((1-\alphaRate)B\big)$ is computable from quantities we already
log; at the minibatch size of \Cref{tab:hyperparams} and the whole-run mean
recovery rates of \Cref{app:pretraining-accounting} it is on the order of
$10^{-4}$ to $10^{-3}$ in relative terms, far below the Monte-Carlo noise of
the gradient itself. Second, part (iii) is not asymptotic; on \envHC{} and \envAnt{} the
recovery rate reaches zero at the printed precision by the end of training, so
there the ratio bias is a transient of the early curriculum rather than a
property of the converged update; on \envGo{} it persists at the residual
$\alphaRate \approx 0.048$. Third, the bias is removable at will. Normalizing
by the fixed batch size $B$ instead of $|\Tsafe|$ is exactly unbiased for
$g^\star$ itself, $(1-\alphaRate)$ times the self-normalized limit, at the
cost of an update rescaled by the realized safe fraction rather than its mean; we retain the
self-normalized form because it keeps the gradient scale stable while
$\alphaRate$ anneals, the same bias-for-variance trade that motivates the
weighted over the unweighted importance-sampling estimator.
} Early in training, when $\alphaRate$ is high, $|\Tsafe|$ can be small: with $m = |\Tsafe|$ safe steps the masked mean has variance on the order of $\sigma^2 / m$, so a small $m$ inflates the per-update gradient variance. We do not apply variance-control heuristics. The curriculum schedule (\Cref{sec:algorithm}) keeps $|\Tsafe|$ above a working threshold within the first few updates on every environment we evaluate.

The value-loss default also depends on whether the analytic recovery value (\Cref{prop:analytic}) is in use:
\begin{itemize}[nosep,leftmargin=1.5em]
  \item \textbf{Analytic V on:} train the critic on \emph{all} timesteps. Recovery-state targets under \Cref{prop:analytic} carry no Monte-Carlo variance and are exact over the segment; the only residual error is the critic's own error at the single re-entry state (failure-segment targets are exact), so unmasked is principled and gives the critic dense supervision.
  \item \textbf{Analytic V off:} prefer the masked value loss $\valueLoss(\theta) = |\Tsafe|^{-1}\sum_{t : s_t \in \safeR} (V_\theta(s_t) - \hat R_t)^2$. Otherwise the critic regresses against bootstrapped values that themselves depend on the critic at unsafe states, creating a self-consistency loop.
\end{itemize}
Our experiments use the analytic-V variant by default. \Cref{sec:ablations} ablates this choice.

\subsection{Where the changes land in the rollout/update loop}

The choices above touch a standard \ac{PPO} loop at a few points; the list below states where each is implemented.

\begin{itemize}[leftmargin=2em]
\item \textbf{Rollout, recovery branch:} when $s_t \notin \safeR$, the recovery action is queried, the action is stored in the buffer, $\log\piMain(a_t\mid s_t)$ is computed at the recovery action, and the step is flagged as recovery-controlled.
\item \textbf{Analytic-V overwrite:} \revised{for each recovery segment, the segment endpoint is identified and the value at \emph{every} segment step is overwritten, at offset $j$ with $\gamma^{\ell-j} V_\theta(s_{t+\ell})$ on success or $\gamma^{\ell-1-j} r_{\mathrm{term}}$ on failure, where $\ell$ is the segment length (\Cref{lem:induced-target}).} This happens in the rollout-buffer pass, before \ac{GAE}.
\item \textbf{\ac{GAE}} (standard, no modification): operates on the overwritten value sequence.
\item \textbf{Compatibility loss:} computes $\log\piMain(a_t^{\mathrm{rec}}\mid s_t)$ at unsafe steps and applies the chosen gate ($\sigma_k$ for the hard gate, the signed advantage $w_t^{\mathrm{soft}}$ for the soft variant). The resulting loss is added to the \ac{PPO} surrogate before backpropagation.
\item \textbf{Masked PG and value loss:} the safe-step indicator masks both the policy-gradient term and (in the learned-critic variant, without the analytic value) the value loss.
\end{itemize}

\subsection{Baseline ports}
\label{app:baseline-ports}

The two recovery-based baselines in \Cref{tab:main-results} are ports into the same on-policy loop as \methodName{}: both run inside the identical \ac{PPO} update, share the identical safe-region trigger $s \notin \safeR(d)$ with the same curriculum on $d$, and query the identical frozen \ac{SAC} recovery policy. They differ from \methodName{} only in how recovery-controlled transitions enter the rollout buffer.

\begin{itemize}[leftmargin=2em]
\item \textbf{Recovery RL port} \citep{thananjeyan2021recoveryrlsafereinforcement}: at recovery steps the buffer keeps the \emph{task policy's} action, value, and log-probability, the action-relabeling data-handling rule of the original method; recovery-step rewards are zeroed in the learning signal, as for \methodName{}. The original method treats the recovery as part of the environment dynamics, which is coherent for its off-policy Q-learning objective; the bias corrected in this paper arises only when that relabeling is carried into an on-policy policy-gradient update.
\item \textbf{Safe Legged port} \citep{yang2022safereinforcementlearninglegged}: the same relabeling, plus a reward of $r_t - 1$ stored at recovery-triggered steps, the recovery-penalty shaping component of the published method.
\end{itemize}

What is deliberately \emph{not} ported: Recovery RL's learned safety critic and off-policy training, and the model-based switching criterion of \citet{yang2022safereinforcementlearninglegged}. Replacing both with the shared distance trigger isolates each method's data-handling rule under a matched on-policy setup with an identical trigger and identical recovery controller, which is what makes the comparison interpretable. %
\revised{No baseline-specific tuning enters the \Cref{tab:main-results} port cells beyond the shared \ac{PPO} configuration of \Cref{app:hyperparams}; the bounded per-baseline tuning sweep is in \Cref{app:baseline-tuning}.}

\subsection{Pseudocode parameters}

\Cref{tab:pseudocode-params} gives the default values that instantiate \Cref{alg:safeexplorer} as run in our experiments.

\begin{table}[ht]
\centering\small
\caption{Default parameter values for \Cref{alg:safeexplorer}.}
\label{tab:pseudocode-params}
\begin{tabular}{l l}
\toprule
Symbol in \Cref{alg:safeexplorer} & Default value \\
\midrule
$N$ (total updates) & $\mathrm{total\_timesteps}/(\mathrm{num\_envs} \times \mathrm{num\_steps})$ \\
$E$ (parallel envs) & 4 \\
$T_{\mathrm{ro}}$ (rollout length) & 2048 \\
$K$ (\ac{PPO} epochs) & 10 \\
$c_v$ (value coef) & 0.5 \\
$c_e$ (entropy coef) & 0 \\
$\lambdaCompat$ (compat coef) & $10^{-3}$ \\
$\epsilon$ (\ac{PPO} clip) & 0.2 \\
$\gamma, \mathrm{GAE}\,\lambda$ & 0.99, 0.95 \\
$r_{\mathrm{term}}$ (terminal reward) & environment terminal-step reward (\Cref{app:rterm}) \\
\bottomrule
\end{tabular}
\end{table}

\section{Hyperparameters}
\label{app:hyperparams}

\subsection{Training hyperparameters (shared across all variants)}

Every method in \Cref{tab:main-results} except the two \ac{CMDP} baselines is trained with the single \ac{PPO} configuration of \Cref{tab:hyperparams}, so that any difference among them reflects the method rather than per-variant tuning. \ac{CPO} and \ac{PPO}-Lagrangian additionally carry the constraint hyperparameters their objective requires (\Cref{app:cmdp-hyperparams}). The compatibility regularizer $\lambdaCompat$ is the only entry active for \methodName{} alone. All remaining hyperparameters are shared. The per-environment $\dmax$ values were set by a qualitative criterion, large enough that the recovery policy stops triggering under normal task operation as $d$ approaches $\dmax$; no systematic search over $\dmax$ was run.

\revised{\textbf{What was and was not tuned.} No baseline-specific hyperparameter tuning enters \Cref{tab:main-results} or the ablations. The ports and every ablation variant run the shared configuration above, and the \ac{CMDP} baselines run OmniSafe defaults plus the constraint settings of \Cref{app:cmdp-hyperparams}. Two facts mitigate this. The ports share the identical \ac{PPO} loop, trigger, curriculum, and frozen recovery policy with \methodName{}, and the \ac{CMDP} baselines receive observation and reward normalization that \methodName{} does not use, which favors the baselines. A bounded per-baseline tuning sweep is reported separately in \Cref{app:baseline-tuning}.}

\begin{table}[ht]
\centering
\small
\caption{\ac{PPO} training hyperparameters for every variant in \Cref{tab:main-results}. Per-environment $\dmax$ values are in env-specific task-space units, not comparable across environments.}
\label{tab:hyperparams}
\begin{tabular}{lr}
\toprule
\textbf{Parameter} & \textbf{Value} \\
\midrule
Optimizer & Adam ($\mathrm{eps}=10^{-5}$) \\
Learning rate & $3 \times 10^{-4}$ \\
Anneal LR linearly to 0 & yes \\
Number of parallel envs & 4 \\
Rollout length per env & 2048 \\
Batch size ($\mathrm{num\_envs}\times\mathrm{num\_steps}$) & 8192 \\
Mini-batch size & batch size $/$ 32 \\
\ac{GAE} $\lambda$ & 0.95 \\
Discount factor $\gamma$ & 0.99 \\
\ac{PPO} ratio clipping $\epsilon$ & 0.2 \\
Update epochs per batch & 10 \\
Value loss coefficient $c_v$ & 0.5 \\
Max gradient norm & 0.5 \\
Entropy coefficient $c_e$ & 0 \\
Compatibility regularizer $\lambdaCompat$ (\methodName{} only) & $10^{-3}$ (see \Cref{app:lambda-sensitivity} for sweep) \\
Initial safe-region radius $d_0$ & 0.01 (\envGo: 0.05) \\
Per-env $\dmax$ & \envHC: 2.0; \envAnt: 0.4; \envGo: 0.15 \\
Per-step recovery reward in the learning signal & $0$ (zeroed; \Cref{sec:value}) \\
Terminal-step reward $r_{\mathrm{term}}$ & environment terminal reward (\Cref{app:rterm}) \\
\bottomrule
\end{tabular}
\end{table}

\subsection{Terminal-step reward \texorpdfstring{$r_{\mathrm{term}}$}{r\_term}}
\label{app:rterm}

Intermediate recovery-step rewards are zeroed in the learning signal (\Cref{sec:value}), but the terminal-step reward that the analytic-$V$ failure branch propagates as $r_{\mathrm{term}}$ (\Cref{alg:safeexplorer}) is the environment's \emph{own} reward on the terminating step, preserved in the buffer. Concretely: on \envHC the terminal step includes an explicit $-1$ penalty; on \envGo the per-step reward is floored at zero and the healthy bonus vanishes on an unhealthy step, so $r_{\mathrm{term}} = 0$ (a configurable termination-penalty flag exists but is $0$ in all main-table runs); on \envAnt the constant healthy bonus of $1.0$ is paid even on the terminal step, so its terminal reward is that bonus minus the step's control and contact costs.

\subsection{\texorpdfstring{\ac{CMDP}}{CMDP} baseline hyperparameters}
\label{app:cmdp-hyperparams}

Both \ac{CMDP} baselines run on OmniSafe 0.5.x. The constraint cost is $1$ on each unhealthy termination and $0$ otherwise; time-limit truncations incur no cost, matching the fall-counting convention of \Cref{app:env}. The cost limit is $0.05$. Shared solver settings: steps per epoch $2048$, $10$ update iterations, minibatch size $64$, target KL $0.02$, $\gamma = \gamma_{\mathrm{cost}} = 0.99$, \ac{GAE} $\lambda = \lambda_{\mathrm{cost}} = 0.95$, actor and critic MLPs of two hidden layers of size 64 with $\tanh$ activation, learning rate $3\times10^{-4}$ for both with linear decay, observation normalization on, reward normalization on, cost normalization off. \ac{PPO}-Lagrangian additionally uses ratio clip $0.2$, Lagrange multiplier initialized at $0.001$, and multiplier learning rate $0.035$. \ac{CPO} uses the OmniSafe defaults for its trust-region settings plus the cost limit above. %
\revised{No baseline-specific tuning enters \Cref{tab:main-results} beyond these settings.}

\subsection{Per-environment training budget\texorpdfstring{ \revised{and robot-time equivalents}}{ and robot-time equivalents}}

\begin{table}[ht]
\centering\small
\caption{Per-environment training budget, \methodName{} wall-clock cost, \revised{robot-time equivalents computed from the per-environment control timesteps of \Cref{app:env} ($0.05$\,s on \envHC/\envAnt, $0.01$\,s on \envGo), and the recovery-pretraining budgets with their robot-time equivalents and fall counts. Pretraining fall counts run to the deployed checkpoints ($4$M/$2$M/$5$M steps), with run-end counts in \Cref{app:pretraining-accounting}; the \envHC{} count is an early-termination count with a legacy-trainer caveat (\Cref{app:recovery-training})}.}
\label{tab:training-budget}
\setlength{\tabcolsep}{4pt}
\resizebox{\textwidth}{!}{%
\begin{tabular}{lrrrrrr}
\toprule
Environment & \revised{Training steps} & Wall-clock (4 CPU cores) & \revised{Robot-time equiv.} & \revised{Pretraining steps} & \revised{Pretraining robot-time} & \revised{Pretraining falls} \\
\midrule
\envHC & $8\times 10^6$ & $\sim$3.5 h & \revised{$\sim$111 h} & \revised{$5\times 10^6$} & \revised{$\sim$69 h} & \revised{168} \\
\envAnt & $8\times 10^6$ & $\sim$5 h & \revised{$\sim$111 h} & \revised{$3\times 10^6$} & \revised{$\sim$42 h} & \revised{17{,}723} \\
\envGo & $16\times 10^6$ & $\sim$12 h & \revised{$\sim$44 h} & \revised{$10\times 10^6$} & \revised{$\sim$28 h} & \revised{5{,}469} \\
\bottomrule
\end{tabular}}
\end{table}
The one quantity that does vary across environments is the training budget, which we fix before any rollout begins and report in \Cref{tab:training-budget}. \revised{All reported runs train on 4 CPU cores; no GPU is used. The robot-time column states what these budgets would mean on hardware. The equivalents are orders of magnitude beyond the roughly 20 minutes of real-world training reported by \citet{smith2022walkparklearningwalk}, the comparison behind the framing of \Cref{sec:intro}. Algorithmically, the additions over standard \ac{PPO} (the mask, the per-segment value overwrite, and the gated imitation term) are $O(1)$ buffer operations per step, so GPU throughput follows the underlying \ac{PPO} implementation.} The larger budget on \envGo reflects its 12-\ac{DoF} action space and longer main-task time-to-converge; the frozen \ac{SAC} recovery policy (\Cref{app:recovery-training}) does not enter this budget\revised{, and is accounted separately in \Cref{tab:training-budget} and \Cref{app:pretraining-accounting}}. Matching all environments at $16$M, or repeating \envGo at $8$M for a fully matched comparison, is left to follow-up work.

\subsection{Network architecture}

Within each run, the network that \ac{PPO} optimizes is the standard CleanRL actor/critic: the actor and critic are independent MLPs with two hidden layers of size 64 and $\tanh$ activation. The actor outputs a mean vector and a learnable global $\log\sigma$ for a diagonal-Gaussian policy. Layer weights are initialized orthogonally. The critic's output layer uses gain 1.0 and the actor's mean output uses gain 0.01.

\subsection{Curriculum schedule}

The remaining quantity that changes during a run is the safe-region radius, which is not held fixed but annealed along the linear schedule of \Cref{sec:algorithm}, interpolating $d_0 \to \dmax$ over the $N = \mathrm{total\_timesteps}/(\mathrm{num\_envs}\times\mathrm{num\_steps})$ training updates. \Cref{app:curr-shape} ablates the schedule shape (linear / log / step / constant) on \envAnt; the headline results all use the linear schedule.

\revised{\subsection{Bounded per-baseline tuning sweep}
\label{app:baseline-tuning}

We ran a bounded per-baseline tuning sweep for the port baselines on all three environments, one defining hyperparameter per baseline (the trigger radius $\dmax$ for the Recovery RL port, the shaping coefficient for the Safe Legged port), three values bracketing the default, at the full body budget. Per-cell seed counts are printed in \Cref{tab:baseline-tuning}. \revised{The \envHC{} and \envAnt{} cells are ten seeds each.}

\begin{table}[ht]\revisedfloat
\centering\small
\caption{\revised{Tuned-best and default port cells (train falls / train reward, body protocol, seed count per cell). \methodName{} reference cells are the 10-seed batch of \Cref{tab:main-results}. Radius and hyperparameter values are printed per cell. \revised{The \envGo{} rows and every reference cell use the last-$10\%$ reward window of \Cref{tab:main-results}; the \envHC{} and \envAnt{} port rows use the last-$5\%$ window this sweep was harvested at, a difference far inside seed noise on every cell measured both ways.}}}
\label{tab:baseline-tuning}
\setlength{\tabcolsep}{4pt}
\resizebox{\textwidth}{!}{%
\begin{tabular}{l l l r r}
\toprule
Env & Cell (hyperparameter, operating point) & Seeds & Falls & Reward \\
\midrule
\envHC & Recovery RL port, default radius ($\dmax{=}2.0$) & 10 & $2852\pm4883$ & $1500$ \\
\envHC & Recovery RL port, doubled radius ($\dmax{=}4.0$), tuned best & 10 & $1557\pm2905$ & $3544$ \\
\envHC & Safe Legged port, default coefficient ($c{=}1.0$) & 10 & $943\pm1582$ & $3415$ \\
\envHC & Safe Legged port, largest coefficient ($c{=}3.0$), tuned best & 10 & $738\pm932$ & $4427$ \\
\envHC & \methodName{} (reference, $\dmax{=}2.0$) & 10 & \revised{$15.3\pm14.7$} & \revised{$3855\pm958$} \\
\midrule
\envAnt & Best tuned port cell (Recovery RL, halved radius $\dmax{=}0.2$) & 10 & $15628\pm7280$ & $103$ \\
\envAnt & \methodName{} (reference, $\dmax{=}0.4$) & 10 & $866\pm718$ & $2698\pm237$ \\
\midrule
\envGo & Recovery RL port, halved radius ($\dmax{=}0.05$) & \revised{10} & \revised{$1566\pm1033$} & \revised{$3088\pm817$} \\
\envGo & Recovery RL port, default radius ($\dmax{=}0.1$) & \revised{10} & \revised{$2132\pm2164$} & \revised{$2895\pm753$} \\
\envGo & Recovery RL port, doubled radius ($\dmax{=}0.2$) & \revised{10} & \revised{$2963\pm1911$} & \revised{$3808\pm911$} \\
\envGo & Safe Legged port, $c{=}0.3$ & \revised{9} & \revised{$3651\pm1792$} & \revised{$2056\pm366$} \\
\envGo & Safe Legged port, $c{=}1.0$ (default) & \revised{10} & \revised{$4573\pm1792$} & \revised{$1977\pm335$} \\
\envGo & Safe Legged port, $c{=}3.0$, tuned best & \revised{10} & \revised{$2314\pm777$} & \revised{$2513\pm243$} \\
\envGo & \methodName{} (reference, $\dmax{=}0.15$) & 10 & \revised{$830\pm216$} & \revised{$5509\pm155$} \\
\bottomrule
\end{tabular}}

\end{table}

Tuning does improve the ports in places, but no tuned cell Pareto-dominates \methodName{}; at ten seeds no tuned Recovery RL cell exceeds \methodName's mean reward, and that port's best case, twice the trigger radius on \envHC, reads $1557 \pm 2905$ falls at reward $3544$. On \envHC{} one tuned cell sits nominally above \methodName's 10-seed reward, the Safe Legged port at its largest coefficient ($4427$ against \revised{$3855 \pm 958$}, a gap inside one seed standard deviation at ten seeds per cell), while sitting at roughly \revised{$48\times$} \methodName's falls ($738 \pm 932$ against \revised{$15.3 \pm 14.7$}), so we claim no reward lead in either direction for that cell. \revised{On \envGo{}, \methodName{} leads on both axes in every one of the six tuned cells. The reward lead is interval-separated. \methodName{} reads \revised{$5509 \pm 155$} against a best tuned cell of $3808 \pm 911$, and across all $59$ tuned runs none reaches the \methodName{} mean and none reaches even its worst seed ($5214$; the highest tuned run is $5196$). The falls lead is weaker and is not uniform across the two ports. Against all three Safe Legged cells it separates ($3651 \pm 1792$, $4573 \pm 1792$ and $2314 \pm 777$ against \revised{$830 \pm 216$}; Welch $p = $ \revised{$0.001$, $0.0001$ and $0.0001$}). Against Recovery RL it separates only at the doubled radius ($2963 \pm 1911$, $p = 0.012$) and stays unresolved at the two tighter $\dmax$ values ($1566 \pm 1033$ and $2132 \pm 2164$, $p = 0.16$ and $0.15$), so on those two cells we report an ordering of means and not a separated difference. At three seeds this sweep put the closest Recovery RL cell at $1335 \pm 761$; at ten seeds it reads $1566 \pm 1033$, and the comparison stays unresolved.} On \envAnt{} no hyperparameter value makes either port functional. Halving $\dmax$ does not reduce the Recovery RL port's \envGo{} falls by a detectable amount; against the sweep's own default arm (\revised{$2132 \pm 2164$ at ten seeds}, per-seed range $523$ to $6916$) the difference is noise.}

\section{Environment and Recovery-Policy Details}
\label{app:env}

\subsection{Environment summary}

We evaluate on three MuJoCo locomotion environments that differ in dimensionality and in how the safe region is sensed. \Cref{tab:env-summary} lists, for each one, the degrees of freedom, the observation and action dimensions, the indicator that the recovery policy uses to decide whether the agent has left the safe region $\safeR$, and the termination predicate that defines a fall. The indicator is a function of torso or base height together with tilt to the nominal pose, the signal the recovery policy is trained to restore.

\begin{table}[ht]
\centering\small
\setlength{\tabcolsep}{4pt}
\caption{Per-environment dimensions, the safe-region indicator the recovery policy uses, and the termination predicate that defines a fall.}
\label{tab:env-summary}
\resizebox{\textwidth}{!}{%
\begin{tabular}{l l l l l l}
\toprule
Environment & \ac{DoF} & Obs dim & Action dim & Safe-region indicator & Fall (termination) iff \\
\midrule
\envHC & 6 & 17 & 6 & torso height $z$ + tilt to nominal & $z \notin [-0.5, 0.5]$ or $|\theta_{\mathrm{pitch}}| > 1.57$ \\
\envAnt & 8 & 27 & 8 & torso height $z$ + tilt to nominal & $z \notin (0.3, 1.5)$ or non-finite state \\
\envGo & 12 & 49 & 12 & base $z$ + base tilt (roll, pitch) to nominal & $z \notin (0.22, 0.65)$ or $|\mathrm{roll}|$ or $|\mathrm{pitch}| > 60^{\circ}$ \\
\bottomrule
\end{tabular}%
}
\end{table}

\noindent \textbf{Termination, horizon, and fall counting.} The \envHC termination predicate is a custom addition; the stock Gymnasium \envHC never terminates, which is why our fall counts on \envHC are nonzero. %
\revised{The terminating step also carries a $-1$ reward, the only other deviation from the stock \envHC{} reward function; at most one unit per episode against returns in the thousands, it is immaterial to every comparison, but we state it because the \envHC{} success bar is imported from the stock environment (\Cref{app:fixed-bar-detail}). Our \envAnt{} departs from stock in two respects of its own: the healthy range is $(0.3, 1.5)$ rather than the stock $(0.2, 1.0)$, and the control-cost weight is $0.05$ rather than the stock $0.5$. The first terminates sooner on a collapse and the second penalizes large actions less, so they act in opposite directions on achievable return; we state them because Gymnasium's registered \envAnt{} score of $6000$ is likewise not validated for this variant.} All three environments run under a $1000$-step episode time limit. The control timestep is $0.05$\,s on \envHC and \envAnt (frame skip $5$ at a $0.01$\,s simulation step) and $0.01$\,s on \envGo (frame skip $5$ at a $0.002$\,s simulation step), so the horizons are $50$\,s and $10$\,s of simulated time respectively. A \emph{fall} is an unhealthy termination as defined in \Cref{tab:env-summary}; time-limit truncations are never counted as falls, in training and evaluation alike.

\subsection{Safe-region distance and curriculum}
\label{app:safe-region-distance}

The membership test is $s \in \safeR(d)$ iff $D(s) \le d$, with an environment-specific distance $D$ computed from the base height and orientation:
\begin{itemize}[leftmargin=2em]
\item \envHC: $D(s) = \|(z, \theta_{\mathrm{pitch}})\|_2$ against the nominal $(0, 0)$, with $z$ in meters and pitch in radians, unweighted.
\item \envAnt: $D(s) = \|(z - 0.75, \mathrm{tilt})\|_2$, where $\mathrm{tilt} = \frac{\pi}{180} \arccos(\mathrm{axis}_z) \cdot \phi/\pi$ with $\phi = 2\arccos(q_w)$, a rescaled axis-angle deviation from upright ($q_w$ is the scalar quaternion component and $\mathrm{axis}_z$ the vertical component of the rotation axis). This scaling strongly downweights orientation relative to height on \envAnt; it is an implementation choice that we state plainly and keep for fidelity with the trained runs.
\item \envGo: $D(s) = \|(z - 0.3, \mathrm{roll}, \mathrm{pitch})\|_2$, with $z$ in meters and roll and pitch in radians.
\end{itemize}
Because each $D$ mixes meters and (rescaled) radians in a single Euclidean norm, $d$ and $\dmax$ carry the same mixed units and are therefore not comparable across environments, the reason \Cref{tab:hyperparams} reports them as env-specific values. During training, $d$ follows the linear curriculum $d = d_0 + \frac{u-1}{N}\,\dmax$ at update $u$ of $N$ updates (\Cref{alg:safeexplorer}), with the per-environment $(d_0, \dmax)$ of \Cref{tab:hyperparams}.

\subsection{Recovery-policy training}
\label{app:recovery-training}

For each environment, the recovery policy $\piRec$ is a separately pre-trained \ac{SAC} actor, trained on the same MuJoCo dynamics in a recovery-flavored configuration: the forward and velocity-tracking task terms of the environment reward are zeroed, so the remaining reward favors surviving and staying near the nominal pose. The shared recipe is CleanRL \ac{SAC} with twin $256$-$256$ $Q$ networks, a $256$-$256$ tanh-squashed Gaussian actor, replay buffer of $10^6$ transitions, $\gamma = 0.99$, batch size $256$, and an auto-tuned entropy coefficient; recovery-training episodes are capped at $250$ steps, and at deployment we use the actor's deterministic tanh-mean action. Recovery training is \emph{not} reset-free. At every episode reset, the initial state is aggressively randomized to mimic the failure modes a learning task policy will produce.

\noindent \textbf{Initial-state randomization (per env reset):}
\begin{itemize}[leftmargin=2em]
\item \textbf{Joint positions:} sampled uniformly over the full mechanical joint range (or a large subset). For \envAnt, each joint has env-specific bounds (e.g., $\pm 0.52$ to $\pm 1.23$ rad from nominal). For \envGo, all 12 joint positions are drawn from the full actuator range.
\item \textbf{Base height (and optionally $x, y$):} base $z$ is randomized over a range that includes low and near-fall configurations (e.g., $z \in [0.22, 0.45]$ for \envGo).
\item \textbf{Base orientation:} the base quaternion is randomized to produce varied roll, pitch, and yaw. For \envAnt, a random quaternion is generated and constrained so the forward vector has non-negative $z$ (``face up''). For \envGo, roll and pitch are drawn in $\pm \pi/10$ and yaw in $[0, 2\pi]$, again with a face-up constraint.
\item \textbf{Velocities:} initialized with small noise around zero.
\end{itemize}

\noindent \textbf{Per-environment training and cost.} On \envHC the recovery policy was trained on the stock environment with the forward reward disabled, for $5$M environment steps; the deployed checkpoint is the $4$M-step snapshot. This policy was trained with an earlier revision of the trainer, so we report its checkpoint provenance rather than a reconstructed reward. \revised{That revision also terminated episodes early on a predicate that is not preserved in the repository; the $168$ pretraining falls of \Cref{app:pretraining-accounting} count those early terminations, which is why the count is nonzero although the stock environment never terminates (\Cref{app:env}).} On \envAnt it was trained for $3$M steps ($2$M-step snapshot deployed; policy learning rate $10^{-4}$, $Q$ learning rate $3\times10^{-4}$, target-smoothing $\tau = 0.002$) on the recovery-flavored environment, whose effective reward is the healthy bonus minus control-magnitude and contact costs, so the \envAnt recovery is torque-limited. On \envGo it was trained for $10$M steps ($5$M-step snapshot deployed; $\tau = 0.005$, $Q$ learning rate $10^{-3}$, policy learning rate $3\times10^{-4}$) with the velocity-tracking terms zeroed; the effective reward is the healthy indicator minus quadratic pose, height, and joint-deviation costs, floored at zero. Recovery pre-training therefore costs $5$M (\envHC), $3$M (\envAnt), and $10$M (\envGo) environment steps. %
These steps are not included in the training budgets of \Cref{tab:training-budget}\revised{'s first column, but are reported there alongside them,} and the falls incurred during recovery pre-training are \revised{counted from the pretraining logs in \Cref{app:pretraining-accounting}}. The cost is nonetheless bounded: at most $62.5\%$ of a single $8$M-step \envHC or \envAnt training run ($5$M of $8$M) and $62.5\%$ of the $16$M-step \envGo budget ($10$M of $16$M), which bounds the sense in which the recovery policy is cheap to obtain.

\noindent \textbf{Per-environment recovery quality.} How well this training regime works depends on the environment. Empirically, the \envHC and \envGo recoveries succeed at stabilization with high probability across the randomized initial states. The \envAnt recovery policy does not, producing the unreliable-recovery regime that is the most informative axis of variation in our results (see \Cref{sec:exp-main}).
We attribute this to the recovery policy's reward design rather than to undertraining: the control-magnitude penalty ($0.5\|a\|^2$ in the \envAnt reward) caps how much torque the policy can apply, and the harder 3D \envAnt stabilization needs exactly those aggressive corrections. Relaxing that penalty or adopting a stronger recovery class for \envAnt, and confirming the attribution with a longer-training control, is left to follow-up work. %
\revised{We keep the imperfect \envAnt{} recovery deliberately rather than retraining it. On hardware the recovery controller will be imperfect for reasons outside the learner's control, actuator and torque limits among them, and the regime where recovery sometimes fails is the one in which the outcome-gated imitation term becomes visible at all (the $20.9$-point survival gap in \Cref{tab:operational-recovery}, and the gated variant is the only one to reach the \envAnt{} success bar on every seed). Retraining improves one instance; the method is about training safely with whatever recovery controller one has. The practitioner guidance we state instead is a catch-rate screen. Before any task training is spent, score a candidate recovery offline by its single-shot catch rate from doomed handover states collected at the intended trigger band, a cheap measurement that in our \envWalker{} study predicts the in-loop outcome where the recovery's own training return does not (\Cref{app:walker2d}). Because the screen's handover states are doomed by construction, its catch rate is a stricter quantity than the operational survival rates of \Cref{tab:operational-recovery}, which average over every triggered exit, and the two percentages are not comparable.}

\noindent \textbf{Recovery quality on the operational distribution.} The randomized-reset measurement above grades the recovery on states the designer samples, not on the states the learning policy actually drives it into (which, unlike the face-up-constrained reset distribution, can include face-down configurations), so the operational grade is, if anything, the harder of the two. To measure the latter, during a trained run we log every first exit from $\safeR$ (the state at which recovery is triggered), then replay the \ac{SAC} recovery from each logged state and record the fraction that \emph{survive}, meaning the recovery completes the segment without an episode termination. We do this for \methodName{} and for the masked, analytic-$V$ ablation variant of \Cref{sec:ablations}, which shares the recovery policy and the curriculum with \methodName{} and differs from it only by the hard outcome gate $\Ctheta$ (\cref{eq:hard-gate}). \Cref{tab:operational-recovery} aggregates all logged triggers of one representative run per cell. Two effects stand out. First, \methodName{} leaves $\safeR$ between $2.1\times$ and $4.2\times$ less often, so it calls on the recovery far less. Second, on \envAnt, where the recovery policy is unreliable, the exits \methodName{} produces survive $91.3\%$ of the time against $70.4\%$ for the ungated variant, a gap of $20.9$ points. On \envHC and \envGo, where the recovery policy is reliable, both variants survive at essentially the same rate (about $100\%$ and $97\%$). The \envAnt gap offers an operational reading of the \revised{$+2524$} reward the hard gate adds on \envAnt in \Cref{sec:ablations}. The gate biases the main policy toward the part of state space where the weak \envAnt recovery policy actually succeeds. The trigger counts in \Cref{tab:operational-recovery} place the binomial $95\%$ confidence interval on each survival rate within $\pm 0.5$ points. These intervals reflect sampling over triggers within one run, not variation across training seeds, which we do not characterize here; the table is therefore a single-run illustration rather than a multi-seed estimate.

\begin{table}[ht]
\centering\small
\caption{Recovery quality on the \emph{operational} trigger-state distribution. We replay the \ac{SAC} recovery from every logged first exit from $\safeR$ of one representative trained run per cell. ``Exits'' counts the logged recovery triggers; ``Survived'' is the percentage of those exits the recovery completes without termination. The ungated variant is the masked, analytic-$V$ variant of \Cref{sec:ablations}.}
\label{tab:operational-recovery}
\begin{tabular}{l rr rr}
\toprule
 & \multicolumn{2}{c}{Exits (triggers)} & \multicolumn{2}{c}{Survived (\%)} \\
\cmidrule(lr){2-3}\cmidrule(lr){4-5}
Environment & \methodName & ungated & \methodName & ungated \\
\midrule
\envHC  & 80{,}546  & 338{,}671 & 99.8 & 100.0 \\
\envAnt & 16{,}366  & 35{,}150  & 91.3 & 70.4 \\
\envGo  & 94{,}212  & 255{,}881 & 97.0 & 97.4 \\
\bottomrule
\end{tabular}
\end{table}

\revised{\subsection{End-to-end safety accounting}
\label{app:pretraining-accounting}

All reported safety results are conditional on an already available recovery controller; this subsection accounts for that condition.

\textbf{Pretraining falls.} Counted from the pretraining logs up to the deployed checkpoints ($4$M steps on \envHC, $2$M on \envAnt, $5$M on \envGo), the counts are $168$ falls on \envHC{} (an early-termination count from a legacy revision of the trainer, the caveat stated in \Cref{app:recovery-training}), $17{,}723$ on \envAnt, and $5{,}469$ on \envGo. The two \envWalker{} recovery checkpoints incurred $586$ and $1{,}027$ pretraining falls (the winning \envWalker{} configuration uses the second). Recovery pretraining is standard episodic \ac{RL} with resets from aggressively randomized initial states, run in simulation or, on hardware, in a controlled rig, so the training falls are not eliminated end to end; they are moved into a phase where failures are cheap and controlled.

\textbf{Intervention statistics.} From the existing training logs and the replay machinery behind \Cref{tab:operational-recovery}, the recovery engages on a whole-run mean of $7.8\%$ (\envHC), $27.6\%$ (\envAnt), and $12.9\%$ (\envGo) of training steps, annealing to zero on \envHC{} and \envAnt{} and to about $5\%$ on \envGo; mean recovery-segment durations are $4.9$, $159$, and roughly $20$ steps respectively; and failed-recovery rates on the operational trigger distribution are $0.21\%$, $8.72\%$, and $2.96\%$ per segment (the \envGo{} rate is an upper bound), consistent with the recovery survival rates of \Cref{tab:operational-recovery}. Two log-coverage caveats apply. The dense \envGo{} logs cover the first $6.2$M steps, so its whole-run engagement mean is an estimate and per-quartile durations are not available there; and per-segment length histograms are not recoverable from the existing logs on any environment, so we report exact mean durations per training phase rather than full distributions.

\textbf{End-to-end ratios and the reuse curve.} With the recovery pretraining falls charged in full to a single training run, the end-to-end fall reduction versus \ac{PPO} on the \Cref{tab:main-results} cells is $21.3\times$ on \envHC, $2.0\times$ on \envAnt, and $2.8\times$ on \envGo; counted to the ends of the executed pretraining runs rather than to the deployed checkpoints ($199$, $21{,}425$, and $5{,}798$ falls), the same charge reads $18.2\times$, $1.7\times$, and $2.6\times$. Because the frozen checkpoint is shared by every recovery-based run, the acquisition cost amortizes. As a function of the number $n$ of runs sharing the checkpoint, the checkpoint-basis reduction rises from its single-run charge above to $200\times$ (\envHC), $29\times$ (\envAnt), and $18\times$ (\envGo) at $n = 40$, the number of recovery-using runs in this paper's own study, approaching the training-only ratios of \Cref{tab:main-results} with further reuse.}

\subsection{\envGo{} environment specification}
\label{app:go1-spec}

The \envGo observation is $49$-dimensional, in order: base linear velocity scaled by $2.0$ ($3$), base angular velocity scaled by $0.25$ ($3$), base height $z$ ($1$), roll, pitch, yaw ($3$), commanded velocity scaled by $2.0$ ($3$), joint positions minus their defaults ($12$), joint velocities scaled by $0.05$ ($12$), and the previous action ($12$); the vector is clipped to $\pm 100$. The action is $12$ absolute joint-position targets driving MuJoCo position servos with proportional gain $k_p = 20$, joint damping $1.0$, and an actuator force range of $\pm 23.7$\,N\,m ($\pm 35.55$\,N\,m at the knees). The commanded velocity $(v_x, v_y, \omega_z)$ is resampled once per episode, each component uniform in $[-0.5, 0.5]$. The reward is $4.0\,\exp(-\|v^{\mathrm{cmd}}_{xy} - v_{xy}\|^2/0.25) + 2.0\,\exp(-(\omega^{\mathrm{cmd}}_z - \omega_z)^2/0.25)$, plus a healthy bonus of $1.0$ and a feet-air-time term of weight $1.0$, minus costs with weights: torque $2\times10^{-4}$, vertical velocity $2.0$, $xy$ angular velocity $0.05$, action rate $0.001$, joint-limit violation $0.01$, joint acceleration $1.25\times10^{-9}$, orientation $0.4$, height $4.0\,|z - 0.3|$, and joint deviation $0.1$; the total is floored at $0$.

\subsection{Custom XML modifications}

Supporting the recovery training and value conventions above requires three changes to the environment definitions. Our custom MuJoCo XMLs extend the standard Gymnasium MuJoCo environments with: (i) explicit joint-range attributes used by the recovery policy's randomized-reset training (\Cref{app:recovery-training}), (ii) a configurable flag that zeros the training reward outside $\safeR$ (the learning-signal convention of \Cref{sec:value}; reported returns are unmodified), and (iii) for \envGo, the \texttt{mujoco\_menagerie} Unitree Go1 model \citep{menagerie2022mujoco} with a custom 12-\ac{DoF} position-actuator configuration and a velocity-tracking reward adapted from \texttt{legged\_gym} \citep{rudin2022learningwalkminutesusing}.

\section{Additional Results}
\label{app:extra-results}

\textbf{Run-batch provenance.} The sensitivity studies in this appendix come from dedicated sweep batches run independently of the \Cref{tab:main-results} runs; each caption states its batch, operating point, and seed count. Cumulative falls are heavy-tailed across seeds, so independent batches at identical settings can differ noticeably in absolute falls while preserving the orderings the text relies on. Where a table's operating point differs from the body protocol of \Cref{sec:exp-setup}, specifically the wider $\dmax = 3.2$ on \envHC{} and $0.56$ on \envAnt{} used by the \revised{sensitivity sweeps}, the caption says so; within-table comparisons remain matched. \revised{The sweeps at these wider $\dmax$ values measure the effect of the trigger radius; they are not a second answer to the body-protocol question. On \envGo{} the sensitivity radius is wide enough that the trigger almost never fires, so the method reduces to \ac{PPO} by construction and the reduction there is $1.6\times$ $[1.4, 1.8]$, a property of the radius and not of the estimator.}

\subsection{Headline comparison: rliable interquartile-mean intervals}

This appendix supplies the evidence behind every claim made in the body, ordered as a descent from the headline numbers to the design choices that produce them, the robustness checks that stress them, and the theory and baselines that justify them. We begin with the headline. \Cref{fig:headline-iqm} gives the per-metric reward and total-falls intervals for the six comparison-set methods, complementing \revised{the single-bar falls-to-success view} (\Cref{fig:headline}). It shows rliable interquartile means with $95\%$ stratified-bootstrap confidence intervals \citep{agarwal2021deep}, drawn as ranked intervals so the falls and reward orderings read directly.

\begin{figure}[ht]
\centering
\revisedgraphics[width=\linewidth]{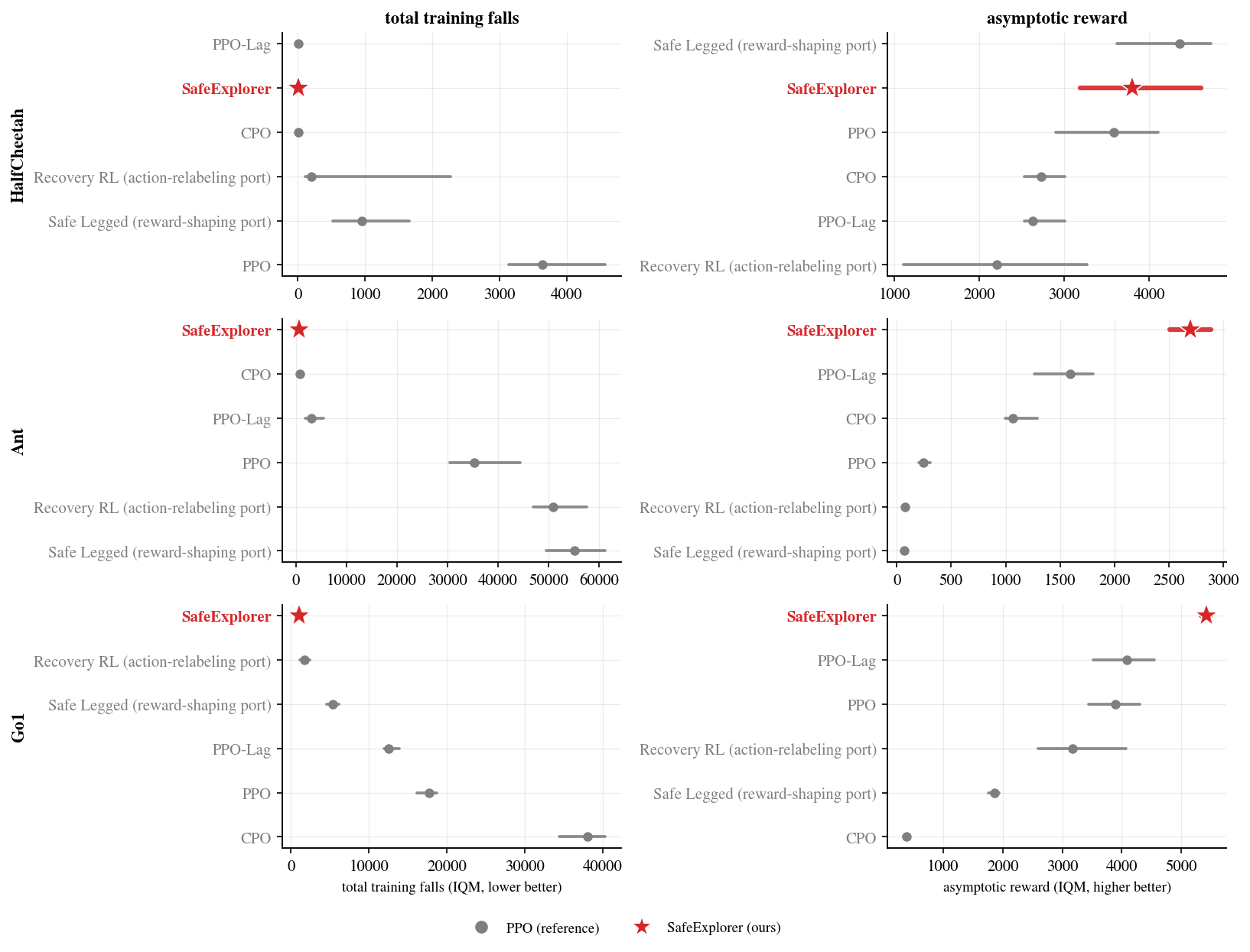}
\caption{rliable interquartile-mean intervals of reward and total falls for the six methods in \Cref{tab:main-results}, with $95\%$ stratified-bootstrap confidence intervals; a per-metric companion to \revised{the single-bar falls-to-success view} (\Cref{fig:headline}).}
\label{fig:headline-iqm}
\end{figure}

\revised{\subsection{Fixed-bar metric: fine print and per-seed detail}
\label{app:fixed-bar-detail}

This subsection carries the fine print of the metric of \Cref{sec:exp-setup} (the threshold profile itself is \Cref{fig:profile}, in the body), with the per-seed detail behind \Cref{subsec:tradeoff} and \Cref{tab:main-results}.

\begin{figure}[ht]
\centering
\revisedgraphics[width=\linewidth]{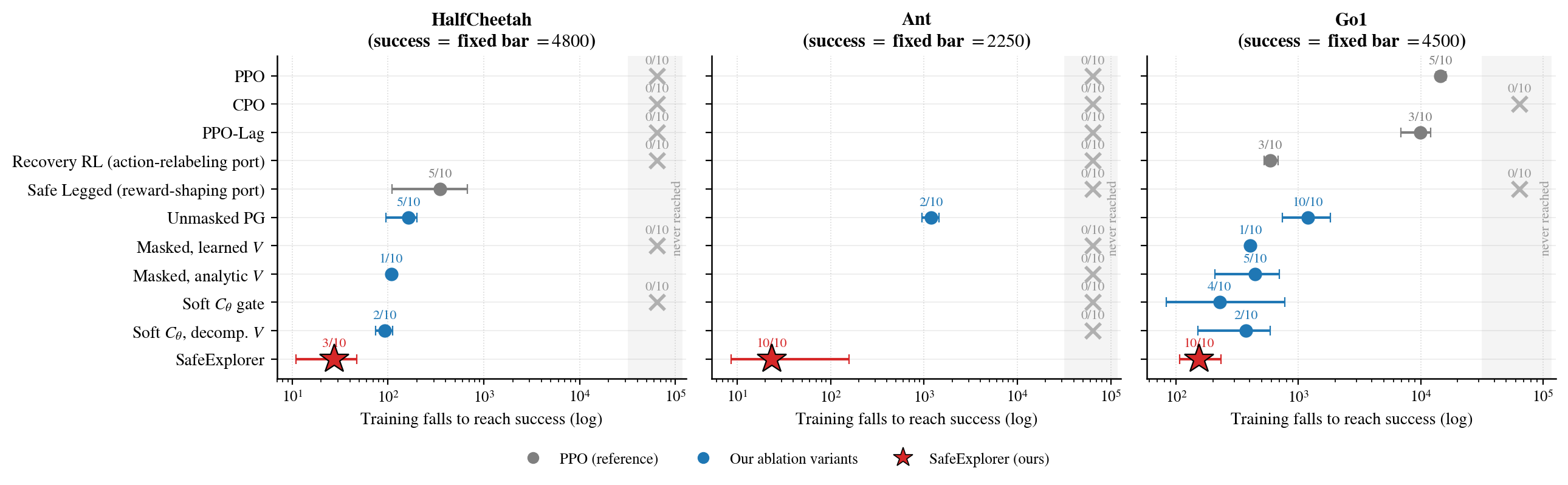}
\caption{\revised{The single-bar view ($4800$ on \envHC, $2250$ on \envAnt, $4500$ on \envGo), carrying every ablation-ladder variant, which \Cref{fig:profile} omits for legibility. The 10-seed batch of \Cref{tab:main-results}, every ladder variant likewise evaluated at ten seeds (the unmasked variant on its own batch, \Cref{tab:soft-variants}); interquartile mean over succeeding seeds with $95\%$ bootstrap confidence intervals \citep{agarwal2021deep}; labels show $k/n$, never-reachers censored and marked at right; log axis, fewer falls is better. Baselines gray, ladder variants (\Cref{sec:ablations}) blue, \methodName{} the red star; on \envHC{} the ranking claim is carried by the threshold profile.}}\label{fig:headline}
\end{figure}

\textbf{Metric fine print.} \revised{The \envHC{} bar ($4800$) is imported from the stock Gymnasium score; the \envAnt{} and \envGo{} bars ($2250$ and $4500$) are set by us.} \revised{Neither our \envHC{} nor our \envAnt{} is the stock Gymnasium environment (\Cref{app:env}), so the registered scores of $4800$ and $6000$ are transplanted rather than validated here; $6000$ in particular sits above anything any method reaches, so adopting it would leave every method unsolved, and \envGo{} has no registered score at all.} \revised{Since none of the three is externally validated for our variants, \Cref{fig:profile} reports the metric at seven thresholds per environment spanning the range in which any method reaches it; the three bars anchor the numbers tabulated here.} \revised{That externality carries one qualification. The value $4800$ is the conventional solved score for the \emph{stock} Gymnasium \envHC, which never terminates; our variant adds the termination predicate of \Cref{app:env}, so episodes can end before the $1000$-step limit and the same score is strictly harder to reach here. No method in \Cref{tab:main-results} reaches it in mean, whereas \ac{PPO} clears it routinely on the stock environment. We keep the bar unchanged rather than recalibrate it to this variant, because changing a threshold after observing that it is unflattering is precisely the failure a fixed bar is adopted to prevent. The consequence is that \envHC{} is measured against a deliberately conservative anchor, which is why its solve counts are low for every method and why the threshold profile, not this bar, is the discriminating instrument there.} \revised{On \envAnt{} and \envGo, where we set the bar, the outcome is saturated (\methodName{} at $10$ of $10$ seeds, every baseline censored on all $10$ on \envAnt{} and on at least half of its seeds on \envGo), so no bar in a wide interval changes the comparison. On \envHC{} the imported bar sits above every method's mean final reward and cannot discriminate.} Falls-to-success is the cumulative training-fall count at the first logged point where a trailing mean of the reward curve, window $\mathrm{clip}(N/20, 3, 25)$ logged points (about $5\%$ of the series), reaches the bar \revised{and stays at or above it for the whole of that window}; the fall count at the crossing is obtained by linear interpolation of the monotone cumulative-falls curve. \revised{Dropping the persistence requirement, so that a single crossing counts, would lower \methodName{}'s \envAnt{} falls-to-success from $24$ to $19$, its \envGo{} figure from $154$ to $120$ and its \envHC{} figure from $27.3$ to $27.0$, while raising several baselines' solve counts (\envGo{} \ac{PPO} from $5/10$ to $10/10$, the \envHC{} Safe Legged port from $5/10$ to $6/10$). We keep the persistent rule because a single crossing of a trailing mean can be produced by one noisy evaluation.} \revised{Of the two natural treatments for never-solving seeds, censoring and assignment of the full budget, we take censoring. The metric is denominated in falls rather than steps, so assigning a failed run its full-budget fall count would score a policy that never solves and never falls as the safest method present; a censored lower bound carried with the $k/n$ solve count cannot.}

\textbf{Per-seed detail.} On \envGo{} \methodName{} solves $10/10$ at $154$ falls, with every baseline's censored lower bound $10\times$ to $246\times$ higher. \revised{On \envAnt{} \methodName{} reaches the $2250$ bar on $10$ of $10$ seeds; each of the five baselines is censored on all $10$; and the unmasked variant reaches it on $2$ of $10$. The remaining four ladder variants are censored on all $10$ seeds, their \envAnt{} rewards far below the bar (\Cref{tab:soft-variants}).} On \envHC{} \methodName{} reaches the Gym bar on \revised{$3$ of $10$} seeds, the Safe Legged port on $5$ of $10$, and the unmasked variant on $5$ of $10$; neither difference from \methodName{} is significant (Fisher exact \revised{$p = 0.65$} for both). Among seeds that solve, \methodName{} pays $11$ to $47$ falls against $88$ to $200$ for the unmasked variant and $80$ to $698$ for the Safe Legged port\revised{, a complete separation from the port with probability $1/56$ under seed exchangeability}; the \envHC{} threshold-sweep ranking is a ranking of censored lower bounds, which is why it is not headlined as a fewest-falls claim.

\textbf{Provenance of the reduction ratios.} All three ratios take both sides from the 10-seed batch of \Cref{tab:main-results}, reading \envHC{} $255\times$ $[148, 491]$, \envAnt{} $43\times$ $[27, 79]$ and \envGo{} $21\times$ $[18, 25]$ ($95\%$ bootstrap intervals), and the threshold sweep brackets each bar, its lowest point between $0.3\times$ and $0.5\times$ of the bar and its highest between $1.1\times$ and $1.3\times$, across the three environments. The every-seed dominance statement holds in all three environments. Of the bolded leads in \Cref{tab:main-results}, two are nominal with overlapping intervals, the Safe Legged port's \envHC{} reward over \methodName{}'s and \ac{CPO}'s \envAnt{} falls over \methodName{}'s.}

\subsection{Per-environment learning curves}

The interval plot summarizes the endpoints; the learning curves show how each method gets there. \Cref{fig:curves} traces reward and cumulative falls across the full training run, the trajectory view of the headline endpoints just reported.

\begin{figure}[ht]
\centering
\revisedgraphics[width=\linewidth]{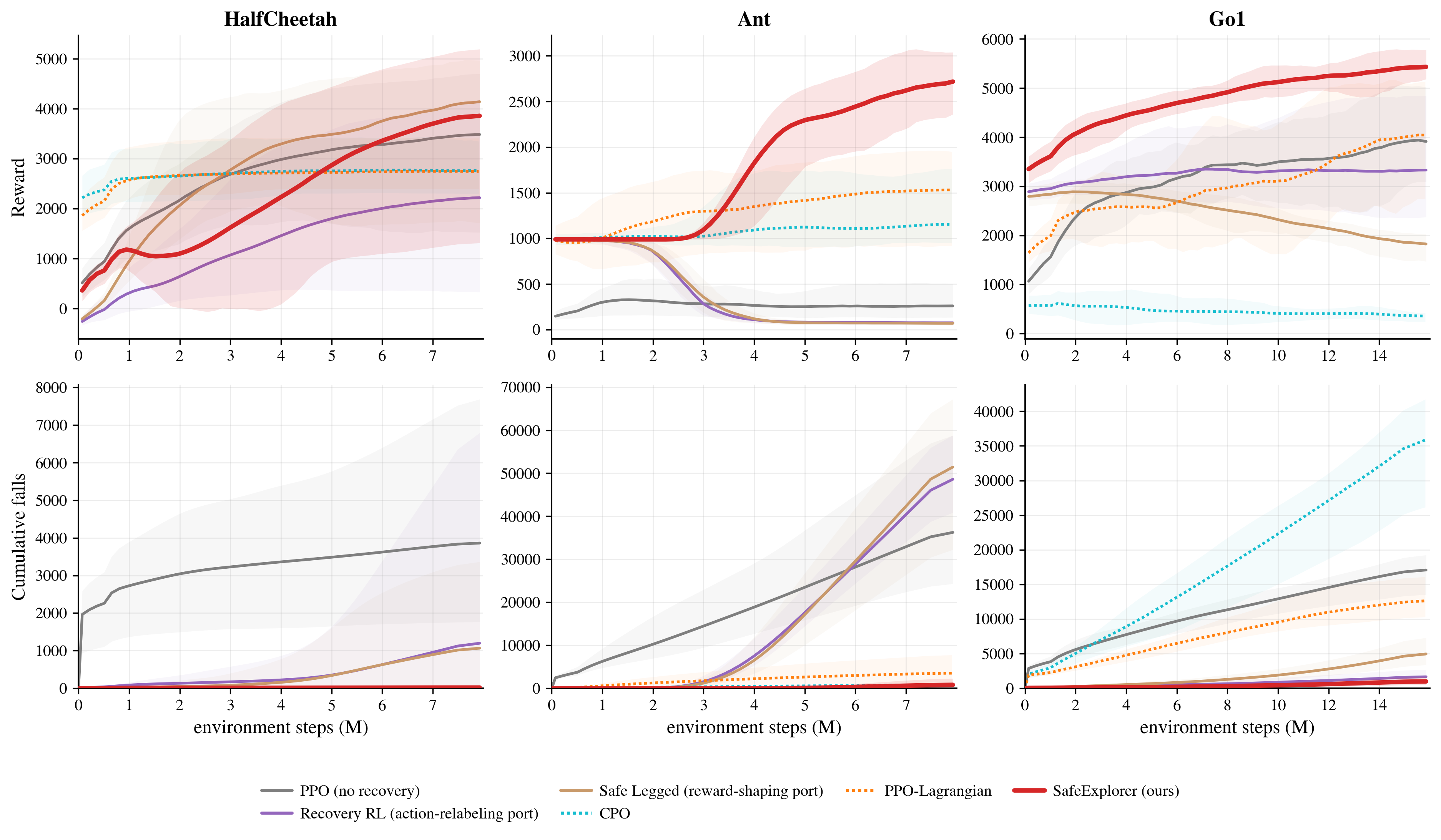}
\caption{Per-environment learning curves: rows are episodic reward and cumulative training-time falls; columns are \envHC, \envAnt, \envGo. Solid = mean across seeds, shaded = [min, max] envelope; curves smoothed with a moving average for display. Recovery-rate curves are deferred to \Cref{app:thm2-check}.}\label{fig:curves}
\end{figure}

\subsection{Additional ablations}
\label{app:ablation-ladder}

Having established what the headline shows, we now ask which design choices produce it, starting from the component ablation. \Cref{sec:ablations} analyzes that ablation in the body. \Cref{fig:ablation-pareto} gives its reward-vs-falls view, the trade-off complement to the falls-to-success ranking of \Cref{fig:headline}; here we also report the remaining variants. The two threads opened here are picked up in turn below: the soft-gate variants (\Cref{app:soft-variants}) help marginally on \envAnt and hurt on \envHC, and the $\lambdaCompat$ sweep is in \Cref{app:lambda-sensitivity}.

\begin{figure}[ht]
\centering
\revisedgraphics[width=\linewidth]{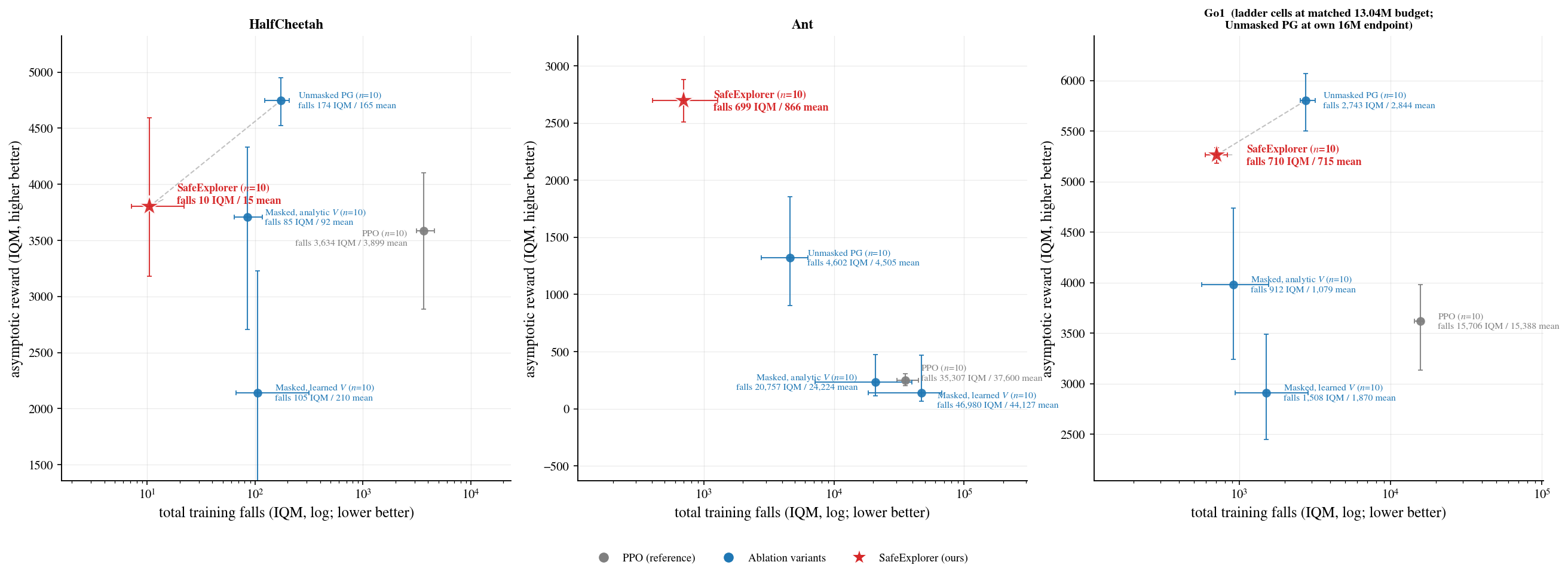}
\caption{\revised{Ablation ladder on the reward-vs-falls plane, the 10-seed batch of \Cref{tab:main-results}, single reward pipeline (its last-$10\%$ convention). Markers are rliable interquartile means with $95\%$ stratified-bootstrap confidence intervals \citep{agarwal2021deep}; per-point seed counts are printed (the unmasked variant on its own $10$ seeds, \Cref{tab:soft-variants}). Total training falls on a log axis, upper-left is better. \methodName{} (red star) is Pareto-non-dominated in every environment, under both the IQM and the mean convention, including against the \ac{CMDP} baselines of \Cref{tab:main-results}.}}
\label{fig:ablation-pareto}
\end{figure}

\revised{The \envGo{} ladder cells of \Cref{fig:ablation-pareto} are read at the matched $13.04$M budget, the largest at which all ten seeds of each cell are comparable.}

\revised{\subsection{Crossed ablation cells: masking \texorpdfstring{$\times$}{x} analytic value \texorpdfstring{$\times$}{x} outcome gate}
\label{app:crossed-cells}

\Cref{tab:crossed-cells} reports the four crossed component cells in two blocks whose protocols differ and must not be conflated; each row carries its protocol, batch, and seed count.

\textbf{Recovery-disabled block (end-of-training recovery-disabled evaluation reward, wider sensitivity $\dmax$ values).} The gated learned-critic variant (masked + compatibility imitation with a learned critic in place of the analytic value) is within seed noise of \methodName{} on \envAnt{} and \envHC{} and trails it on \envGo{}, the one environment where the analytic value adds on top of the gate. On \envHC{} and \envAnt{} the outcome-gated imitation, not the analytic value, does most of the work, and the analytic value without the gate does not improve on the masked-only variant's evaluation reward on any environment at this protocol. At the body operating point, however, the \envGo{} gate-off contrast runs the other way (\Cref{tab:recovery-disabled-all}), so the analytic value's standalone payoff is environment- and radius-dependent rather than uniformly absent. The components interact rather than add. A second independent 10-seed batch at the same protocol replicates every shared cell within one standard deviation (\Cref{app:recovery-disabled}).

\textbf{Unmasked block (8M steps, body operating point, train-side).} The two cells, unmasked + analytic value and unmasked + compatibility imitation, were run with a flag-gated lift of the Theorem 1 mask (the masked setting is bit-exact with the unmodified trainer, so the cells differ from their controls in the mask alone). Both collapse in expectation, and the failure is bimodal rather than uniform across seeds. On \envHC{} the analytic-value cell ends with $9$ of $10$ seeds at negative return and its one partial recovery below every seed of the same-batch masked control; the imitation cell ends with $6$ of $10$ seeds negative, its best recovered seed still $2.3\times$ below the same-batch masked control's worst; so no unmasked seed reaches either matched control's mean. The falls column points the same way but only as a direction; its ten-seed standard deviations exceed its means, and on \envHC{} a collapsed policy barely moves and therefore barely terminates, so a low fall count there records inactivity rather than safety, which is why reward is the unambiguous axis. On \envGo{} the unmasked cells show the same bimodal shape, one analytic-value seed partially recovering and no imitation seed doing so, with falls catastrophic in both cells. With the same-batch masked controls, all four masked-versus-unmasked comparisons on \envGo{} separate, by $82\times$ and $164\times$ on falls and $24\times$ and $59\times$ on reward, with no overlap at one standard deviation on either axis. \envHC{} separates on all four on means as well, whereas \envAnt{} separates on the imitation row only, its masked analytic-value control being the high-variance outlier. Against \methodName{}'s \Cref{tab:main-results} \envGo{} cell each single masked component reaches lower reward, so here too the components interact rather than add. The mechanism is the conditioning failure the \ac{IS} appendix documents; recovery actions lie far in the tail of the task policy, so with the mask removed the likelihood ratios at recovery steps explode and the policy-gradient loss reaches $10^{7}$ against roughly $10^{-3}$ in the masked runs; the blowup develops over the first $10^5$ to $6 \times 10^5$ steps rather than at initialization, so it is a training-dynamics failure rather than a startup one. The ladder's Unmasked-PG variant (\Cref{tab:soft-variants}) is a \emph{different} estimator; it keeps the log-probability of the action the main policy sampled while relabeling the action, so its recovery-step ratios collapse toward zero instead of exploding and it trains normally; the two cells here take the analytic-value and imitation estimator literally and delete only the Theorem 1 mask.

\begin{table}[ht]\revisedfloat
\centering\small
\caption{\revised{Crossed component cells in two protocol blocks at $10$ seeds per cell: recovery-disabled (end-of-training recovery-disabled evaluation reward, wider sensitivity $\dmax$, $3.2$ \envHC{} / $0.56$ \envAnt{} / $2.2$ \envGo) and train-side (reward and falls at the body operating point, the four crossed cells and their same-batch masked controls). Spreads are sample standard deviations.}}
\label{tab:crossed-cells}
\setlength{\tabcolsep}{4pt}
\resizebox{\textwidth}{!}{%
\begin{tabular}{l l l l l}
\toprule
Cell & Protocol & \envHC & \envAnt & \envGo \\
\midrule
Masked + gated imitation, learned $V$ & rec-disabled eval reward, sens.\ $\dmax$ & $4658\pm1084$ & $2950\pm623$ & $4106\pm901$ \\
\methodName{} (masked + gate + analytic $V$) & rec-disabled eval reward, sens.\ $\dmax$ & $4859\pm336$ & $2801\pm487$ & $5173\pm618$ \\
\midrule
Unmasked + analytic $V$ & train-side reward, body $\dmax$ & $-552\pm847$ & $382\pm243$ & $193\pm428$ \\
Unmasked + analytic $V$ & train-side falls, body $\dmax$ & $5267\pm10914$ & $13655\pm24718$ & $117340\pm57749$ \\
Unmasked + gated imitation & train-side reward, body $\dmax$ & $-233\pm1377$ & $395\pm231$ & $86\pm52$ \\
Unmasked + gated imitation & train-side falls, body $\dmax$ & $6621\pm11466$ & $24584\pm38407$ & $102435\pm41216$ \\
Masked + analytic $V$ & train-side reward, body $\dmax$ & $2608\pm1090$ & $460\pm502$ & $4641\pm1155$ \\
Masked + analytic $V$ & train-side falls, body $\dmax$ & $243\pm384$ & $30321\pm34149$ & $1436\pm1136$ \\
Masked + gated imitation & train-side reward, body $\dmax$ & $4278\pm633$ & $2866\pm403$ & $5049\pm537$ \\
Masked + gated imitation & train-side falls, body $\dmax$ & $72\pm184$ & $662\pm631$ & $626\pm284$ \\
\bottomrule
\end{tabular}}

\end{table}}

\subsection{Soft compatibility-gate variants}
\label{app:soft-variants}

The body ablation (\Cref{sec:ablations}, \Cref{fig:headline}) analyzes the clean ablation ladder Unmasked PG $\to$ masked learned $V$ $\to$ analytic $V$ $\to$ hard $\Ctheta$ (\methodName{}); \Cref{tab:soft-variants} tabulates its full numbers. The two soft-gate variants, a soft re-entry-advantage gate $w_t^{\mathrm{soft}}$ (\cref{eq:soft-gate}) applied on the analytic-$V$ model, and the same soft gate applied on a decomposed value, are negative results. They help marginally on \envAnt and hurt on \envHC, and never approach \methodName{}'s safety. We report their full numbers here and omit them from the body ablation to keep it focused on variants that isolate a single design choice. \Cref{tab:soft-variants} places them alongside the ladder variants so the off $\to$ soft $\to$ hard progression is visible\revised{; all rows are the 5-seed body batch except the \methodName{} row (the 10-seed batch of \Cref{tab:main-results}) and the Unmasked PG cells (their own $10$ seeds). The three lowest \envGo{} falls cells lie within seed noise of one another and span two batches, so that column orders means rather than separating them. The falls-to-success views (\Cref{fig:headline}, \Cref{fig:ablation-pareto}) evaluate every ladder variant at ten seeds; this table's reward and falls cells keep the batches stated above}.

\begin{table}[ht]
\centering
\caption{Soft-gate ablations against the full ladder: Unmasked PG (the biased no-masking baseline of \Cref{sec:ablations}\revised{, added here so all variants appear in one table}), masked learned-$V$, analytic-$V$, and the hard gate. The soft gate is $w_t^{\mathrm{soft}} = \gamma^k V(s_{t+k}) - V(s_t)$. \textbf{Bold} marks the best value per column.}
\label{tab:soft-variants}
\small
\setlength{\tabcolsep}{4pt}
\resizebox{\textwidth}{!}{%
\begin{tabular}{l rr rr rr}
\toprule
& \multicolumn{2}{c}{\envHC} & \multicolumn{2}{c}{\envAnt} & \multicolumn{2}{c}{\envGo} \\
\cmidrule(lr){2-3}\cmidrule(lr){4-5}\cmidrule(lr){6-7}
Method & Reward & Falls & Reward & Falls & Reward & Falls \\
\midrule
Unmasked PG & \revised{$\mathbf{4884\pm305}$} & \revised{$165\pm53$} & \revised{$1364\pm608$} & \revised{$4505\pm2184$} & \revised{$\mathbf{6004\pm459}$} & \revised{$2844\pm461$} \\
Masked, learned $V$ & $1982\pm1058$ & $306\pm360$ & $222\pm275$ & $50011\pm30828$ & $4374\pm758$ & $972\pm206$ \\
Analytic $V$ (no gate) & $3870\pm447$ & $73\pm19$ & $174\pm150$ & $35481\pm21851$ & $4473\pm628$ & $926\pm259$ \\
Soft $\Ctheta$ gate & $2645\pm829$ & $1215\pm1843$ & $663\pm323$ & $11745\pm18434$ & $4142\pm1471$ & $2044\pm2647$ \\
Soft $\Ctheta$, decomposed $V$ & $3569\pm1074$ & $721\pm985$ & $505\pm319$ & $8183\pm8064$ & $4374\pm1191$ & $1242\pm472$ \\
\textbf{\methodName{} (hard $\Ctheta$)} & \revised{$3855\pm958$} & \revised{$\mathbf{15.3\pm14.7}$} & \revised{$\mathbf{2698\pm237}$} & \revised{$\mathbf{866\pm718}$} & \revised{$5509\pm155$} & \revised{$\mathbf{830\pm216}$} \\
\bottomrule
\end{tabular}}

\end{table}

The mechanism behind the soft gate's failure is given in \Cref{sec:compat}: on a failed recovery segment the soft gate still emits a negative-signed imitation pressure, teaching $\piMain$ to \emph{un}learn the recovery action at those states, a noisy signal that the hard gate removes by zeroing the term on failures.

\subsection{\texorpdfstring{$\lambdaCompat$}{lambda\_compat} sensitivity (\methodName{})}
\label{app:lambda-sensitivity}

With the gate fixed, the remaining hyperparameter of the regularizer is its weight. We sweep $\lambdaCompat \in \{10^{-4},10^{-3},10^{-2},10^{-1},1.0\}$ on \envHC and \envAnt.

\begin{table}[ht]
\centering\small
\caption{Compatibility-regularizer weight sweep on \envHC and \envAnt. Reward and falls are training-time (last-$10\%$ mean reward, cumulative falls); mean $\pm$ std over \revised{all $10$ seeds}, dedicated sweep batch at the body operating point, run independently of the \Cref{tab:main-results} runs. \textbf{Bold} marks the best value per column.}
\label{tab:lambda-sweep}
\setlength{\tabcolsep}{3pt}
\begin{tabular}{rrrrr}
\toprule
$\lambdaCompat$ & \envHC reward & \envHC falls & \envAnt reward & \envAnt falls \\
\midrule
$10^{-4}$ & \revised{$2781\pm1517$} & \revised{$200\pm220$} & \revised{$1874\pm965$} & \revised{$2412\pm2947$} \\
$10^{-3}$ & \revised{$4770\pm458$} & \revised{$\mathbf{33\pm32}$} & \revised{$2936\pm228$} & \revised{$423\pm362$} \\
$10^{-2}$ & \revised{$\mathbf{5886\pm633}$} & \revised{$142\pm121$} & \revised{$\mathbf{3038\pm457}$} & \revised{$531\pm395$} \\
$10^{-1}$ & \revised{$5258\pm649$} & \revised{$112\pm87$} & \revised{$2136\pm944$} & \revised{$\mathbf{168\pm262}$} \\
$1.0$ & \revised{$5352\pm617$} & \revised{$236\pm258$} & \revised{$1010\pm867$} & \revised{$7738\pm18532$} \\
\bottomrule
\end{tabular}

\end{table}
\revised{The $\lambdaCompat = 10^{-3}$ row runs \Cref{tab:main-results}'s configuration as a separate sweep arm.} The same runs also back \Cref{tab:thm2-check} (\envHC/\envAnt).

\emph{The optimum is environment-specific.} On \envHC the reward peaks sharply at $\lambdaCompat = 10^{-2}$ (\revised{$5886$}). On \envAnt the reward is within seed noise from $\lambdaCompat = 10^{-3}$ to $10^{-2}$, with $10^{-3}$ the safer choice. Both environments are robust within a $\sim 10\times$ band around the optimum.

\emph{The main results use a single $\lambdaCompat = 10^{-3}$.} \methodName{} in \Cref{tab:main-results} and \Cref{fig:curves} fixes $\lambdaCompat = 10^{-3}$ across all three environments rather than a per-environment optimum. We make this choice on \emph{safety} grounds. In the sweep above, $\lambdaCompat = 10^{-3}$ attains the lowest \envHC training-fall count (\revised{$33$}, against \revised{$142$} at $10^{-2}$) and a near-lowest count on \envAnt. A larger $\lambdaCompat = 10^{-2}$ would raise \envHC reward substantially (\revised{$5886$ vs $4770$}) and is marginally better on \envAnt, but it does so at roughly \revised{$4.3\times$} the \envHC falls. We therefore keep the uniform $10^{-3}$ default so that the headline fall reductions are not traded away for extra falls, and we note $10^{-2}$ as the reward-optimal setting for \envHC for practitioners who can tolerate the higher fall count.

\subsection{Fixed-\texorpdfstring{$d$}{d} ablation: the safe-region schedule vs.\ a constant \texorpdfstring{$d$}{d}}
\label{app:fixed-d}

With the gate and its weight fixed, the last training-time design choice is the safe-region radius, which \methodName{} anneals on a curriculum. A natural counterfactual is to pick a smaller $\dmax$ and hold $d$ constant there. \Cref{tab:fixed-d} sweeps fixed $d \in \{0.25, 0.5, 0.75, 1.0\}\cdot\dmax$ on each of the three environments, with otherwise-identical \methodName{}.

\begin{table}[ht]
\centering\small
\caption{Fixed-$d$ sweep at the body operating point (\methodName{}), dedicated sweep batch, mean $\pm$ std over \revised{all $10$ seeds}. $d$ is held constant at the listed percentage of the per-env $\dmax$ ($2.0$ / $0.4$ / $0.15$) throughout training. Cells are \texttt{Reward / Falls}; \textbf{bold} marks the best reward and lowest falls per environment\revised{, ranging over the whole column including the curriculum reference row; at $d = \dmax$ the \envHC{} intervention rate is $0$ on all ten seeds, so its $100\%$ cell reproduces \ac{PPO} without recovery, while the \envAnt{} and \envGo{} $100\%$ cells run with the mechanism live. The sweep's own linear arm replicates the curriculum reference within seed noise}. Bottom row: linear-curriculum reference at the same $\dmax$, \revised{the 10-seed cells of \Cref{tab:main-results}}.}
\label{tab:fixed-d}
\small
\setlength{\tabcolsep}{4pt}
\resizebox{\textwidth}{!}{%
\begin{tabular}{l lll}
\toprule
$d/\dmax$ & \envHC & \envAnt & \envGo \\
\midrule
$25\%$ & \revised{$3053\pm1622 / 1400\pm1102$} & \revised{$990\pm2  / \mathbf{4\pm3}$} & \revised{$2852\pm132 / 2010\pm453$} \\
$50\%$ & \revised{$2272\pm524 / 3487\pm1747$} & \revised{$1322\pm616 / 8233\pm4481$} & \revised{$5409\pm92 / \mathbf{390\pm43}$} \\
$75\%$ & \revised{$3154\pm1095 / 5038\pm3744$} & \revised{$724\pm497  / 17854\pm5284$} & \revised{$5558\pm176 / 813\pm290$} \\
$100\%$ & \revised{$\mathbf{3904\pm814} / 3746\pm689$} & \revised{$469\pm247  / 25700\pm9405$} & \revised{$\mathbf{5817\pm127} / 1466\pm306$} \\
\midrule
\textbf{\makecell[l]{Linear curr.\ \\ (\methodName, ref.)}}
        & $\dmax{=}2.0$:\ \revised{$3855\pm958 / \mathbf{15.3\pm14.7}$}
        & $\dmax{=}0.4$:\ \revised{$\mathbf{2698\pm237} / 866\pm718$}
        & $\dmax{=}0.15$:\ \revised{$5509\pm155 / 830\pm216$} \\
\bottomrule
\end{tabular}}

\end{table}

\noindent \textbf{\envHC.} Every fixed-$d$ cell pays \revised{$\sim 1{,}400$ to $5{,}000$} training falls and tops out near \revised{$2{,}300$ to $3{,}900$} reward, while the linear-curriculum reference reaches \revised{$3{,}855$} reward at only \revised{$15.3$} falls \revised{(a roughly $92\times$} to \revised{$329\times$} falls advantage \revised{on the \Cref{tab:main-results} cell}). No constant $d$ matches the curriculum's reward-at-low-falls. The small-$d$ cells stay safer but low-reward and the large-$d$ cells climb in falls without gaining reward.

\noindent \textbf{\envAnt.} The smallest fixed-$d$ (p25, $d=0.1$) is very safe ($4$ falls) but useless ($990$ reward). Larger constant $d$ accumulates $8{,}200$ to $25{,}700$ falls without reaching the curriculum's reward. The linear curriculum gets \revised{$2698$} reward at \revised{$866$} falls \revised{(the \Cref{tab:main-results} cell)}, which no constant $d$ matches.

\noindent \textbf{\envGo.} On \envGo the comparison is closer, and the curriculum does \emph{not} strictly dominate. A tuned constant $d$ is competitive. Fixed-$d$-p50 (\revised{$5409 / 390$}) attains essentially the curriculum's reward (\revised{$5509$ on the \Cref{tab:main-results} cell}) at \emph{fewer} falls (\revised{$390$ against the $830$ of \Cref{tab:main-results}; both the fixed-$d$ \envGo{} cells and the reference are 10 seeds}), and \revised{p75 matches the curriculum on both axes within seed noise ($5558 \pm 176$ / $813 \pm 290$ against $5509 \pm 155$ / $830 \pm 216$), while p100 reaches a higher reward, $5817 \pm 127$ against $5509 \pm 155$, and pays $1.8\times$ its falls ($1466$ against $830$) to get there}. The curriculum's value on \envGo is therefore robustness to the choice of $d$, it reaches high reward at low falls without a per-environment $d$ search, rather than a strict win over the best constant $d$.

\noindent \textbf{Summary.} On \envHC and \envAnt no constant $d$ matches the schedule's high reward at comparable safety. On \envGo a tuned fixed-$d$ (p50) is competitive on both axes, so there the schedule's contribution is sparing the practitioner a per-environment $d$ search rather than a strict dominance. The safe-region schedule is a supporting mechanism for the objective-gap bound (\Cref{thm:convergence}). The method's correctness rests on the unbiased masked gradient (\Cref{thm:gradient}), not on the schedule.

\subsection{Curriculum-schedule sensitivity (\methodName{} on \envAnt)}
\label{app:curr-shape}

\begin{table}[htbp]
\centering\small
\caption{Schedule shape on \envAnt \methodName{}. Reward / Falls, mean $\pm$ std over \revised{all $10$ seeds}; \revised{sensitivity} batch at the wider radius ($\dmax = 0.56$) rather than the body $0.4$, so within-table comparisons are matched but absolute values differ from \Cref{tab:main-results}. \textbf{Bold} marks the best value per column; Linear and Log rewards tie within seed noise and are both bolded.}
\label{tab:curr-shape}
\begin{tabular}{lrr}
\toprule
Schedule & Reward & Falls \\
\midrule
\textbf{Linear} & \revised{$2938\pm268$} & \revised{$\mathbf{1058\pm934}$} \\
Log & \revised{$\mathbf{3055\pm346}$} & \revised{$1120\pm752$} \\
Step (4 jumps) & \revised{$1601\pm464$} & \revised{$5796\pm2282$} \\
Constant ($d=\dmax$) & \revised{$305\pm85$} & \revised{$30305\pm6863$} \\
\bottomrule
\end{tabular}
\end{table}
Granting that some schedule beats a constant $d$, the next question is which one. \Cref{tab:curr-shape} reports three alternatives against the linear default. Linear and log are tied within seed noise on every metric. The step schedule is worse on both, with lower reward and \revised{$5.5\times$} more falls. Constant-$d$ (no curriculum) collapses, accumulating \revised{$29\times$} the falls of the linear curriculum at a tenth of its reward. This is consistent with the fixed-$d$ ablation above (\Cref{app:fixed-d}), whose constant $d{=}\dmax$ cell shows the same qualitative collapse (at the body radius rather than this table's wider one), and with the prediction of \Cref{thm:convergence} that the gap closes as $\beta$ shrinks, which the curriculum drives.

\subsection{Compute-scaling: \texorpdfstring{$2\times$}{2x} matched-budget on \envHC and \envAnt}
\label{app:16m-scaling}

\begin{table}[t]
\centering\small
\setlength{\tabcolsep}{3pt}
\caption{Compute-scaling: \methodName{} vs \ac{PPO} at the main budget ($8$M) and $2\times$ ($16$M) on \envHC and \envAnt. Reward / Falls, mean $\pm$ std over \revised{all $10$ seeds}. All cells are the \revised{sensitivity} batch at the wider $\dmax$ ($3.2$ on \envHC, $0.56$ on \envAnt) rather than the body protocol; the $16$M runs stretch the $d$-curriculum over the doubled budget. \Cref{tab:main-results}'s body-radius values (\envHC{} \methodName{} \revised{$3855\pm958$ / $15.3\pm14.7$}) therefore differ from the Main column here; see the preamble of this appendix. \textbf{Bold} marks the better method per environment and budget.}
\label{tab:16m-scaling}
\resizebox{0.62\textwidth}{!}{%
\begin{tabular}{ll ll}
\toprule
& Method & Main ($8$M) & $2\times$ ($16$M) \\
\midrule
\multirow{2}{*}{\envHC}  & \ac{PPO}  & \revised{$3277\pm1088 / 5505\pm2457$}  & \revised{$4129\pm1463 / 10206\pm7734$} \\
                         & \textbf{\methodName{}}  & \revised{$\mathbf{4609\pm1062 / 164\pm211}$} & \revised{$\mathbf{5805\pm732 / 210\pm321}$} \\
\midrule
\multirow{2}{*}{\envAnt} & \ac{PPO}  & \revised{$289\pm51 / 33984\pm8130$}  & \revised{$223\pm113 / 84739\pm31184$} \\
                         & \textbf{\methodName{}}  & \revised{$\mathbf{2938\pm268 / 1058\pm934}$} & \revised{$\mathbf{2822\pm628 / 4227\pm2958}$} \\
\bottomrule
\end{tabular}}

\end{table}
With the design choices settled, we turn to whether the chosen configuration holds up under conditions the body did not exercise, starting with more compute. We re-run \ac{PPO} and \methodName{} at $2\times$ the main budget ($8\mathrm{M}\to16\mathrm{M}$) on \envHC and \envAnt (\Cref{tab:16m-scaling}; wider $\dmax$ than the body protocol, see the caption).
Re-tuning $\dmax$ for longer horizons or a multi-stage curriculum is left to follow-up work.

\subsection{Compatibility-regularizer normalization (global vs.\ per-segment)}
\label{app:compat-norm}

\begin{wraptable}{r}{0.60\textwidth}
\centering\small
\setlength{\tabcolsep}{2pt}
\vspace{-1.2em}
\caption{Compat-norm: global vs.\ per-segment. Reward / Falls, mean $\pm$ std over \revised{all $10$ seeds}; \revised{sensitivity} batch at the wider $\dmax$ ($3.2$ \envHC, $0.56$ \envAnt). \textbf{Bold} marks the arm better on both metrics per environment.}
\label{tab:compat-norm}
\begin{tabular}{lll}
\toprule
Env & global (default) & per-segment \\
\midrule
\envHC & \revised{$4609\pm1062 / 164\pm211$} & \revised{$\mathbf{4968\pm519 / 90\pm111}$} \\
\envAnt & \revised{$\mathbf{2938\pm268 / 1058\pm934}$} & \revised{$2525\pm439 / 2038\pm1395$} \\
\bottomrule
\end{tabular}
\vspace{-0.8em}
\end{wraptable}
The $\Ctheta$ loss in \Cref{eq:compat-loss} averages across all recovery-active timesteps in the current minibatch (``global''). An alternative ``per-segment'' normalization averages within each contiguous safe$\to$unsafe$\to$safe segment, then over segments, which prevents long unsafe excursions from dominating the loss. \Cref{tab:compat-norm} reports both on \envHC and \envAnt. Per-segment helps on \envHC (higher reward, fewer falls) but hurts on \envAnt (lower reward, more falls). With no consistent winner across environments and the means within roughly one standard deviation, we keep \texttt{global} as the simpler default.

\subsection{Action-noise robustness (\envHC \methodName{})}
\label{app:action-noise}

\begin{table}[htbp]
\centering\small
\caption{Action-noise on \envHC \methodName{} at $\lambdaCompat=10^{-3}$. Reward / Falls, mean $\pm$ std over \revised{all $10$ seeds}; \revised{sensitivity} batch at the wider \envHC{} radius ($\dmax = 3.2$). The $\sigma_a = 0$ arm is the same runs as the global arm of \Cref{tab:compat-norm} and the Main \methodName{} cell of \Cref{tab:16m-scaling}.}
\label{tab:action-noise}
\begin{tabular}{rrr}
\toprule
$\sigma_a$ & Reward & Falls \\
\midrule
$0.00$ & \revised{$4609\pm1062$} & \revised{$164\pm211$} \\
$0.05$ & \revised{$4212\pm587$} & \revised{$182\pm267$} \\
$0.10$ & \revised{$2813\pm981$} & \revised{$388\pm329$} \\
\bottomrule
\end{tabular}
\end{table}
A second stress axis perturbs the actions themselves. We re-train \methodName{} on \envHC with Gaussian action noise of standard deviation $\sigma_a \in \{0.0, 0.05, 0.10\}$ added to every action (recovery and main) at the $\lambdaCompat = 10^{-3}$ default (\Cref{tab:action-noise}). Reward and falls stay within seed noise up to $\sigma_a = 0.05$, then degrade at $\sigma_a = 0.10$.
An action-noise sweep at the HC-tuned $\lambdaCompat = 10^{-2}$ is left to follow-up work. We also ran $\sigma_a = 0.20$ as a stress test. \methodName{} collapses ($1863\pm542$ reward, $1659\pm1807$ falls) when the injected noise is comparable in scale to the recovery's action magnitude, and we omit it from the table as a known out-of-regime failure mode.

\subsection{Failure-mode catalog}
\label{app:failure-modes}

The stress tests above probe how the chosen configuration degrades; this catalog records the qualitative failure each rejected variant exhibits, so the quantitative gaps in the ablations have a concrete behavioral reading.

\begin{itemize}[leftmargin=2em]
\item \textbf{\ac{PPO}} on \envGo: agent falls within first 200 steps; recovery is unavailable; episode terminates.
\item \textbf{Analytic-$V$ variant} on \envAnt: agent learns to balance briefly but never escapes the recovery policy's basin; $\alpha$ stays near 1 for the full training run.
\item \textbf{Soft-gate variant} on \envHC: regularizer fits failed-recovery actions; reward oscillates; final policy is unreliable.
\item \textbf{\methodName{}} on \envAnt: clean separation, $\alpha$ drops to near zero around update 1000, reward climbs steadily thereafter, falls plateau.
\end{itemize}

\revised{\subsection{Recovery-disabled evaluation for all methods}
\label{app:recovery-disabled}

\Cref{tab:recovery-disabled-all} reports recovery-disabled evaluation for the full method set. The sharpest contrast is on \envAnt, where recovery is unreliable; \methodName{} is the only method whose recovery-disabled policy works. On \envHC{} most methods transfer, and \methodName{}'s recovery-disabled falls are within noise of the best variant at near-top reward. On \envGo{} at the near-body $\dmax$ ($0.1$, just under the body $0.15$) \methodName{}'s recovery-disabled policy leads all eight methods on both axes; at the wider $\dmax$ it again leads on reward but four masked variants record fewer evaluation falls. An independent 10-seed batch at the same protocol (the crossed-cell batch of \Cref{app:crossed-cells}) replicates every shared cell within one standard deviation. On \envWalker{} the winning configuration is \emph{not} self-sufficient; recovery-disabled evaluation reaches under half the mixed-controller reward ($10$ seeds; \Cref{app:walker2d}).

\begin{table}[ht]\revisedfloat
\centering\small
\caption{\revised{Recovery-disabled final evaluation (deterministic main policy, recovery branch disabled): final eval reward and cumulative eval falls, 10 seeds unless noted. \envHC{}/\envAnt{}: the wider sensitivity $\dmax$ ($3.2$ / $0.56$), 8M steps. \envGo{}: 16M steps at two logged $\dmax$ values ($0.1$ and $2.2$). \revised{The columns differ only in the training batch the policies come from. Coverage is every method except the Unmasked-PG variant on \envGo{}, never run there; two \envGo{} cells are short of ten seeds.}}}
\label{tab:recovery-disabled-all}
\setlength{\tabcolsep}{4pt}
\resizebox{\ifdim\width>\textwidth\textwidth\else\width\fi}{!}{%
\begin{tabular}{l rr rr}
\toprule
& \multicolumn{2}{c}{\envHC{}} & \multicolumn{2}{c}{\envAnt{}} \\
\cmidrule(lr){2-3}\cmidrule(lr){4-5}
Method & Eval reward & Eval falls & Eval reward & Eval falls \\
\midrule
\ac{PPO} (no recovery) & $3525\pm989$ & $232\pm173$ & $299\pm95$ & $939\pm23$ \\
Recovery RL (action-relabeling port) & $2341\pm1775$ & $216\pm318$ & $83\pm69$ & $972\pm5$ \\
Safe Legged (reward-shaping port) & $4756\pm1406$ & $196\pm120$ & $70\pm36$ & $974\pm1$ \\
Unmasked PG & $4863\pm296$ & $44\pm14$ & $1215\pm406$ & $460\pm189$ \\
Masked, learned $V$ & $3266\pm1099$ & $63\pm117$ & $404\pm422$ & $670\pm279$ \\
Masked, analytic $V$ & $3421\pm1482$ & $65\pm121$ & $433\pm498$ & $688\pm324$ \\
Soft $\Ctheta$ gate & $2845\pm1460$ & $251\pm296$ & $877\pm664$ & $437\pm291$ \\
Soft $\Ctheta$, decomposed $V$ & $4033\pm815$ & $162\pm192$ & $683\pm780$ & $469\pm220$ \\
\methodName{} & $4585\pm1320$ & $52\pm101$ & $3124\pm358$ & $90\pm72$ \\
\midrule
& \multicolumn{2}{c}{\envGo{} near-body batch} & \multicolumn{2}{c}{\envGo{} sensitivity batch} \\
\cmidrule(lr){2-3}\cmidrule(lr){4-5}
\ac{PPO} (no recovery) & $3555\pm1187$ & $1113\pm83$ & $4063\pm934$ & $1054\pm103$ \\
Recovery RL (action-relabeling port) & $1926\pm1076$ & $1593\pm338$ & $3788\pm678$ & $1074\pm85$ \\
Safe Legged (reward-shaping port) & $126\pm41$ & $1951\pm1$ & $2928\pm1175$ & $1365\pm193$ \\
Masked, learned $V$ & $1275\pm1299$ & $1779\pm327$ & $4198\pm903$ & $845\pm167$ \\
Masked, analytic $V$ & $3196\pm1384$ & $1310\pm456$ & $3697\pm1076$ & $958\pm110$ \\
Soft $\Ctheta$ gate & $3027\pm1136$ & $1302\pm453$ & $4647\pm775$ & $984\pm79$ \\
Soft $\Ctheta$, decomposed $V$ & $1901\pm1778$ & $1555\pm469$ & $3959\pm1010$ & $911\pm214$ \\
\methodName{} & $4821\pm755$ & $912\pm69$ & $5173\pm587$ & $1005\pm71$ \\
\bottomrule
\end{tabular}}

\end{table}}

\subsection{Empirical check of \texorpdfstring{\Cref{thm:convergence}}{the objective-gap bound}}
\label{app:thm2-check}

We now turn from the method's empirical robustness to its theory, checking each of the paper's two theorems against the runs, starting with the objective-gap bound. \Cref{thm:convergence} bounds the deployment gap $|J^{\mathrm{mix}}(\theta_k) - J(\theta_k)|$, between the during-training mixed-policy return and the return of $\piMain$ deployed alone, by the main-policy out-of-region rate $\betaRate$. The bound is a proportionality, not a vanishing guarantee. It predicts a small gap exactly when $\betaRate$ is small, and \Cref{cor:fixed-point} sharpens this to an exact fixed point only in the idealized $\betaRate \to 0$ limit. We test that prediction directly. Every 4 \ac{PPO} updates we run a recovery-disabled evaluation pass, averaging $5$ episodes per pass, that samples from the same stochastic policy $\piMain(\cdot|s) = \mathcal{N}(\mu_\theta(s), \sigma_\theta^2)$ used in training, the operative $J(\theta_k)$, and compare it to the mixed-policy rollout return $J^{\mathrm{mix}}(\theta_k)$. Because $\betaRate$ is an out-of-region rate under $\piMain$ alone and is not directly observable, we report alongside it the observable recovery rate $\alpha$, which \Cref{sec:convergence} identifies as a diagnostic for $\betaRate$ rather than a certified upper bound.
A deterministic pass using $\mu_\theta(s)$ alone is also logged as a noise-reduced deployment diagnostic but is not the quantity the theorem concerns.

\begin{table}[ht]\centering\small
\setlength{\tabcolsep}{4pt}
\caption{End-of-training values for \methodName{} at the body operating point, \revised{all $10$ seeds}, over the last $20\%$ of training; \revised{The batch is shared with the $\lambdaCompat = 10^{-3}$ arm of \Cref{tab:lambda-sweep} on \envHC{} and \envAnt{} and the body-$16$M family on \envGo{}.} Columns: the mixed-policy return $J^{\mathrm{mix}}$, the $\piMain$-only stochastic eval return $J$, their signed difference, the recovery rate $\alpha$ (observable diagnostic for $\betaRate$), cumulative training-time falls, and per-episode eval-time falls under $\piMain$ alone (recovery disabled).}
\label{tab:thm2-check}
\begin{tabular}{l rrr r rr}
\toprule
Env & $J^{\mathrm{mix}}$ & $J$ & $J^{\mathrm{mix}} - J$ & $\alpha$ & Falls (train) & Falls (eval/ep) \\
\midrule
\envHC  & \revised{$4671 \pm 461$} & \revised{$4670 \pm 474$} & \revised{$0 \pm 16$}    & $0.000$ & \revised{$33 \pm 32$}   & \revised{$0.00$} \\
\envAnt & \revised{$2906 \pm 229$} & \revised{$2912 \pm 230$} & \revised{$-6 \pm 30$}  & $0.000$ & \revised{$423 \pm 362$} & \revised{$0.12$} \\
\envGo  & \revised{$5384 \pm 453$} & \revised{$4800 \pm 620$} & \revised{$584 \pm 222$} & \revised{$0.048 \pm 0.009$} & \revised{$911 \pm 267$} & \revised{$0.39$} \\
\bottomrule
\end{tabular}

\end{table}

The result follows the bound's structure precisely, and the key observation is that \emph{the gap tracks $\betaRate$}. On \envHC and \envAnt the recovery rate falls to essentially zero ($\alpha = 0.000$ on both at end-of-training), and the gap closes to within seed noise of zero (\revised{$0 \pm 16$ and $-6 \pm 30$}). These are the self-stable environments.
Here $\piMain$ holds the safe region unaided, so $\piMixed \equiv \piMain$ along the trajectory, $\betaRate \to 0$, and the fixed point of \Cref{cor:fixed-point} is reached up to the invariance-slack residual $\eta_\star$ the corollary anticipates (the residuals here are statistically indistinguishable from zero; empirically it is the stochastic policy's entropy floor that keeps $\eta_\star$ from being exactly zero). On \envGo the policy does not become fully self-stable, because the quadruped's gait cannot stay inside the tight $\safeR(\dmax)$ tube. The recovery rate therefore settles at a small positive floor (\revised{$\alpha = 0.048 \pm 0.009$}) rather than at zero, and the gap settles at a correspondingly positive residual (\revised{$584 \pm 222$, about $11\%$} of $J^{\mathrm{mix}}$) rather than vanishing.
After the early-curriculum transient the gap collapses from its peak on \envHC, stays within seed noise of zero on \envAnt in the final average despite a late upward drift in the plotted per-seed magnitude, and declines but stays positive on \envGo (\Cref{fig:thm2-traj}). On \envGo it falls from \revised{$\sim 1750$} when the safe region is small to a small positive residual at $16$M env-steps. The end-of-training gap and recovery rate in \Cref{tab:thm2-check} make the proportionality explicit. The gap is near zero exactly where $\alpha$ is, and the only environment with a residual gap is the only one with a residual $\alpha$.

\begin{figure}[t]
\centering
\begin{minipage}{0.62\linewidth}
\revisedgraphics[width=\linewidth]{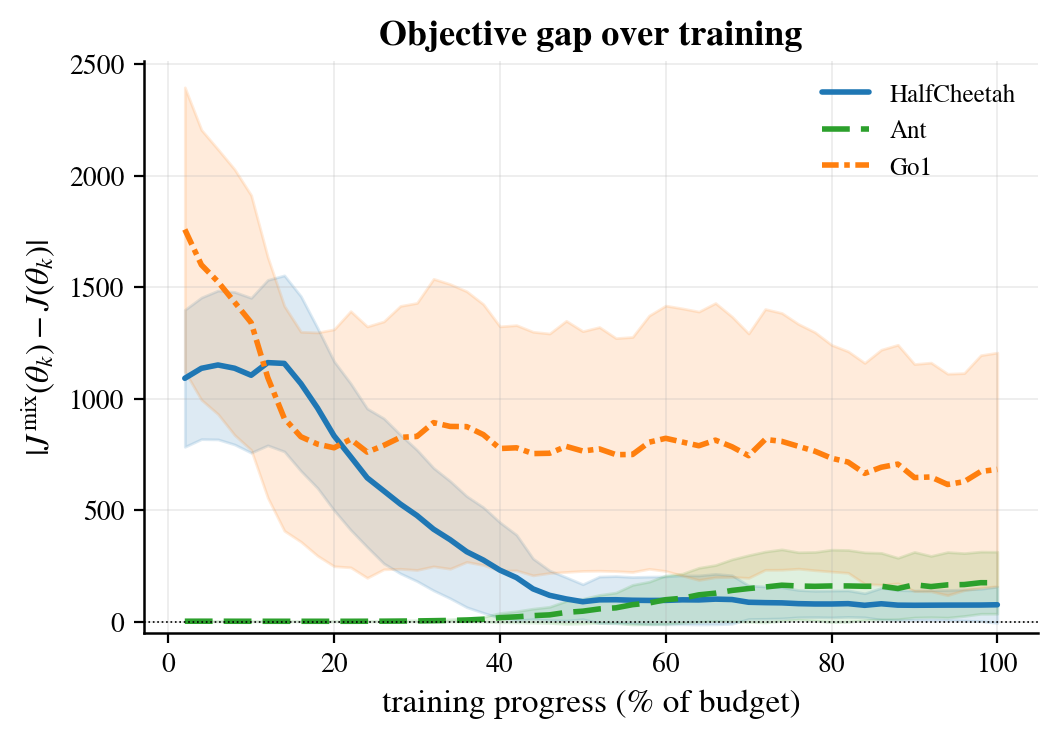}
\end{minipage}
\caption{Objective gap over training for the \methodName{} runs of \Cref{tab:thm2-check}: the curve is the seed-mean of the per-seed absolute gap $|J^{\mathrm{mix}}(\theta_k) - J(\theta_k)|$, smoothed with a 9-point moving average for display; shaded = $\pm 1$ std. Because the curve averages magnitudes while \Cref{tab:thm2-check} reports the signed across-seed mean over the final $20\%$ of training, the curve's end value can sit above the table's near-zero signed gap (as on \envAnt). \revised{The curves are drawn from the ten seeds of \Cref{tab:thm2-check}.}}
\label{fig:thm2-traj}
\end{figure}

This two-regime behavior is a stronger test of \Cref{thm:convergence} than a uniform collapse to zero would be. The bound does not assert that the gap vanishes; it asserts that the gap is governed by $\betaRate$. \envHC and \envAnt realize the vanishing regime ($\betaRate \to 0$), and \envGo realizes the residual regime ($\betaRate$ bounded away from zero because \Cref{asm:invariance} holds only approximately for a quadruped confined to a tight tube), with the observable $\alpha$ moving in lockstep with the gap in both. The sign of the gap is itself interpretable. On \envGo, $J^{\mathrm{mix}} - J > 0$ means that deploying $\piMain$ alone forfeits the return that the recovery policy would otherwise secure, so the gap measures the policy's residual reliance on the recovery policy, which \envHC and \envAnt drive to zero and \envGo reduces to about \revised{$11\%$}.

The absolute prefactor of \Cref{thm:convergence} ($2\rmax/(1-\gamma)^2 \approx 2 \times 10^5$ for $\rmax = 10$, $\gamma = 0.99$) is loose, as is standard for performance-difference-lemma bounds, so we do not read the inequality quantitatively. Its operative content is the structural claim that any mechanism reducing $\betaRate$, here the radius curriculum together with the compatibility regularizer, tightens the gap. \Cref{fig:thm2-traj} confirms that claim across the full range of $\betaRate$ our three environments realize.

\revised{\textbf{Measuring $\betaRate$ and the bound's numerical magnitude.} $\betaRate$ is reported two ways, collected below. The first is a pair of certified lower bounds per environment (except on \envHC{}, explained below) from the existing recovery-disabled evaluation logs, a fall route (every evaluation-time fall is an out-of-region step) and a gap route (the measured gap together with the bound itself). The second is a direct measurement $\hat\beta$ on the saved final checkpoints, the fraction of steps outside $\safeR(d_{\mathrm{end}})$ over $20$ recovery-disabled episodes per seed, at $10$ seeds; six of the ten \envAnt{} seeds also measure exactly $0$. Gaps are the absolute tail gaps over the 10-seed batch; \Cref{tab:thm2-check} prints the signed end-of-training values, and we label both rather than let them read as one quantity.

\begin{center}\small
\setlength{\tabcolsep}{4pt}
\begin{tabular}{l lll ll l}
\toprule
 & \multicolumn{3}{c}{$\betaRate$} & \multicolumn{2}{c}{Bound right-hand side} & \\
\cmidrule(lr){2-4}\cmidrule(lr){5-6}
Env & fall route & gap route & direct $\hat\beta$ & at fall-route $\betaRate$ & at $\hat\beta$ & Measured gap \\
\midrule
\envHC & $\geq 1.7 \times 10^{-5}$ & $\geq 5.0 \times 10^{-6}$ & $< 6 \times 10^{-6}$ & $3.4$ & & $1.0$ \\
\envAnt & $\geq 1.4 \times 10^{-4}$ & $\geq 4.6 \times 10^{-5}$ & $(2.4 \pm 3.6) \times 10^{-4}$ & $28$ & $48$ & $9.3$ \\
\envGo & $\geq 3.9 \times 10^{-4}$ & $\geq 2.9 \times 10^{-3}$ & $0.032 \pm 0.034$ & $78$ & $6400$ & $585$ \\
\bottomrule
\end{tabular}
\end{center}

The route-based right-hand sides are lower bounds on the true right-hand side, so they can legitimately sit below the measured gap. At the directly measured $\hat\beta$ the \envAnt{} inequality holds with a factor-$5$ margin rather than vacuously, while the \envGo{} right-hand side sits within a factor of two of the trivial return bound $\rmax H = 10^{4}$; \envGo{} is beyond the $\betaRate < 1 - \gamma$ regime where the bound is informative, which is why we read the bound structurally. On \envHC{} every non-terminal state lies inside $\safeR(d_{\mathrm{end}})$, so $\hat\beta$ is zero by construction. A \envHC{} fall can also end inside the region, and the measured \envHC{} gap is then seed noise, so neither \envHC{} route-based cell is a lower bound. The bound therefore cannot be tested there, and we draw no conclusion on \envHC{}.}

 Eval-time falls under $\piMain$ alone (recovery disabled) corroborate the reading. They are \revised{$0.00$, $0.12$, and $0.39$} per episode on \envHC, \envAnt, and \envGo, ordered exactly as the gaps and the recovery rates are. \Cref{fig:alpha-curves} plots the full recovery-rate trajectories that this diagnostic summarizes.

\begin{figure}[ht]
\centering
\revisedgraphics[width=\linewidth]{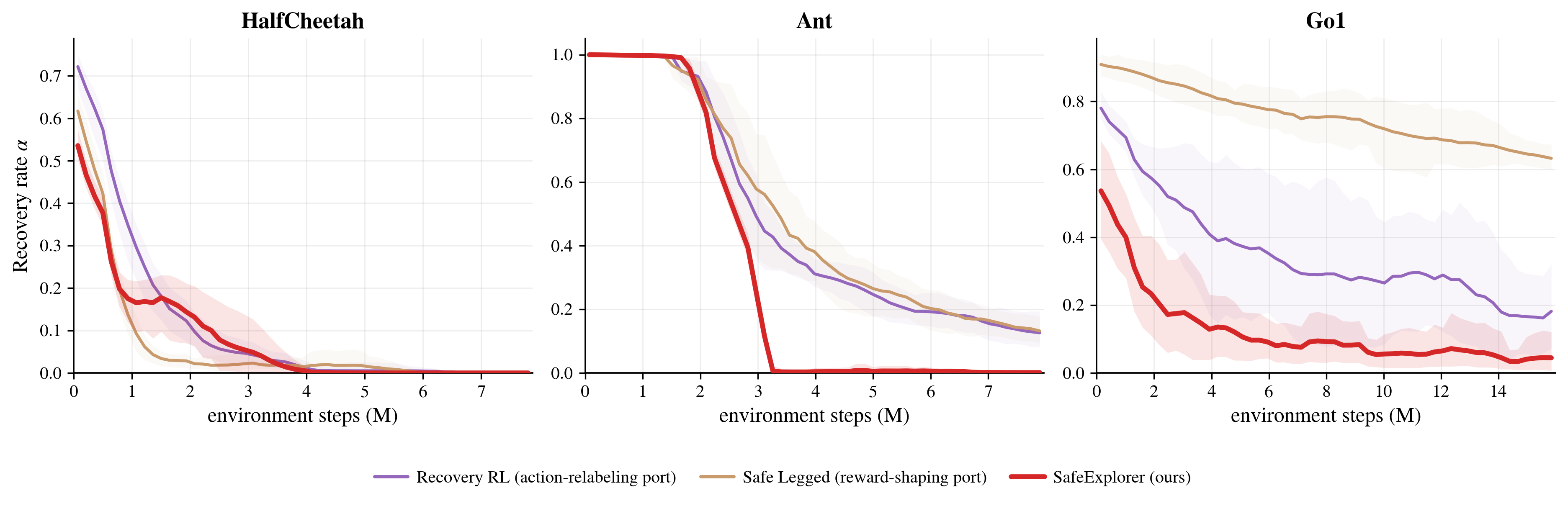}
\caption{Recovery rate $\alpha$ (fraction of rollout steps on which recovery fired) over training, per environment, for the recovery-using methods (the \ac{CMDP} solvers act with a single policy and never invoke recovery, so they have no $\alpha$). Solid = mean across seeds, shaded = [min, max] envelope.}
\label{fig:alpha-curves}
\end{figure}

\subsection{Masked gradient versus \texorpdfstring{\ac{IS}}{importance sampling} under a stochastic recovery}
\label{app:is-variance}

With the objective-gap bound checked, we turn to the gradient estimator: \Cref{thm:gradient} makes the masked estimator exactly unbiased at recovery steps, while per-step \ac{IS}, where it applies at all, is biased in general (\Cref{rem:is-variance}); what remains empirical is how the two compare in practice. \Cref{thm:gradient} holds for any $\theta$-independent recovery; for a deterministic one, not even \ac{IS} applies (\Cref{sec:gradient}). Here we ask what happens when the recovery \emph{is}
stochastic, so \ac{IS} becomes an option. Writing
$w_t = \piMain(a_t \mid s_t) / \piRec(a_t \mid s_t)$ for the importance weight at a recovery step,
the \ac{IS} estimator of $\nabla_\theta J^{\mathrm{mix}}(\theta)$ (\Cref{cor:safe-region-gradient}) decomposes as
\begin{equation}
  \hat g_{\mathrm{IS}}
  = \underbrace{\sum_{t :\, s_t \in \safeR} \nabla_\theta \log \piMain(a_t\mid s_t)\, A_t}_{\hat g_{\mathrm{masked}}}
  \;+\; \sum_{t :\, s_t \notin \safeR} w_t\, \nabla_\theta \log \piMain(a_t\mid s_t)\, A_t,
  \label{eq:is-decomp}
\end{equation}
with $A_t$ the \ac{PPO} advantage estimate. The second sum would be mean zero by the score-function
identity only if $A_t$ were replaced by an action-independent baseline $b(s_t)$; with the actual
advantage, which depends on $a_t$ through the reward and the successor state, its conditional mean
is the off-policy policy-gradient contribution at the unsafe states, which is generically nonzero.
The per-step \ac{IS} estimator is therefore biased for $\nabla_\theta J^{\mathrm{mix}}$, and
masking removes a bias-carrying term, not merely a noisy one. A recovery policy differs sharply
from the main policy by design, so the term's weights are heavy-tailed; truncating them
\citep{espeholt2018impala, munos2016retrace} trades variance for further bias. The measurement
below therefore compares two estimators that do not share a mean: the recorded explosion of the
\ac{IS} weights is evidence about the conditioning of the added term, not a like-for-like variance
comparison between unbiased estimators.

Concretely, we rerun \methodName{} on all three environments (settings of
\Cref{sec:exp-setup}) with the recovery action \emph{sampled} from the \ac{SAC} recovery rather than
taken greedily at its mean, and shadow-compute, per rollout, the weights' effective sample size and
each estimator's gradient variance. Training itself always uses the masked gradient.

\Cref{fig:is-variance} shows that the three regimes expose complementary failure modes. On
\envHC the strong recovery makes the weights degenerate from the first update. The effective sample
size never exceeds $0.5\%$ of the recovery steps, and the \ac{IS} gradient variance exceeds the
masked variance by a factor of up to $8\times10^{12}$. On \envAnt the weak recovery initially
overlaps the main policy and \ac{IS} stays formally usable for longer, but the variance ratio holds
near $10^{13}$ for the first $1.2$M steps. \envGo completes the spectrum. Over its 12-dimensional
action space the recovery's actions are so unlikely under the main policy that the weights die at
once (effective sample size always below $3\%$, reaching $7\times10^{-5}$), so the added term
vanishes before its variance can build, and the ratio never rises above $\sim\!40$ and sits at one
for the full $16$M steps, over which the quadruped's recovery rate never reaches zero
(\Cref{app:thm2-check}).
Once the main policy separates from the recovery the weights vanish and the ratio returns to one.
\ac{IS} proves either catastrophically noisy or identical to the masked estimator, never better in any
regime we measure.

\begin{figure}[ht]
  \centering
  \includegraphics[width=\linewidth]{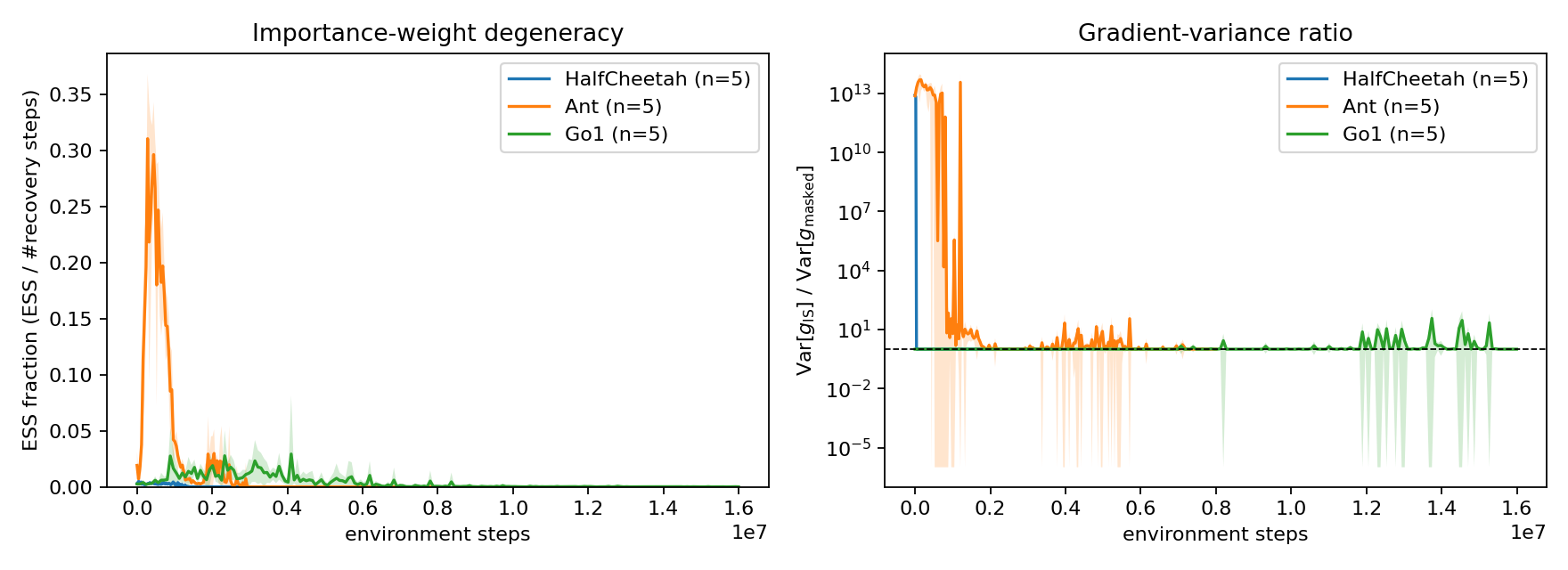}
  \caption{Masked gradient versus \ac{IS} on a stochastic recovery, \methodName{} (mean
  $\pm$ std). \emph{Left}: the effective sample size of the importance weights collapses toward zero
  when the recovery is active, the signature of weight degeneracy. \emph{Right}: the ratio of \ac{IS}
  to masked gradient variance (log scale; dashed line at one). The masked estimator targets
  $\nabla_\theta J^{\mathrm{mix}}$ (\Cref{eq:is-decomp}) without the \ac{IS} term's bias and
  variance pathologies. The \envHC and \envAnt curves end where the
  recovery stops engaging, because the diagnostics exist only on updates with recovery steps; on
  \envGo recovery stays mildly engaged throughout, so the curve spans the full run.}
  \label{fig:is-variance}
\end{figure}

Training with the truncated-\ac{IS} gradient ($w_t$ clipped at $c = 10$) confirms the cost
(\Cref{fig:is-train}). These runs use the sampled-recovery setup of this appendix, so their absolute rewards are not comparable to \Cref{tab:main-results}'s greedy-recovery numbers; only the within-setup masked-vs-\ac{IS} differences are meaningful. On \envHC the weights die immediately and the two variants are statistically
indistinguishable (reward $4675 \pm 642$ versus $4421 \pm 647$). On \envAnt the million-step
degeneracy window is expensive: $793 \pm 214$ reward against $1325 \pm 289$ for the masked gradient,
a $40\%$ drop, and its lower fall count ($2066 \pm 1582$ versus $4682 \pm 1047$) reflects the weaker
policy rather than safer behavior. On \envGo the truncated-\ac{IS} variant again trails the masked
gradient ($4884 \pm 241$ versus $5243 \pm 219$ reward, a $7\%$ drop) at statistically
indistinguishable falls ($3293 \pm 589$ versus $3401 \pm 351$).

\begin{figure}[ht]
  \centering
  \includegraphics[width=\linewidth]{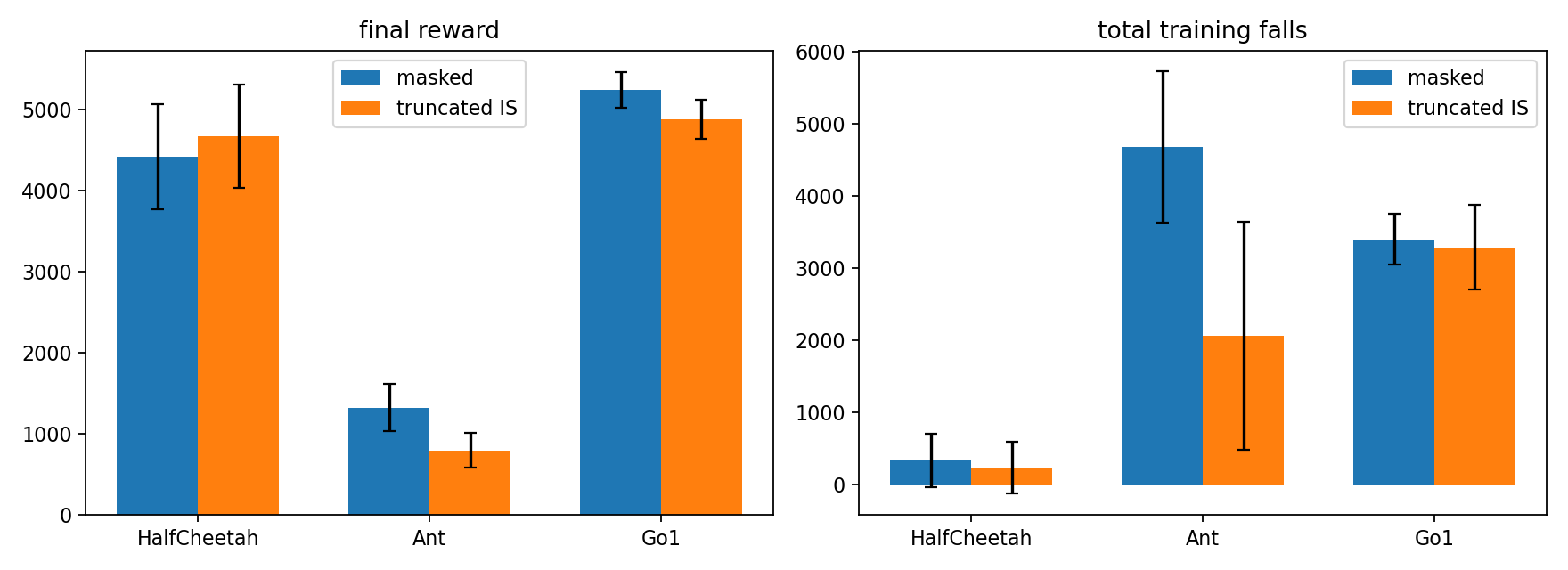}
  \caption{Training \methodName{} with the masked gradient versus the truncated-\ac{IS} gradient (mean $\pm$ std). \emph{Left}: final reward. \emph{Right}: total training falls.}
  \label{fig:is-train}
\end{figure}

\subsection{\texorpdfstring{\ac{CMDP}}{CMDP} baselines: constraint feasibility}
\label{app:cmdp-feasibility}

Having checked the method's own theory, we finally audit the baselines it is compared against, asking whether the \ac{CMDP} solvers even satisfy the constraint they are given. This appendix supports the \ac{CMDP} comparison of \Cref{sec:exp-cmdp}. For the comparison to be fair, both solvers use the OmniSafe implementation \citep{ji2023omnisafe} and run on the identical environments, reward functions, network architecture, optimizer, and step budget as the variants in \Cref{tab:main-results}. We use the direct cost encoding, a cost of $1$ on a fall (a termination), the actual safety signal the constraint should bound. The only incidental differences from \methodName{} are observation and reward normalization, applied to the \ac{CMDP} baselines and not to \methodName{}, so they favor the baselines. Reported rewards for both solvers are the raw episode return (\texttt{EpRet} in OmniSafe), un-normalized and therefore on the same task-reward scale as the other methods. The normalization affects only the solvers' internal optimization, not the reported numbers.

On this matched footing, \Cref{tab:cmdp-feasibility} reports the final episodic cost each solver converges to, against the cost budget of $0.05$ it was given (a per-episode fall rate), averaged over the same seeds. The pattern is clean. Both solvers reach feasibility on \envHC, where staying safe is easy, and both \emph{fail} to reach feasibility on \envAnt and \envGo, where the converged cost exceeds the budget by up to $19\times$. The poor safety of the \ac{CMDP} baselines is therefore not a tuning artifact. The constrained optimization does not find a feasible policy at all in the regimes where the safe-region intervention is needed.

\begin{table}[ht]
\centering
\caption{Converged episodic cost vs.\ the cost budget ($\le 0.05$) for the \ac{CMDP} baselines (final $5\%$ of training). \cmark{} feasible, \xmark{} budget violated.}
\label{tab:cmdp-feasibility}
\small
\begin{tabular}{lccc}
\toprule
Method & \envHC & \envAnt & \envGo \\
\midrule
\ac{CPO}            & $0.007\pm0.011$ \cmark & $0.103\pm0.029$ \xmark & $0.956\pm0.025$ \xmark \\
\ac{PPO}-Lagrangian & $0.001\pm0.001$ \cmark & $0.326\pm0.177$ \xmark & $0.377\pm0.126$ \xmark \\
\bottomrule
\end{tabular}
\end{table}

A safe-region indicator cost (cost $1$ at every step outside $\safeR(\dmax)$) is a second natural encoding\revised{. We report it for both \ac{CMDP} solvers (\Cref{tab:cmdp-dense}). Across the table, \emph{feasibility in the region metric does not bound falls}. On \envHC{} the dense and sparse encodings coincide, the constraint never binding. On \envAnt{} both solvers are feasible on the region budget yet fall at least $12\times$ (\ac{CPO}) and $9\times$ (\ac{PPO}-Lagrangian) \methodName{}'s count, at far lower reward; on \envGo{} the dense cost raises \ac{PPO}-Lagrangian's reward, but both solvers violate their own budget and still fall at least $12\times$ (\ac{PPO}-Lagrangian) and $80\times$ (\ac{CPO}) \methodName{}'s count. Every comparative claim here holds at either endpoint of the falls brackets.

\begin{table}[ht]\revisedfloat
\centering\small
\caption{\revised{Dense safe-region \ac{CMDP} baseline: cost $1$ per step outside $\safeR(\dmax)$, limit $5.0$, $10$ seeds per cell, body budgets and $\dmax$. Reward is the last-$10\%$ mean; episodic region cost is the final-$5\%$ mean against the limit ($\le 5.0$); falls are conservative reconstructed brackets $[\mathrm{lb}, \mathrm{ub}]$ from episode-length accounting, since these runs log no per-fall signal.}}
\label{tab:cmdp-dense}
\setlength{\tabcolsep}{4pt}
\resizebox{\ifdim\width>\textwidth\textwidth\else\width\fi}{!}{%
\begin{tabular}{l l r r l l}
\toprule
Env & Method & Reward & Episodic region cost & Feasible? & Falls bracket \\
\midrule
\envHC & \ac{CPO} & $2729\pm303$ & $0.0\pm0.0$ & \cmark{} (never binds) & $[4.5,\ 312]$ \\
\envHC & \ac{PPO}-Lagrangian & $2733\pm336$ & $0.0\pm0.0$ & \cmark{} (never binds) & $[5.2,\ 316]$ \\
\envAnt & \ac{CPO} & $676\pm118$ & $4.2\pm0.4$ & \cmark & $[11119,\ 19078]$ \\
\envAnt & \ac{PPO}-Lagrangian & $1265\pm640$ & $3.5\pm1.9$ & \cmark & $[8384,\ 16296]$ \\
\envGo & \ac{CPO} & $993\pm531$ & $5.9\pm1.1$ & \xmark & $[66747,\ 82746]$ \\
\envGo & \ac{PPO}-Lagrangian & $4498\pm589$ & $5.1\pm1.5$ & \xmark & $[10208,\ 26203]$ \\
\bottomrule
\end{tabular}}
\end{table}}

\revised{\subsection{\envWalker: trigger calibration and dose-response}
\label{app:walker2d}

This subsection reports \envWalker{}, a fourth environment kept out of the body results. It sits at the narrow end of the failure spectrum; falls are fast and intrinsically hard to prevent, \ac{PPO} averaging one fall per roughly $210$ training steps. The protocol is $8$M steps per run and $5$ to $10$ seeds per arm with $10$ on the headline comparison, with \ac{PPO} retrained under identical code and separately pretrained \ac{SAC} checkpoints as the recovery policies. The legacy distance is $\mathcal{D}(s) = \|(z - 1.25,\ \theta)\|$; the calibrated distance normalizes each component by its distance from the nominal state to the termination surface, $\mathcal{D}'(s) = \|((z - 1.25)/0.55,\ \theta/2.0)\|$, so both fall routes cross $\mathcal{D}' \approx 1$ exactly at termination. \Cref{tab:walker-calibration} holds the full trigger-calibration dose-response. Trigger tightness under the calibrated distance acts as a dial; falls are monotone in the radius, and the widest calibrated radius reproduces the transplanted failure of the legacy rows. That legacy transplant failure is geometric rather than a competence failure; $61\%$ of the falls in the legacy arms occurred in episodes where the trigger never fired, a height collapse capping the legacy distance at about $0.55$ while the annealed radius sweeps past it, even though the offline catch-rate screen clears the deployed checkpoint at $96$ to $99\%$ at every trigger distance tested. The winning cell is a moderated claim by construction, a real but single-digit reduction ($3.94\times$, $95\%$ stratified-bootstrap CI $[3.27, 4.78]$, $10$ seeds) at a $25\%$ reward cost, not the orders of magnitude of the wide-band regime.

\textbf{Catch-rate screen and disclosures.} The practitioner guidance we draw is a screening protocol rather than a retraining recommendation. Before any task training is spent, roll a candidate recovery out from doomed handover states collected at the intended trigger band and record the fraction it saves. The measurement is cheap, needs no task training, and in this study predicts the in-loop outcome where the recovery's own training return does not. The screen itself uses $200$ handover states per trigger band, snapshotted at the first threshold crossing of each excursion of an unassisted walker; left uncontrolled, $88$ to $97\%$ of them fall within about $15$ to $20$ steps, so the states are genuinely doomed, and the catch rate is the fraction the candidate returns to the safe region before termination. \Cref{tab:walker-competence} holds the in-loop result. With the trigger fixed at the winning calibrated setting and only the recovery controller varying, the deployed $97\%$ catcher yields the $3.94\times$ reduction, while both sub-threshold checkpoints ($72\%$ and $58\%$) produce roughly $3\times$ \emph{more} falls than running with no recovery controller at all while destroying task learning. The effect is a cliff rather than a dial; the two sub-threshold arms are statistically indistinguishable from each other. The legacy-trigger $58\%$ arms of \Cref{tab:walker-calibration} sit at only $1.1$ to $1.45\times$ \ac{PPO}'s falls, not $3\times$, because there the mis-calibrated trigger rarely fired; the trigger determines how often the controller is consulted, and the catch rate what happens when it is. Both recovery checkpoints' pretraining falls are counted in \Cref{app:pretraining-accounting}. Recovery-disabled evaluation of the winning configuration reaches under half the mixed-controller reward; at the tight trigger the learned policy leans on the recovery, so unlike on \envHC{} and \envAnt{} it is not self-sufficient, and the deployment-relevant quantity is the mixed-controller number under the recovery-permanent $J^{\mathrm{mix}}$ framing of \Cref{sec:background} (\Cref{app:recovery-disabled}).

\begin{table}[ht]\revisedfloat
\centering\small
\caption{\revised{\envWalker{} trigger calibration: full dose-response over both trigger families, $8$M steps per run, mean $\pm$ sample standard deviation (denominator $n-1$) over seeds \revised{(spreads are sample standard deviations; every arm carries at least $5$ seeds and the calibrated arms $10$)}. Catch rate is the offline single-shot catch rate of the arm's recovery checkpoint from doomed handover states at the tight band; $\dmax$ is the annealed radius parameter (final trigger radius $\dmax + 0.05$); Rec.\ frac.\ is the fraction of late-training steps under recovery. The caveat on the winning arm (bold) is that its recovery-disabled evaluation reward is $260 \pm 110$ versus the mixed-controller $636 \pm 117$, so the trained policy is not self-sufficient without the recovery (\Cref{app:recovery-disabled}).}}
\label{tab:walker-calibration}
\setlength{\tabcolsep}{4pt}
\begin{tabular}{llrrrrrr}
\toprule
Trigger & Catcher & $\dmax$ & $n$ & Training falls & Training reward & Rec.-off eval reward & Rec.\ frac. \\
\midrule
none (\ac{PPO}) & none & -- & $10$ & $37954\pm5182$ & $848\pm192$ & $861\pm183$ & $0.000$ \\
\midrule
legacy $\mathcal{D}$ & $58\%$ & $1$ & $5$ & \revised{$54864\pm3977$} & \revised{$270\pm48$} & \revised{$263\pm46$} & $0.029$ \\
legacy $\mathcal{D}$ & $58\%$ & $2$ & $5$ & $45007\pm3479$ & $550\pm205$ & $700\pm520$ & $0.000$ \\
legacy $\mathcal{D}$ & $58\%$ & $4$ & $5$ & \revised{$44755\pm17222$} & \revised{$709\pm153$} & \revised{$768\pm200$} & $0.000$ \\
legacy $\mathcal{D}$ & $58\%$ & $8$ & $5$ & \revised{$41037\pm3226$} & \revised{$761\pm160$} & \revised{$787\pm166$} & $0.000$ \\
legacy $\mathcal{D}$ & $97\%$ & $1$ & $5$ & \revised{$33003\pm3183$} & \revised{$433\pm89$} & \revised{$431\pm107$} & $0.011$ \\
legacy $\mathcal{D}$ & $97\%$ & $2$ & $5$ & \revised{$35477\pm3785$} & \revised{$672\pm89$} & \revised{$696\pm97$} & $0.000$ \\
legacy $\mathcal{D}$ & $97\%$ & $4$ & $5$ & \revised{$33182\pm2500$} & \revised{$811\pm206$} & \revised{$829\pm222$} & $0.000$ \\
legacy $\mathcal{D}$ & $97\%$ & $8$ & $5$ & \revised{$33840\pm3581$} & \revised{$961\pm176$} & \revised{$1000\pm207$} & $0.000$ \\
\midrule
calibrated $\mathcal{D}'$ & $97\%$ & $0.2$ & $10$ & $\mathbf{9634\pm2647}$ & $636\pm117$ & $260\pm110$ & $0.273$ \\
calibrated $\mathcal{D}'$ & $97\%$ & $0.4$ & $10$ & $17722\pm2982$ & $333\pm158$ & $206\pm152$ & $0.200$ \\
calibrated $\mathcal{D}'$ & $97\%$ & $0.8$ & $10$ & $29984\pm8793$ & $283\pm170$ & $261\pm169$ & $0.105$ \\
\bottomrule
\end{tabular}

\end{table}

\begin{table}[ht]\revisedfloat
\centering\small
\caption{\revised{\envWalker{} recovery-controller competence at the winning calibrated trigger ($\mathcal{D}'$, $\dmax = 0.2$): only the recovery controller varies, $8$M steps per run, mean $\pm$ sample standard deviation over $10$ seeds. Catch rate is the offline screen described in the text ($200$ doomed handover states at the tight band). A sub-threshold controller is not merely less helpful; consulted at every excursion, it accumulates roughly $3\times$ \ac{PPO}'s falls at near-zero reward, and the two sub-threshold arms are statistically indistinguishable from each other (Welch $t = 0.22$).}}
\label{tab:walker-competence}
\setlength{\tabcolsep}{4pt}
\begin{tabular}{lrrr}
\toprule
Recovery controller (catch rate) & Training falls & Training reward & Rec.-off eval reward \\
\midrule
$97\%$ (deployed) & $\mathbf{9634\pm2647}$ & $636\pm117$ & $260\pm110$ \\
none (\ac{PPO}) & $37954\pm5182$ & $848\pm192$ & $861\pm183$ \\
$72\%$ (second candidate checkpoint) & $120080\pm12731$ & $29\pm20$ & $40\pm17$ \\
$58\%$ (weaker pretrained checkpoint) & $115608\pm62204$ & $30\pm17$ & $41\pm21$ \\
\bottomrule
\end{tabular}
\end{table}

}

\section{Broader Impact}
\label{app:broader-impact}

The goal of this work is to make reinforcement learning on physical robots safer to train. Achieving that goal would reduce hardware damage during training, lower the experimental cost borne by robot-learning research groups, and provide a foundation for on-robot continual learning that does not depend on extensive pre-training in simulation.

Set against that intent, the dual-use surface of the technique is narrow. It is dual-use only in the trivial sense that any improvement in robotic control technology is so, because the method does not change the capabilities a deployed policy can express, only the training-time safety profile of how that policy is reached. For this reason we do not foresee specific misuse pathways introduced by this work that are not already present in the underlying \ac{PPO} + \ac{SAC} + MuJoCo / Go1 stack.

The experiments themselves are deliberately scoped to limit any training-time harm. We do not collect or release human-subject data and do not perform experiments that could affect bystanders during training. All experiments are in simulation (MuJoCo and the Go1 model), so the training-time falls we reduce are simulated rather than physical; carrying that reduction onto hardware is future work (\Cref{sec:discussion}).
\revised{Complete safety requires hardware safety mechanisms, human supervision, emergency stops, and formal safety validation, and \methodName{} replaces none of them. What we show is that falls can be reduced, which reduces the chance of hardware damage during training. The method still relies on a hand-designed safe region and a recovery controller that can fail.}

Because the contribution is methodological, its value depends on being reproducible and affordable, and we account for both: the full algorithm, hyperparameter, and environment specifications needed to rerun the experiments are given in \Cref{app:algorithm,app:hyperparams,app:env}. We report negative results alongside the positive ones, the soft-gate variants, the fixed-$d$ concessions, and the action-noise stress failures (\Cref{app:soft-variants,app:fixed-d,app:action-noise}), so that future work can avoid repeating them. The compute budget for the component ablation ($9$ variants $\times$ $3$ environments $\times$ $5$ seeds, 4 cores each, weighted by per-environment wall-clock) is approximately $3{,}700$ CPU-core-hours, modest relative to contemporary safe-\ac{RL} benchmarks; the \revised{ten-seed extensions, sensitivity sweeps, and \envWalker{} study of \Cref{app:extra-results} add further compute}.

\end{document}